\pdfoutput=1
\documentclass[letterpaper]{article}
\usepackage{ijcai17}
\usepackage{times}
\usepackage{paralist}
\usepackage{amsmath}
\usepackage{amssymb}
\usepackage{amsthm}
\usepackage{latexsym}
\usepackage{enumerate}
\usepackage{thm-restate}
\usepackage{xspace}
\usepackage{amsfonts}
\usepackage{graphicx}
\usepackage{xcolor}
\usepackage{relsize}
\usepackage{textcomp}
\usepackage[shortlabels]{enumitem}

\newcommand{\A}{{\cal A}}
\newcommand{\I}{{\cal I}}
\newcommand{\J}{{\cal J}}
\newcommand{\K}{{\cal K}}

\newcommand{\T}{{\cal T}}

\newcommand{\Ind}{{\mathbf{I}}}		% individuals (constants)
\newcommand{\Vars}{{\mathbf{X}}}		% variables
\newcommand{\Concepts}{{\mathbf{C}}}	% atomic roles
\newcommand{\Roles}{{\mathbf{R}}}		% atomic roles

\newcommand{\ans}[2]{#1^{#2}}
\newcommand{\cert}[2]{#1^{#2}}
\newcommand{\ansy}[3]{[#1, #2]^{#3}}

\newcommand{\eval}{\lambda}
\newcommand{\eqv}{\sim}
\newcommand{\eqc}[1]{\tilde{#1}}

\newcommand{\bmodels}{\models^\textnormal{\texttt{b}}}
\newcommand{\ccl}[3]{\mathsf{ccl}_{#3}[#1,#2]}

\newcommand{\isa}{\sqsubseteq}

\newcommand{\nati}{\mathbb{N}}

\newcommand{\natinfz}{\nati^\infty_0}

\newcommand{\cansymb}{\mathcal{C}}
\newcommand{\can}[1]{\cansymb(#1)}
\newcommand{\pcan}[2]{\cansymb_{#2}(#1)}

\newcommand{\set}[1]{\{#1\}}

\newcommand{\tpl}[1]{\mathbf{#1}}
\newcommand{\emptytpl}{\langle\rangle} 

\newcommand{\dllite}{\textit{DL-Lite}}
\newcommand{\dlliteC}{\textit{DL-Lite}_{\textit{core}}}
\newcommand{\dlliteR}{\textit{DL-Lite}_{\cal R}}

\newcommand{\dlliteCbag}{\textit{DL-Lite}_{\textit{core}}^{\textit{bag}}}
\newcommand{\dlliteRbag}{\textit{DL-Lite}_{\cal R}^{\textit{bag}}}

\newcommand{\bagl}{\{\hspace{-.5ex}|}
\newcommand{\bagr}{|\hspace{-.5ex}\}}
\newcommand{\emptybag}{\emptyset}

\newcommand{\BagCert}{\textsc{BagCert}\xspace}

\newcommand{\balgonee}{\textsc{BALG}^1_\varepsilon}
\newcommand{\np}{\textsc{NP}}
\newcommand{\conp}{\textsc{coNP}}

\newcommand{\aczero}{\textsc{AC}^0}
\newcommand{\logspace}{\textsc{LogSpace}}

\newcommand{\ex}{{\textit{ex}}}

\newcommand{\minus}{\setminus}
\newcommand{\veedot}{%
\mathop{\mathchoice%
  {\mathrel{\ooalign{\hss$\vee$\hss\cr%
  \kern0.43ex\raise0.5ex\hbox{$\cdot$}\raise0.5ex\hbox{}}}}%
  {\mathrel{\ooalign{\hss$\vee$\hss\cr%
  \kern0.43ex\raise0.5ex\hbox{$\cdot$}\raise0.5ex\hbox{}}}}%
  {\mathrel{\ooalign{\hss$\vee$\hss\cr%
  \kern0.43ex\raise0.5ex\hbox{$\cdot$}\raise0.5ex\hbox{}}}}%
  {\mathrel{\ooalign{\hss$\vee$\hss\cr%
  \kern0.43ex\raise0.5ex\hbox{$\cdot$}\raise0.5ex\hbox{}}}}%
}}
\newcommand{\bigveedot}{%
\mathop{\mathchoice%
  {\mathrel{\ooalign{\hss$\bigvee$\hss\cr%
  \kern0.64ex\raise0.5ex\hbox{$\cdot$}\raise-0.5ex\hbox{}}}}%
  {\mathrel{\ooalign{\hss$\bigvee$\hss\cr%
  \kern0.64ex\raise0.5ex\hbox{$\cdot$}\raise-0.5ex\hbox{}}}}%
  {\mathrel{\ooalign{\hss$\bigvee$\hss\cr%
  \kern0.64ex\raise0.5ex\hbox{$\cdot$}\raise-0.5ex\hbox{}}}}%
  {\mathrel{\ooalign{\hss$\bigvee$\hss\cr%
  \kern0.64ex\raise0.5ex\hbox{$\cdot$}\raise-0.5ex\hbox{}}}}%
}}

\newcommand{\pair}[1]{( #1 )}

\usepackage{algorithm}
\usepackage{algorithmicx}
\usepackage[noend]{algpseudocode}

\usepackage{comment}
\usepackage{nameref}
\usepackage[hidelinks]{hyperref}

\newtheorem{theorem}{Theorem}
\newtheorem{definition}[theorem]{Definition}
\newtheorem{proposition}[theorem]{Proposition}  
\newtheorem{lemma}[theorem]{Lemma}

\newtheorem{myexample}[theorem]{Example}
\let\oldmyexample\myexample
\renewcommand{\myexample}{\oldmyexample\normalfont}
\newenvironment{example}{\begin{myexample}}{\hfill$\lozenge$\end{myexample}}

\makeatletter
\def\namedlabel#1#2{\begingroup
   \def\@currentlabel{#2}%
   \label{#1}\endgroup
}
\makeatother

\begin{document}
\title{The Bag Semantics of Ontology-Based Data Access%
\thanks{This work was supported by the Royal Society under a University Research
Fellowship, the EPSRC projects ED3 and DBOnto, and the Research Council of
Norway via the Sirius SFI.}}

\author{Charalampos Nikolaou \and
        Egor V.\, Kostylev \and
        George Konstantinidis \and \\
        \textbf{Mark Kaminski} \and
        \textbf{Bernardo Cuenca Grau} \and
        \textbf{Ian Horrocks}\\
Department of Computer Science, University of Oxford, UK}

\maketitle
  
\begin{abstract}
Ontology-based data access (OBDA) is a popular approach
for integrating and 
querying multiple data sources by means of a shared ontology.
The ontology is linked to the  sources
using mappings, which assign views over the data to ontology predicates. 
Motivated by the need for OBDA systems supporting
database-style aggregate queries, we propose a 
bag semantics for  
OBDA, where duplicate tuples in the views defined by the mappings 
are retained, as is the case in standard databases.
%taken into account. 
We show that bag semantics
makes conjunctive query answering in OBDA $\conp$-hard in data complexity.
To regain  tractability, we consider a rather general class of queries and show 
its %such queries are 
rewritability to a generalisation of the relational calculus to bags.
%\newtext{Unfortunately, the nice properties of rooted queries
%do not extend to the more expressive language $\dlliteR$.}
\end{abstract} 

%!TEX root = main.tex

\section{Introduction}

Ontology-based data access (OBDA) is an increasingly
popular approach to enable uniform access to 
multiple data sources with diverging schemas \cite{PoggiJDS08}.

In OBDA, an ontology 
provides a unifying conceptual model  for
the data sources together with
domain knowledge.
The ontology is linked to 
each source by  global-as-view (GAV)
mappings \cite{DBLP:conf/pods/Lenzerini02}, which 
assign  views over the data 
to ontology predicates.
Users 
%of OBDA systems 
%are
%typically  unaware of the details of the source schemas, and 
access the data by means of
queries formulated using 
the vocabulary of the ontology; 
query answering amounts to 
computing the certain answers to the query over the union of
ontology and  the materialisation of the views
defined by the mappings. 
The formalism  of choice for representing ontologies in
OBDA is
the description logic
$\dlliteR$ \cite{CalvaneseJAR07}, which underpins OWL 2 QL \cite{OWL2-profiles}. $\dlliteR$ was 
designed to
ensure that 
queries
against the ontology are \emph{first-order rewritable}; that is, they
can be 
reformulated as a set of relational queries
over the sources \cite{CalvaneseJAR07}.

\begin{example}\label{ex:running}
A company stores data about  departments and their
employees in several databases. The 
sales department uses the  schema  
$\mathsf{SalEmployee}(\mathsf{id}, \mathsf{name}, \mathsf{salary}, \mathsf{loc}, \mathsf{mngr})$,
where attributes $\mathsf{id}$, $\mathsf{name}$, 
$\mathsf{salary}$, $\mathsf{loc}$, and $\mathsf{mngr}$ 
stand for employee ID within the department, their name, salary, location, and name of their manager.
In turn, the IT department stores data using the schema
$\mathsf{ITEmployee}(\mathsf{id}, \mathsf{surname}, \mathsf{salary}, \mathsf{city})$, where
managers
are not
specified.
To integrate employee data, the company relies on an  ontology
with TBox $\T_{\ex}$, which
defines
unary predicates such as $\mathsf{SalEmp}$,  $\mathsf{ITEmp}$, and 
$\mathsf{Mngr}$,
and binary predicates
such as  $\mathsf{hasMngr}$ 
relating employees to their managers.
The following mappings determine the extension of the
predicates based
on the data, where each $\mathit{att}_i$ represents the
attributes occurring only in the source:    
\[
\begin{array}{lcl}
\!\!\!\mathsf{SalEmployee}(\mathsf{name}, \mathit{att}_1) & \!\!\!\rightarrow\!\!\! & \mathsf{SalEmp}(\mathsf{name}), \\
\!\!\!\mathsf{SalEmployee}(\mathsf{name}, \mathsf{mngr}, \mathit{att}_2)\! &\!\!\!\rightarrow\!\!\! & \mathsf{hasMngr}(\mathsf{name}, \mathsf{mngr}), \\
\!\!\!\mathsf{SalEmployee}(\mathsf{mngr}, \mathit{att}_3) & \!\!\!\rightarrow\!\!\! & \mathsf{Mngr}(\mathsf{mngr}), \\
%& \mathsf{SalesEmp}( \mathsf{t_4}, \mathsf{name}, \mathsf{sal} ) \rightarrow
% \mathsf{sSal}(\mathsf{name}, \mathsf{sal})\\
\!\!\!\mathsf{ITEmployee}(\mathsf{surname},  \mathit{att}_4) & \!\!\!\rightarrow\!\!\! & \mathsf{ITEmp}(\mathsf{surname}).
%& \mathsf{ITEmp}( \mathsf{t_6}, \mathsf{surname}, \mathsf{ysal} ) \rightarrow
% \mathsf{itSal}(\mathsf{surname}, \mathsf{ysal})
\end{array}
\]
TBox $\T_{\ex}$ 
specifies the meaning of its 
vocabulary using inclusions
\begin{inparaenum}[(i)]
\item $\mathsf{SalEmp} \sqsubseteq \mathsf{Emp}$ and
$\mathsf{ITEmp} \sqsubseteq \mathsf{Emp}$, which say that
both sales and IT employees are company employees;
%\item $\mathsf{sSal} \sqsubseteq \mathsf{hasSalary}$ and 
%$\mathsf{itSal} \sqsubseteq \mathsf{hasSalary}$, which 
%define a global relation between employees and salaries across the company,
\item $\exists \mathsf{hasMngr}^- \sqsubseteq \mathsf{Mngr}$, specifying the range of the
$\mathsf{hasMngr}$ relation, and
\item $\mathsf{Emp} \sqsubseteq \exists \mathsf{hasMngr}$, requiring that
employees have a (maybe unspecified) manager.
\end{inparaenum}
Such inclusions influence query answering: 
when asking for the names of all company employees,
% with a salary
%exceeding $50k$, 
the system will retrieve
all relevant sales and IT employees; this is achieved via query
rewriting, where the query is reformulated as the union of
queries over the sales and IT  databases.
\end{example}
OBDA has received a great deal of attention in recent years. 
Researchers have studied the limits of 
first-order rewritability in ontology languages \cite{CalvaneseJAR07,DBLP:journals/jair/ArtaleCKZ09}, 
established bounds on the size of rewritings \cite{DBLP:journals/ai/GottlobKKPSZ14,DBLP:conf/csl/KikotKPZ14},
developed optimisation techniques \cite{DBLP:conf/semweb/KontchakovRRXZ14}, and implemented
systems well-suited for real-world applications \cite{DBLP:journals/semweb/CalvaneseCKKLRR17,DBLP:journals/semweb/CalvaneseGLLPRRRS11}.

An important observation about
the conventional
semantics of OBDA  is that it is
set-based: the materialisation of the views
defined by the mappings is formalised as 
a \emph{virtual ABox}  consisting of a set of facts over the 
ontology predicates. This treatment is, however,
in contrast with the semantics of database views, which is based on bags
(multisets) and where duplicate tuples are retained by default.
The distinction between set and bag semantics in databases is 
very significant in practice; in particular, it influences the 
evaluation of aggregate queries, which combine various
aggregation functions such as $\mathsf{Min}$, $\mathsf{Max}$, $\mathsf{Sum}$, $\mathsf{Count}$ or
$\mathsf{Avg}$ with the grouping functionality provided in SQL by the
$\mathsf{GroupBy}$ construct. 
\begin{example}\label{ex:running-cont}
 Consider the query asking for the number of employees named
 Lee. Assume there are two different 
 employees named Lee, which are represented as different tuples in the
 sales database (e.g.,  tuples with the same employee name, but different ID).
 Under the conventional semantics of OBDA, the virtual
 ABox would
 contain a single fact $\mathsf{SalEmp}(\mathit{Lee})$; hence,
 the query would wrongly return one, even under 
 the semantics for counting aggregate queries in
 \cite{CalvaneseONISW08,KostylevJWS15}.
 The correct count can be obtained by considering  the extension of
 $\mathsf{SalEmp}$ as a bag with multiple occurrences of
 $\mathit{Lee}$.
\end{example}

%In our example, these include queries such as 
%``How many employees in the company are called Lee?'', 
% ``What is the average salary of employees based in London?''.
%Although aggregate queries are heavily used in
%data warehousing and business intelligence scenarios,
%they have received only limited attention in the context
%of ontologies, with \cite{CalvaneseONISW08,KostylevJWS15} being notable
% exceptions.

%, and where 
%grouping and aggregation operators are naturally
%formalised in terms of bags rather than sets.
%

\vspace{-0.15cm}
The goal of this paper is to propose and study
a bag semantics for OBDA which is compatible with the
semantics of standard databases and can 
provide a suitable foundation for the future study of aggregate queries. 
We focus on conjunctive query (CQ) answering over $\dlliteR$ ontologies under
bag semantics, and our main contributions are
as follows.
%, where
%ABoxes are formalised as bags of facts, rather than sets.
% thus providing a 
%faithful representation  of the views defined by OBDA mappings.
%

\begin{enumerate}[noitemsep,topsep=0pt,leftmargin=*]
\item We propose the ontology language $\dlliteRbag$ and its restriction $\dlliteCbag$, where 
ABoxes consist of a bag of facts, thus providing a 
faithful representation  of the views defined by OBDA mappings. We define 
the semantics
of query answering in this setting 
and show that it is compatible with the conventional set-based
semantics.

\item We show that, in contrast to the set case, 
ontologies may not have a universal model (i.e., a single model over which
all CQs can be correctly evaluated), and 
bag query answering becomes $\conp$-hard in data complexity even if we restrict ourselves to $\dlliteCbag$ ontologies.

\item To regain tractability, we study the class
of \emph{rooted CQs}~\cite{BienvenuKR12}, where each  
connected component of the query graph is required to contain
an individual or an answer variable. 
This is a very general class, which arguably captures
most practical OBDA queries.
We show that rooted CQs over $\dlliteCbag$ ontologies 
not only admit a universal model and enjoy favourable computational properties, but also 
allow for rewritings that can be directly evaluated over the bag
ABox of the ontology.
%Unfortunately, these nice properties  do not extend to $\dlliteRbag$, where query answering is again $\conp$-hard.
\end{enumerate}
%For complete proofs we refer to~\cite{nikolaouIJCAICORR17}.
%Proofs of all results may be found in~\cite{nikolaouIJCAICORR17}.
%For the proofs of all results we refer to~\cite{nikolaouIJCAICORR17}.
Proofs of all results are deferred to the appendix.

%%% Local Variables:
%%% mode: latex
%%% TeX-master: "main"
%%% End:

%!TEX root = main.tex

\section{Preliminaries}
\label{sec:prelim}

\noindent
\textbf{Syntax of Ontologies }
We fix a vocabulary consisting
of countably infinite and pairwise disjoint sets of \emph{individuals} $\Ind$ (i.e., constants), \emph{variables} $\Vars$,
\emph{atomic concepts} $\Concepts$ (unary predicates) and 
\emph{atomic roles} $\Roles$ (binary predicates).
A  \emph{role} is  an atomic role $P \in \Roles$ or its \emph{inverse} $P^-$.
A \emph{concept} is an atomic concept in $\Concepts$ or
an expression $\exists R$, where $R$ is a role.
An \emph{inclusion} is an expression of the form $S_1 \isa S_2$ with $S_1$ and $S_2$ either both concepts or both roles.
%A \emph{concept inclusion} is of the form $C_1 \isa C_2$ with $C_1$ and $C_2$ concepts.
%A \emph{role inclusion} is  of the form $R_1 \isa R_2$ with $R_1$ and $R_2$ roles.
A \emph{disjointness axiom} is an expression of the form $\mathsf{Disj}(S_1,S_2)$ with $S_1$ and $S_2$
either both concepts or both roles.
%$C_1$ and $C_2$ concepts. A
%\emph{role disjointness axiom} is an expression $\mathsf{Disj}(R_1,R_2)$, with
% $R_1$ and $R_2$ roles.
A \emph{concept assertion} is of the form $A(a)$ with $a \in \Ind$ 
and $A \in \Concepts$. A \emph{role assertion} is of the form
$P(a,b)$ with $a,b \in \Ind$  
and $P \in \Roles$.
A  $\dlliteR$ \emph{TBox} is a finite set of inclusions and
disjointness axioms.  An \emph{ABox}
is a finite set of concept and role assertions.
A $\dlliteR$ \emph{ontology} is a pair $\langle \T, \A \rangle$
with $\T$ a $\dlliteR$ TBox and $\A$ an ABox.
The ontology language $\dlliteC$ restricts
$\dlliteR$ by disallowing inclusions and disjointness axioms for roles.

\smallskip
\noindent
\textbf{Semantics of Ontologies } An \emph{interpretation} $\I$ is a pair
$\langle \Delta^{\I}, \cdot^{\I} \rangle$, where
the \emph{domain} $\Delta^{\I}$ is a non-empty set,
and the \emph{interpretation function} $\cdot^{\I}$
maps each $a \in \mathbf{I}$ to $a^{\I} \in \Delta^{\I}$ 
such that $a^\I \neq b^\I$ for all $a, b \in \Ind$,\footnote{We adopt the
unique name assumption for convenience; dropping it does not affect results 
(modulo minor changes of definitions).}
%each $a \in \mathbf{I}$ to itself, 
each $A \in \Concepts$ to a subset $A^{\I}$ of
$\Delta^{\I}$
and each $P \in \Roles$ to a subset $P^{\I}$ 
of $\Delta^{\I} \times \Delta^{\I}$.
The
interpretation function extends to concepts and roles as follows:
$(R^-)^{\I} = \{(u,v) \mid (v,u) \in R^{\I}\}$ and
$(\exists R)^{\I} = \{ u \in \Delta^{\I} \mid (u,v) \in R^{\I} \text{ for some } v \in \Delta^{\I}\}$. 

An interpretation $\I$
\emph{satisfies} ABox $\A$ if $a^{\I} \in A^{\I}$ for all 
$A(a) \in \A$ and $(a^{\I}, b^{\I}) \in P^{\I}$
for all 
$P(a,b) \in \A$;
%.
$\I$ \emph{satisfies}  TBox $\T$ if $S_1^{\I} \subseteq S_2^{\I}$ for all  $S_1
\isa S_2$ in $\T$ and $S_1^{\I} \cap S_2^{\I} = \emptyset$ for all 
$\mathsf{Disj}(S_1,S_2)$ in $\T$;
%\begin{inparaenum}[(i)]
%\item $C_1^{\I} \subseteq C_2^{\I}$ for all its concept inclusions $C_1 \isa C_2$;
%\item $R_1^{\I} \subseteq R_2^{\I}$ for all its role inclusions $R_1 \isa R_2$;
%\item $C_1^{\I} \cap C_2^{\I} = \emptyset$
%for all its concept disjointness axioms 
%$\mathsf{Disj}(C_1,C_2)$; and 
%\item $R_1^{\I} \cap R_2^{\I} = \emptyset$ for all its role disjointness axioms $\mathsf{Disj}(R_1,R_2)$.
%\end{inparaenum}
%
%Finally,
$\I$ is a \emph{model} of ontology $\langle \T, \A
\rangle$ %, written $\I \models \K$, %% we never use this
if it satisfies $\T$ and $\A$. An ontology is \emph{satisfiable} if it has a
model.

\smallskip
\noindent
\textbf{Queries} A \emph{conjunctive query} (\emph{CQ})
$q(\tpl x)$ with \emph{answer} variables $\tpl x$ is a formula $\exists \tpl y. \, \phi(\tpl x, \tpl y)$, where $\tpl x$,
$\tpl y$ are (possibly empty) repetition-free tuples of variables %from
% $\Vars$
and $\phi(\tpl x, \tpl y)$ is a conjunction of atoms of the form $A(t)$, 
$P(t_1, t_2)$ or $z = t$, where  $A \in \Concepts$,
$P \in \Roles$, $z \in \tpl x \cup \tpl y$, and $t, t_1, t_2 \in \tpl x \cup \tpl y \cup \Ind$. 
If $\tpl x$ is inessential, then we write $q$ instead of  $q(\tpl x)$.
If $\tpl x$ is the empty tuple $\emptytpl$, then $q$ is \emph{Boolean}.
%If $\phi(\tpl x, \tpl y)$ consists of a single atom
%and $\tpl y$ is empty, then the query is \emph{atomic}.  
 A \emph{union} of CQs (\emph{UCQ}) is a disjunction of CQs with the same answer
variables.

The equality atoms in a CQ $q(\tpl x) = \exists \tpl y. \, \phi(\tpl x, \tpl y)$
yield an equivalence relation $\eqv$ on terms 
$\tpl x \cup \tpl y \cup \Ind$, and we write $\eqc t$ for the equivalence class of
a term $t$.
The \emph{Gaifman graph} of $q(\tpl x)$
has a node $\eqc t$ for each
$t\in\tpl x \cup \tpl y \cup \Ind$ in $\phi$,
  %with the equivalence classes as nodes
and an edge $\set{\eqc t_1,\eqc t_2}$ for each atom in $\phi$ over %terms
$t_1$ and $t_2$.  We assume that all CQs are \emph{safe}:
for each $z \in \tpl x \cup \tpl y$, the class $\eqc z$ contains a term mentioned in an atom
of $\phi(\tpl x, \tpl y)$ that is not an equality.

The 
\emph{certain answers} 
$\cert{q}{\K}$ 
to a (U)CQ $q(\tpl{x})$ over a $\dlliteR$ ontology $\K$ are the set of all tuples $\tpl{a}$ of individuals such that $q(\tpl a)$ holds in every 
model of $\K$. 
%We also white $\cert{q}{\A}$ instead of $\cert{q}{\langle \emptyset, \A\rangle}$, that is, in the case when the TBox is empty.
%
A class of queries $\mathcal Q_1$ is \emph{rewritable} to a class %of queries
$\mathcal Q_2$ for an ontology language $\mathcal O$ if for any
$q_1 \in \mathcal Q_1$ and TBox $\T$ in $\mathcal O$, there is
$q_2 \in \mathcal Q_2$ such that, for any ABox $\A$ in $\mathcal O$ with
$\langle \T, \A \rangle$ satisfiable,
$\cert{q_1}{\langle \T, \A \rangle}$ equals the answers to $q_2$ in (the least model of) $\A$.
%least model of $\langle\emptyset,\A\rangle$.  
Checking %whether
$\tpl a \in \cert{q}{\langle \T, \A \rangle}$ for a tuple $\tpl a$, (U)CQ
$q$, and $\dlliteR$ ontology $\langle \T, \A \rangle$ is an \np-complete problem
with $\aczero$ data complexity (i.e., when $\T$ and $q$ are fixed)
\cite{CalvaneseJAR07}.  The latter follows from the rewritability of UCQs to themselves for $\dlliteR$.
% A class of queries $\mathcal Q_1$ is \emph{rewritable} to a class of queries
% $\mathcal Q_2$ for an ontology language $\mathcal O$ if for any
% $q_1 \in \mathcal Q_1$ and $\T$ in $\mathcal O$, there exists
% $q_2 \in \mathcal Q_2$ such that, for any ABox $\A$,
% $\cert{q_1}{\langle \T, \A \rangle} = \cert{q_2}{\langle \T, \A
% \rangle}$. %with $\langle \T, \A \rangle$ satisfiable.
% Checking whether $\tpl a \in \cert{q}{\langle \T, \A \rangle}$ for
% individuals $\tpl a$, (U)CQ $q$ and $\dlliteR$ ontology
% $\langle \T, \A \rangle$ is an \np-complete problem with $\aczero$ data
% complexity (i.e., when $\T$ and $q$ are fixed) \cite{CalvaneseJAR07}.  The
% latter result follows from the rewritability of the class of UCQs to itself
% for $\dlliteR$.
% %These results make practical, and this paper studies to what extent they transfer to bag semantics.

\smallskip
\noindent
\textbf{Bags }
%Let 
%$\nati$ be the positive integers, $\natinfi = \nati \cup \{\infty\}$, and 
%$\natinfz$ be the set of nonnegative integers and infinity.
A \emph{bag} over a set $M$
is a function $\Omega: M \to \natinfz$, where $\natinfz$ is the set of 
nonnegative integers and infinity. 
%\footnote{\newtext{
%Our results do not depend on the presence of infinity, however its 
%use makes some proofs easier.}}.
The value $\Omega(c)$ is the \emph{multiplicity} of $c$ in $M$.
A bag $\Omega$ is \emph{finite} if
there are finitely many $c \in M$
with
$\Omega(c)>0$ and there is no $c$ with
$\Omega(c) = \infty$. The \emph{empty bag}~$\emptybag$ over $M$
is the bag such that $\emptybag(c) = 0$ for all $c \in M$.
Given bags $\Omega_1$ and $\Omega_2$ over $M$, %$S_1$ and $S_2$ respectively,
let
$\Omega_1 \subseteq \Omega_2$ if
$\Omega_1(c) \leq \Omega_2(c)$ for each $c \in M$.

The \emph{intersection} $\cap$, \emph{max union} $\cup$, \emph{arithmetic union} $\uplus$, and \emph{difference} $-$ are the binary operations 
defined for bags $\Omega_1$ and $\Omega_2$ over the same set $M$ as follows: 
for every
  $c \in M$, $ (\Omega_1 \cap \Omega_2)(c) = \min\{\Omega_1(c), \Omega_2(c)\}$, 
  $(\Omega_1 \cup \Omega_2)(c) = \max\{\Omega_1(c), \Omega_2(c)\}$, $(\Omega_1 
  \uplus \Omega_2)(c) = \Omega_1(c) + \Omega_2(c)$, and $(\Omega_1 - 
  \Omega_2)(c) = \max\{0,\Omega_1(c) - \Omega_2(c)\}$; difference is 
  well-defined only when  $\Omega_2$ is finite.

\section{\texorpdfstring{$\dlliteR$}{DL-Lite\_R} with Bag Semantics}
\label{sec:sem}

In this section we present a bag semantics for $\dlliteR$ ontologies, define
the associated query answering problem, and establish its intractability  in data complexity. 

%\subsection{The Ontology Language \texorpdfstring{$\dlliteRbag$}{DL-Lite\_R}}

We formalise ABoxes as bags of facts (rather than sets)
in order to faithfully represent 
the materialised views over source data defined by OBDA mappings. 

\begin{definition}
A \emph{bag ABox} is a finite bag 
over the set of concept and role assertions.
A $\dlliteRbag$ \emph{ontology} is a pair $\langle \T, \A \rangle$
of a $\dlliteR$ TBox
$\T$ and a bag ABox $\A$; %$\dlliteCbag$ is a restriction of $\dlliteRbag$ to TBoxes in $\dlliteC$.
the ontology is 
$\dlliteCbag$ if\/ $\T$ is a $\dlliteC$ TBox.
\end{definition}

%\begin{example}
% whereas  $\K_{\ex}' = \langle \T_{\ex}', \A_{\ex} \rangle$, with
%%$\T_{\ex}'$ obtained from $\T_{\ex}$ by removing all role inclusions, is
%in  $\dlliteCbag$.
 %In turn, the 
 %TBox $\T_1$ may consist of axioms
 %$\mathsf{SalEmp} \sqsubseteq \mathsf{Employee}$ and
 %$\mathsf{ITEmp} \sqsubseteq \mathsf{Employee}$.
%\end{example}

The semantics of $\dlliteRbag$ is based on \emph{bag interpretations} $\I$, with
atomic concepts and roles mapped to bags of domain elements and pairs of
elements, respectively, and where the 
interpretation function is extended to complex concepts and roles
in the natural way; in particular, 
%existentially quantified 
a concept $\exists P$ is interpreted
as the bag projection 
%(e.g., see \cite{DBLP:books/daglib/0020812} Chapter 5) 
of $P^{\I}$ 
to the first component, where
each occurrence of a pair $(u,v)$ in $P^{\I}$ 
contributes to the
multiplicity of domain element $u$ in $(\exists P)^{\I}$.

\begin{definition}
A \emph{bag interpretation} $\I$ is a pair $\langle \Delta^{\I}, \cdot^{\I} \rangle$ 
defined the same as in the set case %(set) interpretation
with the exception that $A^{\I}$ and $P^{\I}$ are bags (not sets) over $\Delta^{\I}$ and $\Delta^{\I} \times \Delta^{\I}$, respectively.
The interpretation function extends to concepts and roles as follows:
$(P^-)^{\I}$ maps
each $ (u,v) \in \Delta^{\I} \times \Delta^{\I}$ to
$P^{\I}(v, u)$, and
$(\exists R)^{\I}$ 
maps each $u \in \Delta^{\I}$ to
$\sum_{v \in \Delta^{\I}} R^{\I}( u,v)$.
\end{definition}

The definition of semantics of ontologies is as expected.
%: each assertion 
%in the ABox must be present in the interpretation with at least as high a multiplicity, and
%each axiom in the TBox must be satisfied as in the set case

\begin{definition}
\label{def:bag-model}
A bag interpretation $\I = \langle \Delta^{\I}, \cdot^{\I} \rangle$
\emph{satisfies} a bag ABox $\A$ if 
$\A(A(a)) \leq A^{\I}(a^\I)$ for each concept assertion $A(a)$ in $\A$
and $\A(P(a,b)) \leq P^{\I}(a^\I, b^\I)$ 
for each role assertion $P(a,b)$.
\emph{Satisfaction} of\/ $\T$ is defined 
as in the set case, except that $\subseteq$ and $\cap$ are applied to bags instead of sets.
Bag interpretation $\I$ 
is a \emph{bag model} of the $\dlliteRbag$ ontology $\langle \T, \A \rangle$, written $\I \bmodels \langle \T, \A \rangle$, 
if it satisfies both $\T$ and $\A$. The ontology is \emph{satisfiable} if it has a bag model.
\end{definition}

\begin{example}
\label{ex_ontology}
Let $\K_{\ex} = \langle \T_{\ex}, \A_{\ex} \rangle$ be a 
$\dlliteRbag$
ontology with $\T_{\ex}$ as 
in Example~\ref{ex:running} and 
$\A_{\ex}$ has
$\mathsf{SalEmp}(\mathit{Lee})$ with multiplicity $3$,
$\mathsf{ITEmp}(\mathit{Lee})$ and
$\mathsf{hasMngr}(\mathit{Lee},\mathit{Hill})$ both with 
multiplicity $2$ (and all other assertions with multiplicity 0).
Let $\I_{\ex}$ be the bag interpretation mapping
individuals to themselves and with the following non-zero values:
\[
\begin{array}{ll}
\mathsf{SalEmp}^{\I_{\ex}}(\mathit{Lee}) = \mathsf{Emp}^{\I_{\ex}}(\mathit{Lee}) = 3, \ \mathsf{ITEmp}^{\I_{\ex}}(\mathit{Lee}) = 2, \\ 
\mathsf{hasMngr}^{\I_{\ex}}(\mathit{Lee}, \mathit{Hill}) = 2, \ \mathsf{hasMngr}^{\I_{\ex}}(\mathit{Lee}, w) = 1, \\
\mathsf{Mngr}^{\I_{\ex}}(\mathit{Hill}) = 2, \ \mathsf{Mngr}^{\I_{\ex}}(w) = 1,
\end{array}
\]
where $w$ is a fresh element.
We can check that $\I_{\ex} \bmodels \K_{\ex}$. 
%with the following non-zero values:  $\mathsf{SalEmp}^{\I_{\ex}}(\mathit{Lee}) = \mathsf{Emp}^{\I_{\ex}}(\mathit{Lee}) = 3$,  $\mathsf{ITEmp}^{\I_{\ex}}(\mathit{Lee}) = 2$, $\mathsf{hasMngr}^{\I_{\ex}}(\mathit{Lee}, \mathit{Hill}) = 2$, $\mathsf{hasMngr}^{\I_{\ex}}(\mathit{Lee}, w) = 1$, $\mathsf{Mngr}^{\I_{\ex}}(\mathit{Hill}) = 2$, and $\mathsf{Mngr}^{\I_{\ex}}(w) = 1$.
%$\mathsf{SalEmp}^{\I_{\ex}}$ and $\mathsf{Emp}^{\I_{\ex}}$ consist of $3$ copies of $\mathit{Lee}$ each, $\mathsf{ITEmp}^{\I_{\ex}}$ contains $\mathit{Lee}$ twice, $\mathsf{hasMngr}^{\I_{\ex}}$ contains $\langle \mathit{Lee}, \mathit{Hill} \rangle$ twice and $\langle \mathit{Lee}, w \rangle$ once, $\mathsf{Mngr}^{\I_{\ex}}$ contains $\mathit{Hill}$ twice and $w$ once, and where all other predicates are empty.
%%Note that $3$ copies of $\mathit{Lee}$
%suffice in $\mathsf{Emp}^{\I_{\ex}}$ to satisfy the relevant axioms, and
%the ``anonymous'' element
%$w$ is used to satisfy 
%the requirement that
%each employee have a manager, where $\mathsf{Mngr}^{\I_{\ex}}$
%only contains two copies of $\mathit{Hill}$.
\end{example}

%\subsection{Bag Semantics of Query Answering}

We now define the notion of query answering under bag semantics.
We first define the answers $q^{\I}$ of a CQ $q(\mathbf{x})$ over a bag interpretation $\I$. 
Intuitively, $q^{\I}$ is a bag of tuples of individuals such that 
each valid embedding $\lambda$ of the body of $q$ into $\I$
contributes separately to the multiplicity of the tuple $\lambda(\mathbf{x})$ in $q^{\I}$; in turn,
the contribution of each specific $\lambda$ is the product of the
multiplicities of the images of the query atoms under $\lambda$. The latter is in accordance with
the interpretation of joins in the bag relational algebra and SQL, where the multiplicity of a tuple in a join
is the product of the 
multiplicities of the joined tuples (e.g., see \cite{DBLP:books/daglib/0020812}).

\begin{definition}\label{def:CQ-eval}
  Let $q(\tpl x) = \exists \tpl y. \, \phi(\tpl x, \tpl y)$ be a CQ.  The
  \emph{bag answers} $\ans{q}{\I}$ to $q$ over a bag interpretation
  $\I = \langle \Delta^{\I}, \cdot^{\I} \rangle$ are defined as the bag over tuples of individuals from
  $\Ind$ %of individuals from $\Ind$
  of the same size as $\tpl x$ such that, for every such tuple $\tpl a$,
%$$
%\ans{q}{\I}(\tpl a) \ \ = \ \ \mathlarger{\mathlarger{\mathlarger\Sigma}}_{\lambda \in \Lambda} \ \ \mathlarger{\mathlarger{\mathlarger\Pi}}_{S(\tpl t) \textnormal{ in } \phi(\tpl x, \tpl y)} S^\I(\eval(\tpl t)),
%$$
$$
\ans{q}{\I}(\tpl a) \ \ = \ \ \sum\nolimits_{\lambda \in \Lambda} \ \ \prod\nolimits_{S(\tpl t) \textnormal{ in } \phi(\tpl x, \tpl y)} S^\I(\eval(\tpl t)),
$$
where $\Lambda$ is the set of all \emph{valuations} $\eval: \tpl x \cup \tpl y  \cup \Ind \to \Delta^\I$ such that $\eval(\tpl x) = \tpl a^\I$, $\eval(a) = a^\I$ for each $a \in \Ind$, and $\eval(z) = \eval(t)$ for each $z = t$ in $\phi(\tpl x, \tpl y)$.
%The \emph{bag answers} $\ans{q}{\I}$ to a UCQ $q(\tpl x) = q_1(\tpl x) \lor \cdots \lor q_n(\tpl x)$ over $\I$ are the bag $\ans{q_1}{\I} \cup \cdots \cup \ans{q_n}{\I}$.
%The evaluation of a conjunctive query $q$  over a bag interpretation $\I$ is the bag $q^{\I}$ over valuations from $\mathsf{free}(q)$ to $\mathbf{I}$ inductively defined as follows: \emph{(i)} $q$ an atomic query $P(\mathbf{t})$, then $q^{\I}(\eval) =  P^{\I}(\eval(\mathbf{t}))$; \emph{(ii)} $q = q_1 \wedge q_2$; then, $q^{\I}(\eval) = \sum_{\eval = \eval_1 \cup \eval_2}  q_1^{\I}(\eval_1) \times  q_2^{\I}(\eval_2)$; \emph{(iii)} $q = \exists x.q_1$; then,   $q^{\I}(\eval) = \sum_{\eval = \eval'\vert_{x} } q_1^{\I}(\eval')$. 
\end{definition}
%

%Note that 
If $q$ is Boolean then $q^{\I}$ are defined only
for the empty tuple~$\emptytpl$. Also,
conjunction $\phi(\tpl x, \tpl y)$ %in a CQ
may contain repeated atoms, and hence can be seen as a bag of atoms; while
repeated atoms are redundant in the set case, they are essential in the bag
setting~\cite{VardiPODS93} 
%~\newtext{\cite{VardiPODS93,IoannidisTODS95}} 
and thus the definition of 
$\ans{q}{\I}(\tpl a)$ 
treats each copy of
a query atom
$S(\tpl t)$ separately.

The following definition of certain answers, capturing open-world query answering, is a reformulation of  
the definition in \cite{KostylevJWS15} for counting queries. It is a natural extension of the set notion to bags:
a query answer is 
certain for a given multiplicity if it occurs with at least that multiplicity in every bag model of the ontology. 

\begin{definition}
  The \emph{bag certain answers} $\cert{q}{\K}$ to a query $q$ over a
  $\dlliteRbag$ ontology $\K$ are the bag %is the bag
  $\bigcap\nolimits_{\I \bmodels \K}
  \cert{q}{\I}$. %over tuples over $\Ind$ %of individuals in $\Ind$
%  of the same size as $\tpl x$.
\end{definition}

% We study the problem \BagCert of checking, given an ontology
% $\K = \langle \T,\A \rangle$, query $q$, tuple of individuals $\tpl{a}$, and
% number $k$, whether $\cert{q}{\K}(\tpl a) \geq k$; \emph{data complexity} of
% \BagCert is studied under the assumption that $\T$ and $q$ are fixed.
% Following \cite{GrumbachJCSS96}, we assume that the multiplicities of
% assertions in $\A$ and $k$ are given in unary.
We study the problem $\BagCert[\mathcal Q,\mathcal O]$ %Q and L are not inputs of the problem, so using () is confusing
of checking, given a
query $q$ from a class of CQs $\mathcal Q$, ontology $\K = \langle \T,\A \rangle$ from an ontology language $\mathcal O$,
tuple $\tpl{a}$ over $\Ind$, and number $k \in \natinfz$, whether
$\cert{q}{\K}(\tpl a) \geq k$; \emph{data complexity} of \BagCert is studied
under the assumption that $\T$ and $q$ are fixed.  Following
\cite{GrumbachJCSS96}, we assume that the multiplicities of assertions in $\A$
and $k$ (if not infinity) are given in unary.

%$\mathsf{eval}(I,q(\mathbf{x}), \eval, m) = \mathsf{true} $ for every model $I$ of $\K$.

\begin{example}
Let $q_{\ex}(x) = \exists y.\, \mathsf{hasMngr}(x,y)$ and $\K_{\ex}$ be as in 
Example~\ref{ex_ontology}. Then $q_\ex^{\K_{\ex}}(\mathit{Lee})=3$. Indeed, on 
the one hand, $q_\ex^{\I_\ex}(\mathit{Lee}) = 3$ for $\I_\ex$ in 
Example~\ref{ex_ontology}. On the other, for any bag model $\I$ of $\K_{\ex}$, 
$q_\ex^\I(\mathit{Lee})=\Sigma_{u\in\Delta^\I}\mathsf{hasMngr}^{\I}(\mathit{Lee}^\I,u)\ge
 3$, because $\A_{\ex}(\mathsf{SalEmp}(\mathit{Lee})) = 3$ and $\T_{\ex}$ 
contains inclusions $\mathsf{SalEmp}\sqsubseteq\mathsf{Emp}$ and 
$\mathsf{Emp}\sqsubseteq\exists\mathsf{hasMngr}$ .
\end{example}

% \begin{restatable}{proposition}{BackCompat}
% \label{thm:back-compat}
% Let $\T$ be a  $\dlliteR$ TBox,
% $\A$ be a (set) ABox, and $\A'$ be any
% bag ABox satisfying $\{S(\tpl t) \mid \A'(S(\tpl t)) \geq 1\} = \A$. Then, the following hold:
% %
% \begin{enumerate}
% \item $\langle \T, \A \rangle$ is satisfiable (over set semantics) if and only if $ \langle \T, \A' \rangle$ is satisfiable (over bag semantics);
% \item for any (U)CQ $q(\tpl x)$, 
%  a tuple  $\tpl a$ of individuals is a certain answer
% to $q$ over $\langle \T, \A \rangle$ iff 
% $\cert{q}{\langle \T, \A' \rangle}(\tpl a) \geq 1$. 
% %\item For a satisfiable $\K$, $\mathsf{cert}(q,\K,\eval) = \mathsf{certBag}(q,\K',\eval,1)$.
% \end{enumerate}
% \end{restatable}

%\subsection{Comparison of Bag and Set Semantics}

The bag semantics can be seen as a 
generalisation of the
set semantics of \dllite:
first, satisfiability under bag semantics reduces to
the set case; second, certain answers
under bag and set semantics coincide if multiplicities are ignored.
 
%In particular, we prove that
%a satisfiability under set and bag semantics are equivalent notions; 
%furthermore, a tuple is a certain answer to a query w.r.t.\ $\K$
%iff it is also
%a bag certain answer with multiplicity at least one.

\begin{restatable}{proposition}{BackCompat}
  \label{thm:back-compat}
  Let $\langle\T,\A\rangle$ be a $\dlliteR$ ontology and $\langle\T,\A'\rangle$ be
  a $\dlliteRbag$ ontology with the same TBox such that
  $\{S(\tpl t) \mid \A'(S(\tpl t)) \geq 1\} = \A$. Then, the following holds:
\begin{enumerate}[noitemsep,leftmargin=*,topsep=0pt]
\item $\langle\T,\A\rangle$ is satisfiable if and only if $\langle\T,\A'\rangle$ is satisfiable;
\item for each CQ $q$ and tuple $\tpl a$ of individuals from $\Ind$, 
 $\tpl a \in \cert{q}{\langle\T,\A\rangle}$  if and only if
$\cert{q}{\langle\T,\A'\rangle}(\tpl a) \geq 1$. 
%\item For a satisfiable $\K$, $\mathsf{cert}(q,\K,\eval) = \mathsf{certBag}(q,\K',\eval,1)$.
\end{enumerate}
\end{restatable}
%
%It follows from this result that
%; hence, in the remainder of this paper we focus exclusively
%on CQ answering.
An important property of satisfiable $\dlliteR$ ontologies $\K$ is the
existence of so called universal models for CQs, that is, models $\I$ such that the
certain answers to every CQ $q$ over $\K$ can be obtained by evaluating $q$
over $\I$ %such a universal model
\cite{CalvaneseJAR07}.  This notion extends naturally to bags.

\begin{definition}
A bag model $\I$ of a $\dlliteRbag$ ontology $\K$ is \emph{universal} 
for a class of queries $\mathcal Q$ if $\cert{q}{\K} = \ans{q}{\I}$ for any $q \in \mathcal Q$.
\end{definition}

% In the set case, it is well-known that a model of
% $\K$ is universal for the class of all CQs if and and only if it can be homomorphically embedded
% into every other model of $\K$. This characterisation of
% universal models can be lifted to the setting of bags using a suitable notion of homomorphism. 

% \begin{definition}
% \label{def:homo}
% Given bag interpretations $\I = \langle \Delta^{\I}, \cdot^{\I} \rangle$ and $\J = \langle \Delta^{\J}, \cdot^{\J} \rangle$,
% a function $h: \Delta^\I \to \Delta^\J$ is a \emph{bag homomorphism} if, for all  $a \in \Ind$, $A \in \Concepts$, $P \in \Roles$, and $v, v_1, v_2 \in \Delta^{\J}$,
% $$
% \begin{array}{c}
% \vspace{5pt} h(a^\I) = a^\J, \\  
% \vspace{3pt} \sum\limits_{u \in \Delta^{\I}: \ h(u) = v} A^\I(u) \leq A^\J(v), \text{ and } \\
% \sum\limits_{u_1, u_2 \in \Delta^{\I}: \ h(u_i) = v_i,\, i = 1, 2} \hspace{-20pt}P^\I(u_1, u_2) \leq P^\J(v_1, v_2).
% \end{array}
% $$
% \end{definition}

%% \begin{example}
%%   Consider the $\dlliteRbag$ ontology $\K = (\T, \A)$ such that
%%   $\T = \{ A \isa \exists R, \exists R^- \isa B \}$ and
%%   $\A = \bagl A(a), A(a), R(a, b), B(b) \bagr$.  Then, interpretation
%%   $\I = \bagl A(a), A(a), R(a,b), R(a,c), B(b), B(c)\bagr$ is a
%%   universal~model~of~$\K$.
%% \end{example}

% \begin{restatable}{theorem}{UniversalQA}
% \label{universal-qa}
% %Let $\K$ be a satisfiable $\dlliteRbag$ ontology. 
% A bag model $\I$ of a satisfiable $\dlliteRbag$ ontology $\K$ is universal for 
% the class of all CQs if and only if
% there is a bag homomorphism from $\I$ to any bag model of $\K$.
% \end{restatable}

Unfortunately, in contrast to the set case, even $\dlliteCbag$ ontologies may
not admit a universal bag model for all CQs. 
\begin{restatable}{proposition}{NoUnivModel}
\label{prop:no-univ-model}
There exists a satisfiable $\dlliteCbag$ ontology that has no universal bag 
model for the class of all CQs.
\end{restatable}
%\begin{proof}[Proof (hint)]
%\newtext{
%Consider a variant of our running example where
%$\T = \{\mathsf{Emp} \sqsubseteq \exists \mathsf{hasMngr}, \exists
%\mathsf{hasMngr}^- \sqsubseteq \mathsf{Mngr} \}$ and $\A$ contains
%$\mathsf{Emp}(\mathit{Lee})$ and $\mathsf{Mngr}(\mathit{Hill})$ once,
%and queries 
%$q_1 =\mathsf{hasMngr}(\mathit{Lee},\mathit{Hill})$,
%$q_2 =\exists x.\,\mathsf{Mngr}(x)$.}
%\end{proof}

The lack of a universal model suggests that CQ answering under bag semantics is
harder than in the set case. Indeed, this problem is $\conp$-hard in data 
complexity, which is in stark contrast to
the $\aczero$ upper bound in the set case.
% In this
%theorem, as well as all other theorems below, we assume that multiplicities of
%bag elements are represented in unary, or, equivalently, that bags are
%represented as sequences of possibly repeated elements.

%\begin{restatable}{proposition}{NoUnivModel}
%\label{prop:no-univ-model}
%There exists a satisfiable $\dlliteCbag$ ontology that has no universal bag model
%for the class of all CQs.
%\end{restatable}

\begin{restatable}{theorem}{QAcoNPCore}
  \label{thm:core-conph}
  % \BagCert is $\conp$-hard in data complexity for CQs and $\dlliteCbag$
  % ontologies.
  $\BagCert[\textup{CQs},\dlliteCbag]$ is $\conp$-hard in data complexity.
\end{restatable}
%\begin{proof}[Proof (sketch)]
%\newtext{
%We follow~\cite{KostylevJWS15} and reduce non 3-colourability of undirected 
%graphs to $\BagCert[\textup{CQs},\dlliteCbag]$. 
%We show that if $G = \langle V, E \rangle$ is an undirected and 
%connected graph with no self-loops, then $G$ is not 3-colourable if and only 
%if 
%$q^{\langle \T, \A_G \rangle}(\emptytpl) \ge 3\times|V| + 2$
%where $\T$ is the TBox
%$ \{ Vertex \isa \exists hasColour, \exists hasColour^- \isa ACol \}$,
%$\A_G$ is an ABox constructed based on $G$, and $q$ is the Boolean query
%$\exists x.\,\exists y.\,\exists z.\,\exists w.\,Edge(x, y) \wedge 
%hasColour(x, z) \wedge \; hasColour(y, z) \wedge ACol(w).$}
%\end{proof}

%%% Local Variables:
%%% mode: latex
%%% TeX-master: "main"
%%% End:

%!TEX root = main.tex

\section{Universal Models for Rooted Queries}
\label{sec:anch}

Theorem \ref{thm:core-conph} suggests that bag semantics
is generally not well-suited for OBDA. Our approach to overcome
this negative result
is to consider a restricted class of CQs, introduced in the context of query optimisation in 
DLs~\cite{BienvenuKR12}, called \emph{rooted}: %it is 
%required that each existentially quantified variable in a rooted CQ $q$ is 
%connected in the Gaifman graph of $q$ either to an individual or to an answer 
%variable.
in a rooted CQ, each existential
variable is connected in the Gaifman graph to an
individual or an answer variable.
Rooted CQs capture most practical queries; for example, they 
include all connected non-Boolean CQs.
\begin{definition}
  A CQ $q(\tpl x)$ is 
  \emph{rooted} if each connected component of
  its Gaifman graph has
  a node with a term in
  $\tpl x \cup \Ind$. %A UCQ is \emph{rooted} if so are all its CQs.
\end{definition}

In contrast to arbitrary %the case of
CQs, any satisfiable $\dlliteCbag$ ontology admits a universal bag model for
rooted CQs.  Although we define such a model, called \emph{canonical}, in a 
fully
declarative way, it can be intuitively seen as the result of applying a variant
of the restricted chase procedure~\cite{CaliJAIR13} extended to bags.  Starting
from the ABox, the procedure successively ``repairs'' violations of $\T$ by
extending the interpretation of concepts and roles in a minimal way.

To formalise canonical models, we need two auxiliary
notions. First, 
the \emph{concept closure} $\ccl{u}{\I}{\T}$ of an element
$u \in \Delta^\I$ in a bag interpretation $\I = \langle \Delta^{\I}, \cdot^{\I}
\rangle$ over a TBox $\T$ is the bag of concepts such that, for any concept $C$,
$\ccl{u}{\I}{\T}(C)$  is the maximum value of $C_0^\I(u)$ amongst all concepts $C_0$ satisfying  $\T \models C_0 \isa C$.
Second, the \emph{union} $\I \cup \J$ of bag interpretations 
$\I = \langle \Delta^{\I}, \cdot^{\I} \rangle$ and $\J = \langle \Delta^{\J},
\cdot^{\J} \rangle$ with $a^\I = a^\J$ for all $a \in \Ind$ is the
bag interpretation $\langle \Delta^{\I} \cup \Delta^{\J}, \cdot^{\I \cup \J}
\rangle$ with $a^{\I \cup \J} = a^\I$ for $a \in \Ind$ and $S^{\I \cup \J} =
S^\I \cup S^\J$ for $S \in \Concepts \cup \Roles$.

\begin{definition}
\label{def:canon}
The \emph{canonical bag model} $\can{\K}$ of a $\dlliteCbag$ ontology $\K = \langle \T, \A \rangle$ is the bag interpretation $\bigcup_{i \geq 0} \pcan{\K}{i}$ with  the bag interpretations $\pcan{\K}{i} = \langle \Delta^{\pcan{\K}{i}}, \cdot^{\pcan{\K}{i}} \rangle$ defined as follows:
\begin{itemize}[noitemsep,leftmargin=*,topsep=1pt]
\item[-] $\Delta^{\pcan{\K}{0}} = \Ind$, $a^{\pcan{\K}{0}} = a$ for each $a \in \Ind$, and $S^{\pcan{\K}{0}}(\tpl a) = \A(S(\tpl a))$ for each $S\in \Concepts \cup \Roles$ and individuals $\tpl a$;
\item[-] for each $i > 0$, $\Delta^{\pcan{\K}{i}}$ is
$$
  \begin{array}{@{}r}
\Delta^{\pcan{\K}{i-1}} \cup {} \{w^1_{u, R}, \ldots, w^\delta_{u, R} \mid u \in \Delta^{\pcan{\K}{i-1}}, R \textnormal{ a role}, \\ \delta = \ccl{u}{\pcan{\K}{i-1}}{\T}(\exists R) - (\exists R)^{\pcan{\K}{i-1}}(u)\},
\end{array}
$$
where $w^j_{u, R}$ are fresh domain elements, called \emph{anonymous},
$a^{\pcan{\K}{i}} = a$ for all $a \in \Ind$, and, for all $A \in \Concepts$, $P \in \Roles$, and elements $u$, $v$,
%
% $A^{\pcan{\K}{i}} (u)=\ccl{u}{\pcan{\K}{i-1}}{\T}(A)$ if
% $u \in \Delta^{\pcan{\K}{i-1}}$, and $A^{\pcan{\K}{i}} (u)=0$ otherwise, and
%
\begin{align*}
%a^{\pcan{\K}{i}} &= a, \\
A^{\pcan{\K}{i}} (u) &= \begin{cases}
\ccl{u}{\pcan{\K}{i-1}}{\T}(A), ~~\textnormal{if } u \in \Delta^{\pcan{\K}{i-1}}, 
\\ 0, ~~\textnormal{otherwise},\end{cases} \\
P^{\pcan{\K}{i}} (u, v) &= \begin{cases} P^{\pcan{\K}{i-1}} (u, 
  v), ~~\textnormal{if } u, v \in \Delta^{\pcan{\K}{i-1}}, \\ 1,
  ~~\textnormal{if } v = w^j_{u, P} \textnormal{ or } u = w^j_{v, P^-} , \\ 0,
  ~~\textnormal{otherwise}.\end{cases}
\end{align*}
\end{itemize} 
% \begin{itemize}[noitemsep,leftmargin=*]
% \item[-] $\Delta^{\pcan{\K}{0}} = \Ind$, $a^{\pcan{\K}{0}} = a$ for each $a \in \Ind$, and $S^{\pcan{\K}{0}}(\tpl a) = \A(S(\tpl a))$ for each $S\in \Concepts \cup \Roles$ and tuple $\tpl a$;
% \item[-] for each $i > 0$, $\Delta^{\pcan{\K}{i}}$ is
% $$
%   \begin{array}{@{}r}
% \Delta^{\pcan{\K}{i-1}} \cup {} \{w^1_{u, R}, \ldots, w^\delta_{u, R} \mid u \in \Delta^{\pcan{\K}{i-1}}, R \textnormal{ a role}, \\ \delta = \ccl{u}{\pcan{\K}{i-1}}{\T}(\exists R) - (\exists R)^{\pcan{\K}{i-1}}(u)\},
% \end{array}
% $$
% where $w^j_{u, R}$ are fresh domain elements, called \emph{anonymous},
% and, for each $a \in \Ind$, $A \in \Concepts$ and $P \in \Roles$, 
% \begin{align*}
% a^{\pcan{\K}{i}} &= a, \\
% A^{\pcan{\K}{i}} (u) &= \begin{cases}
% \ccl{u}{\pcan{\K}{i-1}}{\T}(A), \textnormal{if } u \in \Delta^{\pcan{\K}{i-1}}, 
% \\ 0, \textnormal{otherwise},\end{cases} \\
% P^{\pcan{\K}{i}} (u, v) &= \begin{cases} P^{\pcan{\K}{i-1}} (u, 
% v), \textnormal{if } u, v \in \Delta^{\pcan{\K}{i-1}}, \\ 1, 
% \textnormal{if } v = w^j_{u, P} \textnormal{ or } u = w^j_{v, P^-} , \\ 0, 
% \textnormal{otherwise}.\end{cases}
% \end{align*}

%\end{itemize}
\end{definition}

It is easily seen %immediate to check
that %the canonical bag model
 $\can{\K}$ satisfies %is indeed a bag model of
$\K$ whenever $\K$ is satisfiable.
% Furthermore,
We next show that it is universal for %the class of
rooted CQs.

\begin{restatable}{theorem}{RootedUniversality}
  \label{thm:chase-univ}
  The canonical bag model $\can\K$ of a satisfiable $\dlliteCbag$ ontology $\K$
  is universal for %the class of
  rooted CQs.
\end{restatable}

\begin{example}
\label{ex:mngr}
Consider an ontology $\K_{\textit{r}} = \langle \T_{\textit{r}}, \A_{\textit{r}} \rangle$ with 
$$
\begin{array}{c}
\T_{\textit{r}} =  \{\mathsf{Emp} \sqsubseteq \exists \mathsf{hasMngr}, \exists
\mathsf{hasMngr}^- \sqsubseteq \mathsf{Mngr} \},\\
\A_{\textit{r}}(\mathsf{Emp}(\mathit{Lee})) = \A_{\textit{r}}(\mathsf{Mngr}(\mathit{Hill})) = 1.
\end{array}
$$
The canonical model $\can{\K_{\textit{r}}}$ interprets (all with multiplicity 1) $\mathsf{Emp}$ by $\mathit{Lee}$,
$\mathsf{Mngr}$ by $\mathit{Hill}$ and $w^1_{\mathit{Lee},\mathsf{hasMngr}}$, and
$\mathsf{hasMngr}$ by $(\mathit{Lee}, w^1_{\mathit{Lee},\mathsf{hasMngr}})$.
Note that $\can{\K_{\textit{r}}}$ is not universal for all CQs: for instance, 
$\ans{q_{\textit{nr}}}{\can{\K_{\textit{r}}}}(\emptytpl)=2$ for non-rooted
$q_{\textit{nr}}=\exists y. \, \mathsf{Mngr}(y)$, but 
$\ans{q_{\textit{nr}}}{\I_{\textit{nr}}}(\emptytpl)=1$ for
the model $\I_{\textit{nr}}$ interpreting
$\mathsf{Emp}$ by $\mathit{Lee}$, $\mathsf{hasMngr}$ by 
$(\mathit{Lee},\mathit{Hill})$, and $\mathsf{Mngr}$ by $\mathit{Hill}$.
\end{example}

We conclude this section by showing an important
property of rooted CQs, which justifies their
favourable computational properties.
As in the set case for arbitrary CQs, 
given a satisfiable $\dlliteCbag$ ontology $\K$ and a rooted CQ $q$, 
$\cert{q}{\K}$ can be computed over a small sub-interpretation of
$\can\K$. %, whose size depends on $q$ and $\T$, but is independent from $\A$.

\begin{restatable}{theorem}{QABoundedCan}
\label{thm:qa-bounded-can}
Let $\K$ be a satisfiable $\dlliteCbag$ ontology with
$\can{\K} =
\bigcup_{i \geq 0} \pcan{\K}{i}$ and
$q$ be a rooted CQ having $n$ atoms. %in each of its CQs.
Then, $\ans{q}{\can{\K}} =
\ans{q}{\pcan{\K}{n}}$.
\end{restatable}

\section{Rewritability of Rooted Queries}
\label{sec:q-rw}

Rewritability is key for OBDA, and we next establish to what extent rooted CQs 
over bag semantics are rewritable.

The first idea would be to use the analogy with the set case and rewrite to unions of CQs.
 There are two corresponding operations for bags: max union $\cup$ and 
 arithmetic union $\uplus$. So we may consider \emph{max unions} 
 $q_{\textit{max}} = q_1(\tpl x) \lor \dots \lor q_n(\tpl x)$ or 
 \emph{arithmetic unions} $q_{\textit{ar}} = q_1(\tpl x) \veedot \cdots \veedot 
 q_n(\tpl x)$ of CQs $q_i(\tpl x)$, $1 \leq i \leq n$, with the following 
 semantics, for any interpretation $\I$: $\ans{q_{\textit{max}}}{\I} = 
 \ans{q_1}{\I} \cup \cdots \cup \ans{q_n}{\I}$ and $\ans{q_{\textit{ar}}}{\I} = 
 \ans{q_1}{\I} \uplus \cdots \uplus \ans{q_n}{\I}$, respectively.
Our first result is negative: rewriting to either of these classes is not possible even for $\dlliteCbag$.

\begin{restatable}{proposition}{CQsRewriteUCQs}
\label{prop:no-UCQ-rw}
The class of rooted CQs is rewritable neither to max nor to arithmetic unions 
of CQs for
$\dlliteCbag$.
\end{restatable}

%This result, however,
%does not rule out the possibility of a rewriting to 
%some
%other simple language. Next, 
Next we show that rooted queries are rewritable to
$\balgonee$-queries: the class directly corresponding to the algebra
$\balgonee$ for bags
\cite{GrumbachSIGACT96,GrumbachJCSS96,DBLP:journals/jcss/LibkinW97}. Since
%$\aczero \subset \balgonee \subset \logspace$ \cite{GrumbachJCSS96}, where
$\balgonee \subset \logspace$ \cite{GrumbachJCSS96}, where
$\balgonee$ is the complexity class for $\balgonee$ algebra evaluation,
rewritability to $\balgonee$-queries is highly desirable.

%We next establish the key property of rooted queries for OBDA, namely their 
%rewritability to the class of queries $\balgonee$, which directly corresponds 
%to the relational algebra for bags studied in 
%\cite{GrumbachSIGACT96,GrumbachJCSS96,DBLP:journals/jcss/LibkinW97}. However, 
%in contrast to the set case,  rooted queries are not rewritable into UCQs and 
%hence a more expressive fragment of $\balgonee$ is required.

%Thus, its evaluation problem is complete, in data complexity, for the
%corresponding complexity class $\balgonee$.

%\subsection{$\balgonee$-Queries }

Intuitively, in addition to projection $\exists$, join $\wedge$, and unions $\vee$ and $\veedot$,
$\balgonee$ also allows for 
difference $\setminus$. %, thus covering the relational bag algebra. 
Domain-dependent
queries, inexpressible in 
algebraic query languages, are precluded
by restrictions on 
the use of variables.
%As a result,
%(U)CQs are $\balgonee$-queries by definition.

\begin{definition}\label{def:BALG-queries}
  A $\balgonee$-\emph{query} $q(\tpl x)$ with \emph{answer} variables $\tpl x$ is
  one of the following, where $q_i$ are $\balgonee$-queries:
\begin{itemize}[noitemsep,leftmargin=*,topsep=1pt]
\item[-] $S(\tpl t)$, for $S\in\Concepts\cup\Roles$, $\tpl t$ tuple over $\tpl x\cup\Ind$ mentioning all~$\tpl x$;
%\item[-] $t = a$, for $t \in \Vars \cup \Ind$, $a \in \Ind$ and $\tpl x = \{t\} \cap \Vars$;
\item[-] $q_1(\tpl x_1) \wedge q_2(\tpl x_2)$, for $\tpl x = \tpl x_1 \cup \tpl x_2$;
%\item[-] $q_0(\tpl x_0) \wedge (x' = x)$, for $x' \in\Vars$, $x \in \tpl x_0$ and $\tpl x = \tpl x_0 \cup \{x'\}$;
\item[-] $q_0(\tpl x_0) \wedge (x = t)$, for $x \in \tpl x_0$, $t \in \Vars \cup \Ind$, $\tpl x = \tpl x_0 \cup (\{t\} \setminus \Ind)$;
\item[-] $\exists \tpl y. \, q_0(\tpl x, \tpl y)$; $q_1(\tpl x) \vee q_2(\tpl x)$; $q_1(\tpl x) \veedot q_2(\tpl x)$; $q_1(\tpl x) \minus q_2(\tpl x)$.
\end{itemize}
\end{definition}

The semantics of $\balgonee$-queries is defined
as follows.

\begin{definition}
  The \emph{bag answers} $\ans{q}{\I}$ to a $\balgonee$-query $q(\tpl x)$ over
  a bag interpretation $\I = \langle \Delta^{\I}, \cdot^{\I} \rangle$ is the
  bag of tuples over %of individuals from
  $\Ind$ of the same size as $\tpl x$ inductively defined as follows, for each tuple
  $\tpl a$ and the corresponding mapping $\eval$ such that
  % as $0$ if $\tpl x$ does not unify with $\tpl a$, and otherwise as follows,
  $\eval(\tpl x)=\tpl a^\I$ and $\eval(a) = a^\I$ for all $a \in \Ind$:
  \begin{itemize}[noitemsep,leftmargin=*,topsep=1pt]
  \item[-] $S^\I(\eval(\tpl t))$, if $q(\tpl x) = S(\tpl t)$;
%  \item[-] $1$ or $0$, if $q(\tpl x)$ is $t = a$ and $\eval(t) {=} \eval(a)$ or $\eval(t) {\neq} \eval(a)$, resp.;
  \item[-] $\ans{q_1}{\I}(\eval(\tpl x_1)) \times \ans{q_2}{\I}(\eval(\tpl x_2))$, if
    $q(\tpl x) = q_1(\tpl x_1) \wedge q_2(\tpl x_2)$;
  \item[-] $\ans{q_0}{\I}(\eval(\tpl x_0))$, if $q(\tpl x) = q_0(\tpl x_0) \wedge (x = t)$ and $\eval(x) = \eval(t)$;
  \item[-] $0$, if $q(\tpl x) = q_0(\tpl x_0) \wedge (x = t)$ and $\eval(x) \neq \eval(t)$;
    \item[-]  $\sum\nolimits_{\eval': \, \tpl y \to \Delta^\I} \ans{q_0}{\I}(\tpl a^\I, \eval'(\tpl y))$, if $q(\tpl x) = \exists \tpl y. \, q_0(\tpl x, \tpl y)$;
%  \item[-] $(\ans{q_1}{\I} \cup \ans{q_2}{\I})(\tpl a^\I)$ if $q(\tpl x) = q_1(\tpl x) \vee q_2(\tpl x)$; $(\ans{q_1}{\I} \uplus \ans{q_2}{\I})(\tpl a^\I)$ if $q(\tpl x) = q_1(\tpl x)\veedot q_2(\tpl x)$; and $(\ans{q_1}{\I} - \ans{q_2}{\I})(\tpl a^\I)$ if $q(\tpl x) = q_1(\tpl x) \minus q_2(\tpl x)$.
  \item[-] $(\ans{q_1}{\I} \,\mathtt{op}\, \ans{q_2}{\I})(\tpl a^\I)$ if $q(\tpl x) = q_1(\tpl x) \,\mathtt{op}'\,
    q_2(\tpl x)$, where $\mathtt{op}$ is $\cup$, $\uplus$, or $-$ and $\mathtt{op}'$ is $\vee$, $\veedot$, or $\minus$, respectively.
   \end{itemize}
\end{definition}

%Notice that all UCQs are $\balgonee$-queries by definition.
The data complexity of $\balgonee$-query evaluation
is obtained by showing that
$\balgonee$-queries can be be mapped to the 
$\balgonee$ algebra of \cite{GrumbachJCSS96}. 
%Thus, its evaluation problem is complete, in data complexity, for the
%corresponding complexity class $\balgonee$.

\begin{restatable}{proposition}{BALGComplete}
\label{prop:balg-complete}
  Given a fixed $\balgonee$-query $q(\tpl x)$, the problem of checking whether $\cert{q}{\can{\langle \emptyset, \A\rangle}}(\tpl a) \geq k$ for a bag ABox $\A$, tuple
  $\tpl a$, and $k\in\natinfz$ is
  $\aczero$ reducible to
  $\balgonee$.
\end{restatable}

%\subsection{Rewriting of CQs to $\balgonee$}

Our rewriting algorithm is inspired by the algorithm in
\cite{DBLP:conf/kr/KikotKZ12} for the set case of $\dlliteR$.
Before going into details, we provide a high-level description.

The key observation is that the set of valuations of a CQ $q(\tpl x) = \exists \tpl y. \, \phi(\tpl x, \tpl y)$ over the bag canonical 
model $\can{\K}$ can be partitioned into subsets, each of which 
is characterised by variables $\tpl z \subseteq \tpl y$ that are 
sent to anonymous elements of $\can{\K}$. Hence, we can rewrite $q(\tpl x)$ for
each of these subsets separately and then take an arithmetic union of the resulting queries, provided these queries are guaranteed to give the same answers as the corresponding subsets of valuations.

Our rewriting proceeds along the following steps.

\noindent \emph{Step 1.} First, each $\tpl z$ is checked %whether it is 
% \emph{realisable},
for \emph{realisability}, that is, whether the subquery induced by $\tpl z$ can
indeed be folded into the anonymous forest-shaped part of $\can{\K}$.  This can
be done without the ABox, looking only at the atoms of $q$ that link $\tpl z$
to other terms of $q$ (these linking atoms exist because $q$ is
rooted). Non-realisable $\tpl z$ can be disregarded.

\noindent \emph{Step 2.} For every realisable $\tpl z$, CQ $q(\tpl x)$ is 
replaced (for this $\tpl z$ in the arithmetic union) by a CQ $q_{\tpl z}(\tpl x)$ obtained 
from $q$ by replacing each maximal connected component of the subquery induced by $\tpl z$ by just one linking atom. 
This transformation is equivalence-preserving, because the 
anonymous part of $\can{\K}$ does not involve multiplicities
other than 0 and 1.

\noindent \emph{Step 3.} Finally, each resulting $q_{\tpl z}(\tpl x)$ is 
rewritten to a $\balgonee$-query $\bar q_{\tpl z}(\tpl x)$ by ``chasing back'' 
each unary atom and each binary atom mentioning a variable in $\tpl z$ with the 
TBox; for the binary atoms it is also guaranteed, by means of difference, that 
the variable in $\tpl z$ is indeed mapped to the anonymous part, thus avoiding 
double-counting in the arithmetic union.

%Given a CQ $q$, a subset $\tpl z$ of non-answer variables of $q$, and an ontology $\K$, let $\ansy{q}{\tpl z}{\can{\K}}$ be the subbag of $\ans{q}{\can{\K}}$ that is generated by valuations mapping all variables in $\tpl z$ to anonymous elements and the rest to individuals.

%\newpage
For the rest of this section, let us fix a rooted CQ $q(\tpl x) = \exists
\tpl y.
\, \phi(\tpl x, \tpl y)$ and a $\dlliteCbag$ TBox $\T$. We start by formalising
Step 1.

\begin{definition}
\label{def:eval-nulls}
  Given an ontology $\K$ with a TBox $\T$ and
  variables $\tpl z \subseteq \tpl y$, let $\ansy{q}{\tpl z}{\can{\K}}$ be the bag of tuples over $\Ind$ such that, for each tuple $\tpl a$ of individuals,
  \[
    \ansy{q}{\tpl z}{\can{\K}}(\tpl a) \ = \ \sum\nolimits_{\eval \in \Lambda_{\tpl z}} \ \prod\nolimits_{S(\tpl t) \textnormal{ in } \phi(\tpl x, \tpl y)}
    S^{\can{\K}}(\eval(\tpl t)),
  \]
  where $\Lambda_{\tpl z}$ is the set of valuations
  $\eval: \tpl x \cup \tpl y \cup \Ind \to \Delta^{\can{\K}}$ such that
  $\eval(\tpl x) = \tpl a$, $\eval(a) = a$ for each $a \in \Ind$, 
  $\eval(x) = \eval(t)$ for each $x = t$ in $\phi(\tpl x, \tpl y)$, $\eval(z)$ 
  is an anonymous element for each $z \in \tpl z$, and $\eval(y) \in \Ind$ for each $y \in 
  \tpl y \setminus \tpl z$.
 \end{definition}

Hence, the bag answers to $q$ can be partitioned as follows:
\begin{equation}
\label{partition_eq}
\ans{q}{\can{\K}} \ = \ \ \biguplus\nolimits_{\tpl z \subseteq \tpl y} \ansy{q}{\tpl 
z}{\can{\K}}.
\end{equation}

Variables $\tpl z \subseteq \tpl y$ are \emph{equality-consistent} if $\phi(\tpl x, \tpl y)$
has no equality 
$z = t$ 
%or $t = z$ 
with $z \in \tpl z$ and $t \notin \tpl z$.
If $\tpl z$ is not equality-consistent, then $\ansy{q}{\tpl z}{\can{\K}} 
= \emptybag$ and these $\tpl z$ can be disregarded in~\eqref{partition_eq}. Next, we show which other $\tpl z$ can be 
ignored.
\begin{definition}
\label{def:ma-connected}
Given equality-consistent $\tpl z \subseteq \tpl y$,
% Given $\tpl w \subseteq \tpl y$ without an equality $w = t$ in $\phi(\tpl x, 
% \tpl y)$ such that $w \in \tpl w$ and $t \notin \tpl w$,
variables 
$\tpl z' \subseteq \tpl z$ are \emph{maximally connected in the anonymous 
part} (\emph{ma-connected}) if $\eqc{z} \subseteq \tpl z'$ for the equivalence class $\eqc{z}$ of any $z \in \tpl z'$ and  
the equivalence classes $\eqc{\tpl z}'$ are a maximal subset of 
$\eqc{\tpl z}$ connected in the Gaifman graph of $q$ via nodes in~$\eqc{\tpl z}$. 
\end{definition}

\newcommand{\fvzp}{z''}

Next we introduce several notations for ma-connected $\tpl z' \subseteq \tpl z$ with equality-consistent $\tpl z \subseteq \tpl y$. First, let $\phi_{\tpl z'}$ be the sub-conjunction of $\phi(\tpl 
x, \tpl y)$ that consists of all atoms mentioning at least one 
variable in~$\tpl z'$ (these sub-conjunctions are disjoint for different $\tpl z'$).
Second, since $q$ is rooted,
$\phi_{\tpl z'}$ contains an atom $\alpha_{\tpl z'}$ of the form 
$P(t, z)$ or $P(z, t)$
with $z \in \tpl z'$ and 
$t \notin \tpl z$
(note that this definition may be non-deterministic).
Third, %and $\fvzp$ a fresh variable, 
let
%\[
%q^a_{\tpl z'}() =
%\exists \tpl x'.\exists \tpl z'. \exists\fvzp.\,
%       \phi_{\tpl z'} \land \, 
%       \mathsmaller{\bigwedge}\nolimits_{t \in \tpl t\cup\set{\!a\!}} (\fvzp{=}\,t) 
%       \land \,
%      \mathsmaller{\bigwedge}\nolimits_{z \in \tpl z'} (z\,{\neq}\,a)
%\]
\[
q^a_{\tpl z'}() =
       \exists \tpl x'.\, \exists \tpl z'. \;
       \phi_{\tpl z'} \land \; 
       \mathsmaller{\bigwedge}\nolimits_{t \in \tpl t_{\tpl z'}} (t = a) \; \land \; 
       \mathsmaller{\bigwedge}\nolimits_{z \in \tpl z'} (z \neq a),
\]
where $\tpl t_{\tpl z'}$ are all such terms $t$, $a$ is an individual in $\tpl 
t_{\tpl z'}$ if it 
exists or a fresh individual otherwise, and $\tpl x' = \tpl t_{\tpl z'} \cap \Vars$,
(this definition may also be non-deterministic because of $a$).
Notice that $q^a_{\tpl z'}$ is a Boolean CQ with possible equalities of individuals and inequalities, and we can define
the bag answers of such a query $q'$
over a bag interpretation $\I$ in the same way as for usual CQs in Definition~\ref{def:CQ-eval} 
with the extra requirement that each contributing valuation $\lambda$ 
should satisfy
$\lambda(x) \neq \lambda(t)$ for each 
inequality $x \not= t$ of $q'$ (and equalities of individuals are handled as usual equalities).

%$t \in \Ind \cup \tpl x \cup \tpl y \setminus \tpl w$
%Last, let  $\tpl u_{\tpl z'} = \tpl t_{\tpl z'} \cap \Vars$ and let $a$ be a constant in $\tpl t_{\tpl z'}$ or a constant not in  $\phi_{\tpl z'}$ if $\tpl t_{\tpl z'}$ does not contain constants.  We associate with a ma-connected $\tpl z'$ the CQ with inequalities. 

%Finally, for $a \in \Ind$, we define as 
%${\cal C}_\T^{\exists R}(a)$ the restriction of the 
%canonical model of ontology 
%$\K' = \pair{\T, \{ (\exists R)(a) \}}$
%to the domain that consists of $a$ and the anonymous elements $u$ defined 
%recursively as
%$w_{a, R}$ or $w_{u, S}$ for some role $S$.
%Note that $\Delta^{\pcan{\K'}{0}}$ in this case is set to 
%$\Ind \cup \{w_{a, R}\}$ and $R^{\pcan{\K'}{0}}$ to the bag containing only
%$\pair{a, w_{a, R}}$ (or $\pair{w_{a, R}, a}$ if $R$ is not atomic).

%For a rooted CQ, a ma-connected $\tpl w'$ is \emph{realisable} by a 
%$\dlliteCbag$ TBox $\T$ if, for $\alpha_{\tpl w'} = R(t, w)$, $\tpl w'_{w} = 
%\tpl w' \setminus \{w\}$ and $\tpl x_{t} = (\tpl x \cap \tpl t_{\tpl w'}) 
%\setminus \{t\}$,
%\begin{multline*}
%R(t, w), \T \models \\ \exists \tpl x_t. \, \exists \tpl w'_{w}. \, \phi_{\tpl 
%w'} \land \bigwedge_{t' \in \tpl t_{\tpl w'}} (t' = t)  \land 
%\bigwedge_{w' \in \tpl w'} (w' \neq t).
%\end{multline*}
\begin{definition}
\label{def:realisable-ma-connected}
Given equality-consistent variables $\tpl z \subseteq \tpl y$,
ma-connected $\tpl z' \subseteq \tpl z$ are \emph{realisable} by TBox $\T$ if 
$$
(q^a_{\tpl z'})^{\can{\langle \T, \A' \rangle}}(\emptytpl) \geq 1,
$$
where, for a fresh individual $b$, 
$\A'$ is the bag ABox having either only the assertion $P(a, b)$ (with multiplicity 1), when $\alpha_{\tpl z'} = P(t, z)$, or only
$P(b, a)$, when $\alpha_{\tpl z'} = P(z, t)$.
%%%% I really against using the word 'occurrence' only here and no where else in the paper. 
%$\A'$ is the bag ABox consisting of a single occurrence of assertion $P(a, b)$  or $P(b, a)$, when $\alpha_{\tpl z'} = P(t, z)$ or $\alpha_{\tpl z'} = P(z,  t)$, respectively.
%$\A'$ is the bag ABox that is not 0 for a single assertion $P(a, b)$ or $P(b, 
%a)$, when $\alpha_{\tpl z'} = P(t, z)$ or $\alpha_{\tpl z'} = P(z, t)$, 
%respectively; for this assertion $\A'$ is $1$.
\end{definition}

This definition does not depend on the choice of $\alpha_{\tpl z'}$ and $a$.
Indeed, if there are two atoms $P_1(t_1, z_1)$ and $P_2(t_2, z_2)$ satisfying the
definition of $\alpha_{\tpl z'}$, then either $P_1 = P_2$ and both pairs
$(t_1, z_1)$ and $(t_2, z_2)$ are mapped by a valuation of $q^a_{\tpl z'}$ to
the same tuple, or $\tpl z'$ are not realisable regardless of the
choice of $\alpha_{\tpl z'}$.
Similarly, if $\tpl t_{\tpl z'}$ contains two individuals $a$, $a'$, then 
$q^a_{\tpl z'}$ has the equality $a = a'$, 
%$\fvzp = a$ and $\fvzp = a'$
and hence $\tpl z'$ are not realisable regardless of this 
choice.

Intuitively, $\tpl z'$ are realisable if their corresponding subquery $q^a_{\tpl z'}$ is satisfied by 
the tree-shaped model induced by the TBox from a connection
$\alpha_{\tpl z'}$ of $\tpl z'$ and the rest of the query. 
This definition does not essentially involve multiplicities, because all tuples of anonymous
elements in the canonical model have multiplicity at most 1, and, hence, if $q^a_{\tpl z'}$ matches a part of the canonical model, it does so in a unique way.
Thus, checking realisability is decidable using
standard set-based techniques.

\begin{definition}
\label{def:realisability}
Variables\/ $\tpl z \subseteq \tpl y$ are
\emph{realisable} by TBox $\T$ if
they are equality-consistent and each non-empty ma-connected subset of\/ $\tpl z$ 
is realisable by $\T$.
\end{definition}

%We write ${\cal R}(q, \T)$ for the set of all $\tpl z \subseteq \tpl y$ realisable by $\T$.
We proceed to Step 2. For realisable $\tpl z \subseteq \tpl y$, let 
$q_{\tpl z}(\tpl x)$ be the 
CQ $\exists \tpl y'. \, \phi_{\tpl z}(\tpl x, \tpl y')$ such that $\phi_{\tpl z}(\tpl x, \tpl y')$ is obtained from $\phi(\tpl x, 
\tpl y)$ by replacing $\phi_{\tpl z'}$, for each ma-connected $\tpl z'\subseteq \tpl 
z$, with
$$
\alpha_{\tpl z'} \;\; \land \;\;\; \mathsmaller{\bigwedge}\nolimits_{y \in \tpl t_{\tpl z'} \cap \Vars, \, t \in \tpl t_{\tpl z'}} (y = t),%~\text{ where $\tpl t$ is defined as for $q^a_{\tpl z'}$}.
$$
%$$
%\alpha_{\tpl z'} \;\; \land \;\;\; \mathsmaller{\bigwedge}\nolimits_{t_1, t_2 \in \tpl t,t_1\notin\Ind} (t_1 = t_2),%~\text{ where $\tpl t$ is defined as for $q^a_{\tpl z'}$}.
%$$
where $\tpl t_{\tpl z'}$ is as in $q^a_{\tpl z'}$, and $\tpl y'$ is the subset of $\tpl y$ remaining in $\phi_{\tpl z}$.
% $q_{\tpl z}(\tpl x) = \exists \tpl y\exists\tpl u. \, \phi_{\tpl z}(\tpl x, \tpl y,\tpl u)$ be the 
% CQ with $\phi_{\tpl z}(\tpl x, \tpl y,\tpl u)$ obtained from $\phi(\tpl x, 
% \tpl y)$ by replacing $\phi_{\tpl z'}$, for each ma-connected $\tpl z'\subseteq \tpl 
% z$, with
% $$
% \alpha_{\tpl z'} \;\; \land \;\;\; \mathsmaller{\bigwedge}\nolimits_{t\in \tpl t_{\tpl z'}\cup\set{a}} (\fvzp = t)
% $$
% and where $\tpl u$ consists of variables $\fvzp$ for each $\tpl z'$.
%$$
%\alpha_{\tpl z'} \;\; \land \;\;\; \mathsmaller{\bigwedge}\nolimits_{t_1, t_2 
%\in \tpl t_{\tpl z'}} (t_1 = t_2).
%$$
In other words, $q_{\tpl z}$ contains, for each $\tpl z'$,  
just one atom $\alpha_{\tpl z'}$ and equalities identifying $t_{\tpl z'}$ instead of 
conjunction $\phi_{\tpl z'}$ in~$q$.

%\textbf{Bernardo: $\alpha_{\tpl z'}$ above seems undefined}

%\textbf{Babis: handle case where variables in equalities do not appear in atoms}

%For a rooted CQ $q(\tpl x) = \exists \tpl y. \, \phi(\tpl x, \tpl y)$ and a 
%realisable by a $\dlliteCbag$ TBox $\T$ tuple $\tpl z \subseteq \tpl y$, 
%we denote by $q_{\tpl z}(\tpl x)$, where $\tpl {\bar z} = \tpl x \cup \tpl y 
%\setminus \tpl z$, query 
%\[
%q_{\tpl z}(\tpl x) = \exists {\tpl y}.\,
%\phi_{\tpl {\bar z}} \wedge 
%\bigwedge_{\tpl z'} \Big(\alpha_{\tpl z'} \land \bigwedge_{t_1, t_2 \in \tpl 
%t_{\tpl z'}} (t_1 = t_2)\Big).
%\]

The following lemma justifies Steps 1 and 2. It says that in partitioning \eqref{partition_eq} we 
only need to iterate over tuples $\tpl z$ that are realisable by $\T$
and can also replace $q$ with $q_{\tpl z}$ for each $\tpl z$.
\begin{restatable}{lemma}{BALGRewriteLemma}
\label{lemma:rw-realisable}
For any ontology $\K$ with TBox $\T$ and $\tpl z 
\subseteq \tpl y$ with $q_{\tpl z}(\tpl x) = \exists \tpl y'. \, \phi_{\tpl z}(\tpl x, \tpl y')$,
\begin{enumerate}[noitemsep,leftmargin=*,topsep=1pt]
\item if $\tpl z$ is realisable by $\T$ then $\ansy{q}{\tpl z}{\can{\K}} = 
\ansy{q_{\tpl z}}{\tpl z\cap\tpl y'}{\can{\K}}$;
\item if $\tpl z$ is not realisable by $\T$ then $\ansy{q}{\tpl z}{\can{\K}} = 
\emptybag$.
\end{enumerate}
\end{restatable}

For Step 3, it suffices to rewrite each CQ $q_{\tpl z}(\tpl x) = 
\exists \tpl y'\!. \, \phi_{\tpl z}(\tpl x, \tpl y')$ to a $\balgonee$-query $\bar 
q_{\tpl z}(\tpl x)  = \exists \tpl y_{\tpl z}. \, \psi_{\tpl z}(\tpl x, \tpl y_{\tpl z})$, for $\tpl y_{\tpl z} = \tpl y' \setminus \tpl z$, which 
is guaranteed to give $\ansy{q_{\tpl z}}{\tpl z\cap\tpl y'}{\can{\K}}$ as the bag answers 
on the ABox in any ontology $\K$ with TBox~$\T$. 
To this end, we use the following notation: for
%$t \in \Ind \cup \tpl x \cup \tpl y \setminus \tpl z$ 
$t \in \Vars \cup \Ind$,
let $\zeta_A(t) = A(t)$ for $A \in \Concepts$, while $\zeta_{\exists P}(t) = 
\exists y.\,P(t, y)$ and $\zeta_{\exists P^-}(t) = \exists y.\,P(y, t)$ for $P 
\in \Roles$, where $y$ is a variable different from $t$.
Then, formula
$\psi_{\tpl z}(\tpl x, \tpl y_{\tpl z})$ is obtained from 
$\phi_{\tpl z}(\tpl x, \tpl y')$ 
by 
%dropping equality atoms mentioning variables $\tpl z$ and
replacing all atoms mentioning a term $t \in \Ind \cup \tpl x \cup \tpl y_{\tpl z}$ 
or a variable $z \in \tpl z$ as follows:
\begin{itemize}[noitemsep,leftmargin=*,topsep=1pt]
\item[-] each $A(t)$ with $\bigvee\nolimits_{\T \models C \isa A} \zeta_C(t)$;
\item[-] each $P(t, z)$ with $\big(\bigvee\nolimits_{\T \models C \isa \exists P} \zeta_C(t) \big) \,\backslash \, \zeta_{\exists P}(t)$; 
\item[-] each $P(z, t)$ with $\big(\bigvee\nolimits_{\T \models C \isa \exists P^-} \zeta_C(t) \big) \,\backslash \, \zeta_{\exists P^-}(t)$. 
\end{itemize}
%\begin{enumerate}
%\item each $A(t)$ such that $A \in \Concepts$ and $t \in \Ind \cup \tpl x \cup 
%\tpl y \setminus \tpl z$ with
%\[
%\bigvee_{\T \models C \isa A} q_C(t),
%\]
%\item each $P(t, z)$ and $P(z, t)$ such that $P \in \Roles$, $t \in \Ind \cup 
%\tpl x \cup \tpl y \setminus \tpl z$, and $z \in \tpl z$ with
%\[
%\Big( \bigvee_{\T \models C \isa \exists R} q_C(t) \Big) {\Big \backslash} 
%q_{\exists R}(t),
%\]
%where $R$ is $P$ and $P^-$, respectively.
%\end{enumerate}
Note that $\phi_{\tpl z}(\tpl x, \tpl y')$ does not contain any atoms of the 
form $A(z)$ for $z \in \tpl z$, so $\psi_{\tpl z}(\tpl x, \tpl y_{\tpl z})$ does not 
mention variables $\tpl z$. Also, atoms over roles without variables $\tpl z$ stay intact, because $\T$ contains no role inclusions.

Finally, the rewriting of
$q(\tpl x)$ over $\T$
is the $\balgonee$-query
\[
\bar q(\tpl x) = \bigveedot\nolimits_{\tpl z \text{ realisable by } \T} \bar q_{\tpl z}(\tpl 
x).
\]

\begin{example}
Consider TBox $\T_{\textit{r}}$ from Example~\ref{ex:mngr} and the rooted CQ
$q^{\textit{r}}(x) = \exists y.\, {\sf hasMngr}(x, y) \wedge {\sf Mngr}(y)$.
The query $\bar q^{\textit{r}}(x)  = \bar q^{\textit{r}}_{\emptytpl}(x) \veedot \bar q^{\textit{r}}_{y}(x)$, where $\bar q^{\textit{r}}_{\emptytpl}(x)$ and $\bar q^{\textit{r}}_{y}(x)$ are
\[
\begin{array}{l}
\exists y.\,
{\sf hasMngr}(x, y) \wedge 
\big( {\sf Mngr}(y) \vee \exists z.\, {\sf hasMngr}(z, y) \big) \text{ and }\\
\big( {\sf Emp}(x) \vee \exists y.\, {\sf hasMngr}(x, y) \big) 
\setminus
\exists y.\, {\sf hasMngr}(x, y),
\end{array}
\]
is a rewriting of $q^{\textit{r}}$ over $\T_{\textit{r}}$, since $\emptytpl$ and $y$ are realisable.
\end{example}

The following theorem establishes the correctness of our approach and
leads to the main rewritability result.
\begin{restatable}{theorem}{BALGRewrite}
\label{thm:cq-rw}
For any rooted CQ $q$ and $\dlliteCbag$ ontology $\K = \langle \T, \A 
\rangle$ we have that
$\ans{q}{\can\K} = \cert{\bar q}{\can{\langle\emptyset,\A\rangle}}$. 
\end{restatable}

% The following is the main result of this paper.
 
\begin{restatable}{corollary}{BALGRewriteCor}\label{ref:cor-main}
The class of rooted CQs is rewritable to $\balgonee$-queries for 
$\dlliteCbag$.
\end{restatable}

%\begin{corollary}
%Given a fixed UCQ $q$ and TBox $\T$, checking whether $\cert{q}{\langle \T, \A \rangle}(\tpl a) \geq k$ for an ABox $\A$, tuple of individuals $\tpl a$, and number $k \in \natinfz$ is in $\balgonee$.
%\end{corollary}

We conclude this section by establishing 
the complexity of rooted query answering. The bounds 
follow as an easy consequence of Theorem \ref{thm:qa-bounded-can}, Proposition~\ref{prop:balg-complete},
and Corollary \ref{ref:cor-main}.

%%%
%\subsection{Complexity of Rooted Query Answering}
%%%

\begin{restatable}{theorem}{QAComplexity}
  % \BagCert is $\np$-complete and in $\balgonee$ in data complexity for rooted 
  % (U)CQs and $\dlliteCbag$ ontologies.
  $\BagCert[\textup{rooted CQs},\dlliteCbag]$ is $\np$-complete and in
  \logspace{} in data complexity.
%  $\balgonee$ in data complexity.
\end{restatable}

%Since $\aczero \subset \balgonee \subset \logspace$ \cite{GrumbachJCSS96}, rewritability of a class of queries to $\balgonee$-queries is a highly favourable property.

%\begin{restatable}{theorem}{QADataCompl}
%Bag query answering for rooted queries and $\dlliteCbag$ ontologies is in 
%$\balgonee$ in data complexity.
%\end{restatable}

%Unfortunately, this result does not extend to $\dlliteRbag$:

However, 
the next theorem implies
%\newtext{
%employing a similar proof to the one used for Theorem~\ref{thm:core-conph},
%we can prove the following theorem, which states} 
that rooted queries are not 
$\balgonee$-rewritable for
unrestricted $\dlliteRbag$ TBoxes.
\begin{restatable}{theorem}{QARolescoNP}
  \label{thm:role-conph}
  % \BagCert is $\conp$-hard in data complexity for rooted CQs and
  % $\dlliteRbag$ ontologies.
  $\BagCert[\textup{rooted CQs},\dlliteRbag]$ is $\conp$-hard in data 
  complexity.
%  Problem $\mathsf{certBag}(q,\K,\eval,m)$ is $\textsc{coNP}$-hard in data 
%complexity for $\K$ in $\dlliteRbag$, even if $q$ is rooted and $m$ is encoded 
%in unary.
\end{restatable}
%\newtext{
%We stress that the complexity results above are 
%dependent on the unary encoding of the input to problem 
%$\BagCert[\textup{rooted CQs},\dlliteCbag]$, that is, 
%the multiplicity of the ABox assertions and number $k$. 
%Encoding the former multiplicities in unary is reasonable from an 
%implementation's point of view as~\cite{GrumbachJCSS96} argue, however, it 
%would be interesting to see how complexity is affected when $k$ is encoded in 
%binary.
%}

%%% Local Variables:
%%% mode: latex
%%% TeX-master: "main"
%%% End:

%!TEX root = main.tex
\section{Related work}
\label{sec:related}

Query answering under bag semantics has
received significant attention 
in the database literature \cite{DBLP:conf/pods/LibkinW94,GrumbachSIGACT96,GrumbachJCSS96,DBLP:journals/jcss/LibkinW97}.
These works study the relative expressive power of bag algebra primitives, the relationship
with set-based algebras, and establish the
data complexity of query answering. Such problems have also been recently studied in the setting
of Semantic Web and SPARQL 1.1 in \cite{DBLP:conf/www/KaminskiKG16,DBLP:conf/semweb/AnglesG16}.

Bag semantics 
in the context of  Description Logics has been studied in~\cite{DLBag10}, where 
the author proposes a bag semantics for $\mathcal{ALC}$ and provides a tableaux 
algorithm. 
%Their semantics is, however, incomparable to ours. 
 %relies on interpretations with a bag domain assigning to atomic 
%concepts and roles a subbag of the domain. 
%Hence, contrary to our semantics, the maximum multiplicity of an element in a 
%concept is determined by the domain.
%Another important difference, which makes the semantics of \cite{DLBag10} 
%unintuitive, is that concepts of the form $\exists R.\top$ are 
%interpreted as the bag of all occurrences of a domain element related through 
%$R$ with all occurrences of some other domain element.
%Hence, the two approaches are incomparable.
In contrast to our work, their results 
are restricted to ontology satisfiability and do not
encompass CQ answering.

CQ answering under bag semantics is closely related to 
answering $\sf Count$ aggregate queries. The semantics of aggregate queries for database
settings with incomplete information, such as inconsistent databases and data
exchange, have been studied in \cite{ArenasTCS03,DBLP:conf/pods/Libkin06,KolaitisPODS08}. As pointed out in \cite{KostylevJWS15}, these
 techniques are not directly
applicable to ontologies. 
The practical solution in~\cite{CalvaneseONISW08} is to give epistemic 
semantics to
aggregate queries, where the query is evaluated over
ABox facts entailed by the ontology; thus, 
the anonymous part of the
ontology models is essentially ignored, and
the
semantics easily leads to counter-intuitive answers. 
To remedy these issues, \cite{KostylevJWS15}
propose a certain answer semantics for $\sf Count$ aggregate queries over
ontologies and prove tight complexity bounds  for
$\dlliteR$ and $\dlliteC$. Similarly to our work, their
semantics is open-world and
considers all models of the ontology for query evaluation, which leads to more intuitive answers. 
The main difference resides in the definition of
the ontology language, where they consider 
set ABoxes and adopt conventional
set-based semantics for TBox axioms. 
Although $\dlliteRbag$ is closely related to the logic in~\cite{KostylevJWS15}, 
the two settings do not coincide even for set ABoxes.
For example, if $\A$ comprises only assertions $R(a, b)$ and $R(a, c)$ and 
$\T$ comprises axiom $\exists R \sqsubseteq B$,
then the query over $\langle \T, \A \rangle$ that counts the number of 
individuals $a$ in concept $B$ returns $1$ in the setting 
of~\cite{KostylevJWS15}, while the corresponding 
$\dlliteRbag$ query returns~$2$.

\section{Conclusion and Future Work}

We have studied OBDA under bag semantics and
identified a general class of rewritable queries over
$\dlliteCbag$ ontologies. 
As our framework covers already the class of $\sf Count$ aggregate queries,
in future work we plan to extend it to capture further aggregate functions 
and more expressive ontologies.

%We see many challenges for future work.
%First, role inclusions make query answering
%more involved; hence, we shall restrict
%rooted queries further to ensure rewritability for $\dlliteRbag$.
%Second, although our results already cover $\sf Count$
%aggregate queries, we shall extend them to cover
%also other aggregate functions and study implications
%on complexity and rewritability. Finally, we shall 
%design and implement a practical rewriting 
%algorithm based on our 
%results.

\label{sec:concl}

\clearpage
\bibliographystyle{named}
\bibliography{refs}

%!TEX root = main.tex

\onecolumn

\appendix 

\section{Appendix}

In this appendix we give the complete proofs omitted in the main part of the paper.

\smallskip

\BackCompat*

\begin{proof}
Let $\K = \langle \T, \A \rangle$ and
$\K' = \langle \T, \A' \rangle$ for any $\A'$ satisfying 
requirement $\{S(\tpl t) \mid \A'(S(\tpl t)) \ge 1 \} = \A$. 

1.
First assume that $\K$ has a model 
$\I = \langle \Delta^{\I}, \cdot^{\I} \rangle$. We prove that there exists a 
bag model of $\K'$.
To this end, consider the bag interpretation $\I' = 
\langle \Delta^{\I}, \cdot^{\I'} \rangle$ such that, for any $u, v \in 
\Delta^{\I}$ and $a \in \Ind$, 
$$
\begin{array}{rcl}
a^{\I'} & = & a^\I, \\
A^{\I'}(u) & = & \left\{\begin{array}{ll} \infty, & \text{if } u \in A^\I, \\ 0, & \text{otherwise,}\end{array}\right. \\
P^{\I'}(u, v) & = & \left\{\begin{array}{ll} \infty, & \text{if } \pair{u,v} 
\in P^\I, \\ 0, & \text{otherwise}.\end{array}\right.
\end{array}
$$
Bag interpretation $\I'$ satisfies $\A'$ and all axioms in $\T'$, so it is a 
bag model of $\K'$. Therefore, $\K'$ is satisfiable, as required. 

To complete the proof of statement 1, suppose that $\K'$ 
has a bag model $\I' = \langle \Delta^{\I'}, \cdot^{\I'} \rangle$. We construct 
an interpretation $\I = \langle \Delta^{\I'}, \cdot^{\I} \rangle$ of 
$\K$ in a similar way. For $u, v \in \Delta^{\I}$ and $a \in \Ind$, let
$$
\begin{array}{rcl}
& a^{\I} ~~=~~ a^{\I'}, \\
u \in A^\I & \text{if and only if} & A^{\I'}(u) > 0, \\
\pair{u,v} \in P^\I & \text{if and only if} & P^{\I'}(u, v) > 0.
\end{array}
$$
Same as in the previous case, $\I$ is a model of $\K$.

2.
For the forward direction, let $\tpl a \in \cert{q}{\K}$ 
for a tuple of individuals $\tpl a$, but, for the sake of 
contradiction, $\cert{q}{\K'}(\tpl a) = 0$. The 
latter means that there exists a bag model $\I'$ such that $\cert{q}{\I'}(\tpl 
a) = 0$. Consider the interpretation $\I$ constructed on the base 
of $\I'$ as in the second part of the proof of statement 1. On the one hand, it 
is a model of $\K$. On the other, $\I \not \models q(\tpl 
a)$ by construction. However, it contradicts the fact that $\tpl a \in 
\cert{q}{\K}$. Therefore, our assumption was wrong 
and $\cert{q}{\K'}(\tpl a) \geq 1$.

For the backward direction, we proceed similarly. 
For this let $\tpl a$ be a tuple of individuals and assume that
$\cert{q}{\K'}(\tpl a) \ge 1$ holds but $\tpl a \not\in \cert{q}{\K}$.
The latter implies that $\K$ has a model $\I$ such that $\I \not\models q(\tpl 
a)$. But this means that the model $\I'$ of $\K'$ constructed in the proof of 
statement 1.\ on the basis of $\I$ is such that $q^{\I'}(\tpl a) = 0$, which 
contradicts our assumption that $\cert{q}{\K'}(\tpl a) \ge 1$.
\end{proof}

\NoUnivModel*
\begin{proof}
Consider a variant of our running example where
$\T = \{\mathsf{Emp} \sqsubseteq \exists \mathsf{hasMngr}, \exists
\mathsf{hasMngr}^- \sqsubseteq \mathsf{Mngr} \}$ and $\A$ contains
$\mathsf{Emp}(\mathit{Lee})$ and $\mathsf{Mngr}(\mathit{Hill})$ once.
Consider bag interpretations $\I_1, \I_2$ defined as
\[
\begin{array}{r@{~} l r@{~} l r@{~} l r@{~} l r@{~} l r@{~} l}
\Delta^{\I_1} &= \{\mathit{Lee}, \mathit{Hill}\},&
\mathsf{Emp}^{\I_1} &= \{\mathit{Lee}\},&
\mathsf{hasMngr}^{\I_1} &= \{ (\mathit{Lee},\mathit{Hill})\},&
\mathsf{Mngr}^{\I_1} &= \{\mathit{Hill}\}\\
\Delta^{\I_2} &= \{\mathit{Lee}, \mathit{Hill}, w\},&
\mathsf{Emp}^{\I_2} &= \{\mathit{Lee}\}, &
\mathsf{hasMngr}^{\I_2} &= \{ (\mathit{Lee}, w) \}, &
\mathsf{Mngr}^{\I_2} &= \{\mathit{Hill}, w\}.\\
\end{array}
\]
%$\Delta^{\I_1} = \{\mathit{Lee}, \mathit{Hill}\}$ and
%$\Delta^{\I_2} = \Delta^{\I_1} \cup \{w\}$; also, let $\I_1,\I_2$ assign
%predicates to sets as follows:
%$\mathsf{Emp}^{\I_1} = \mathsf{Emp}^{\I_2} = \{\mathit{Lee}\}$,
%$\mathsf{hasMngr}^{\I_1} = \{ \langle \mathit{Lee},\mathit{Hill}
%\rangle\}$,
%$\mathsf{hasMngr}^{\I_2} = \{ \langle \mathit{Lee}, w \rangle \}$,
%$\mathsf{Mngr}^{\I_1} = \{\mathit{Hill}\}$ and
%$\mathsf{Mngr}^{\I_2} = \{\mathit{Hill}, w\}$. 
Both $\I_1$ and $\I_2$ are
bag models of $\K= \langle \T, \A \rangle$.
Moreover, for $q_1 =\mathsf{hasMngr}(\mathit{Lee},\mathit{Hill})$ and
$q_2 =\exists x.\,\mathsf{Mngr}(x)$, we have
$\ans{q_1}{\I_1}(\emptytpl)=1$,
$\ans{q_1}{\I_2}(\emptytpl)=0$,
$\ans{q_2}{\I_1}(\emptytpl)=1$, and $\ans{q_2}{\I_2}(\emptytpl)=2$;
thus, neither model is universal for both $\set{q_1,q_2}$. Suppose there
is a universal model $\I$ for $\set{q_1,q_2}$. Then, since
$\ans{q_1}{\I}(\emptytpl)$ must be zero,
$(\mathit{Lee},\mathit{Hill})$ does not occur
in $\mathsf{hasMngr}^{\I}$;
since $\mathsf{Emp}(\mathit{Lee})$ is an assertion of $\A$ and
$\mathsf{Emp} \sqsubseteq \exists \mathsf{hasMngr} \in\T$, we have
$\langle \mathit{Lee},w'\rangle\in\mathsf{hasMngr}^{\I}$ for some
$w'\in\Delta^{\I}$ distinct from $\mathit{Hill}$; since
$\exists \mathsf{hasMngr}^- \sqsubseteq \mathsf{Mngr}\in\T$, it follows
$w'\in\mathsf{Mngr}^{\I}$, and hence $\ans{q_2}{\I}(\emptytpl)\ge 2$,
contradicting universality of $\I$.
\end{proof}

Following the unary representation of bags~\cite{GrumbachJCSS96}, we represent 
a bag $\Omega$ over a set $S$ using expression $\bagl \cdot \bagr$ within which 
we repeat all elements of $S$ as many times as their multiplicity in $\Omega$. 
For convenience, we shall also write a bag $\bagl a, a, b, b, b \bagr$ in the 
more compressed form $\bagl a, b\bagr_{2,3}$ where instead of repeating an 
element, we list a single occurrence and denote its multiplicity with a number 
in the appropriate subscript position of that bag. 

\QAcoNPCore*
\begin{proof}
We prove that there exists a $\dlliteCbag$ TBox $\T$ and a Boolean CQ $q$ such 
that checking whether $q^{\langle \T, \A \rangle}(\emptytpl) \ge k$ for 
an input 
bag ABox 
$\A$ and $k \in \natinfz$ is $\conp$-hard.
To prove this claim, we follow~\cite{KostylevJWS15} and reduce non 
3-colourability of undirected 
graphs (a \conp-complete problem) to query answering over $\dlliteCbag$ 
ontologies. We show that if $G = \langle V, E \rangle$ is an undirected and 
connected 
graph with no self-loops, then $G$ is not 3-colourable if and only if 
$q^{\langle \T, \A_G \rangle}(\emptytpl) \ge 3\times|V| + 2$
where $\T$ is the TBox
$ \{ Vertex \isa \exists hasColour, \exists hasColour^- \isa ACol \}$,
$\A_G$ is an ABox constructed based on $G$, and $q$ is the Boolean query
\[
q() = \exists x.\,\exists y.\,\exists z.\,\exists w.\; Edge(x, y) \wedge 
hasColour(x, z) \wedge \; hasColour(y, z) \wedge ACol(w).
\]
 
Let $\Ind \supseteq V \cup \{a, r, g, b\}$. ABox $\A_G$ is defined so that it 
contains the following assertions:
\begin{itemize}
\item[--]
$Vertex(u)$ for each $u \in V$,
 
\item[--] 
$Edge(u, v)$, $Edge(v, u)$ for each $\pair{u, v} \in E$,
 
\item[--] 
$ACol(r)$ ($|V| + 1$ times), $ACol(g)$ ($|V|$ times), $ACol(b)$ ($|V|$ times), 
for colours $r$, $g$, $b$, and
 
\item[--] 
$Vertex(a)$, $Edge(a, a)$, and $hasColour(a, r)$, for the auxiliary vertex 
$a$.
\end{itemize}

Individual $a$ corresponds to an auxiliary vertex for the purposes 
of the reduction, whereas individuals $r$, $g$, and $b$ play the role of 
colours.
The usage of $Vertex$ and $Edge$ is clear; they encode $G$. 
Role $hasColor$ plays the role of a colour assignment to the vertices of $G$; 
this is also imposed by axiom 
$Vertex \isa \exists hasColour$.
Concept $ACol$ provides a sufficient number of pre-defined colour copies that 
favours 3-colour assignments based on the colours $r$, $g$, and $b$. Any proper 
assignment of $G$ shall use at most $|V|$ times each one of these colours. 
However, if any assignment is not proper and exhausts the number of available 
colours (i.e., by assigning multiple colours to the same vertex) or uses an 
additional colour, these will have to be added to concept $ACol$ due to the 
axiom 
$\exists hasColour^- \isa ACol$,
effectively increasing its minimum cardinality. This behaviour is the one that 
we exploit in the following reduction.

We next show that 
$G$ is not 3-colourable if and only if 
$q^{\langle \T, \A_G \rangle}(\emptytpl) \ge 3\times|V| + 2$.

``$\Rightarrow$''
Let $G$ be non-3-colourable. 
Consider a model $\I$ of $\langle \T, \A_G \rangle$ (which exists since 
$\langle \T, \A_G \rangle$ is satisfiable) such that, if 
$\gamma: V \to \{r, g, b\}$ 
is an assignment of colours to the vertices of $G$ and $u \not= a$, then
$hasColour^\I(\pair{u^\I, c^\I}) = 1$ if and only if 
$\gamma(u) = c$ with $c \in \{r, g, b\}$.
Since $G$ is not 3-colourable, then, for all assignments $\gamma$, there exists 
at least an edge $\pair{u, v} \in E$ with $\gamma(u) = \gamma(v) = c$. 
Consequently, for all models $\I$ defined on the basis of $\gamma$, 
$hasColour^\I$ contains tuples
$\pair{u^\I, c^\I}$ and $\pair{v^\I, c^\I}$, and hence, the subquery of $q$
\[
q_1(x,y, z) = 
Edge(x, y) \wedge hasColour(x, z) \wedge hasColour(y, z)
\]
has at least two matches, each one contributing multiplicity $1$; one match 
corresponds to valuation 
$\{ x/u^\I, y/v^\I, z/c^\I \}$
and one to valuation
$\{ x/a^\I, y/a^\I, z/r^\I \}$.
Observe also that atom $ACol(w)$ contributes at least multiplicity 
$3 \times |V| + 1$.
%whenever these models do not exhaust the colours available in $ACol$, i.e., do 
%not assign multiple colours to vertices so that the number of available copies 
%of a colour in $ACol$ is exhausted.
Therefore, 
$q^\I(\emptytpl) \ge 2\times(3\times|V| + 1)$ 
for every model $\I$ following a proper 3-colour assignment, and hence, 
$3\times|V| + 2$ is a certain multiplicity with respect to all these models, as 
required. 
Clearly, the same statement holds for all of 
the models that add additional elements in $Vertex$, $Edge$, or assign multiple 
colours to some vertices exceeding the number of available colours. 
What is left to consider is those models that assign additional colours to 
vertices and not just one among $r$, $g$, and $b$. For such colour assignments, 
$G$ might turn out to be colourable. Suppose $G$ is 4-colourable (if it is not, 
then the above discussion carries over) and let $p \in \Ind$. 
Then, there 
exists a model that follows a 4-colour assignment 
$\gamma: V \to \{r, g, b, p\}$ such that 
$\gamma(u) \not= \gamma(v)$ for every 
$\pair{u, v} \in E$. 
Therefore, for that model we would get one match with multiplicity $1$ for 
subquery $q_1(x, y, z)$, that is, for valuation $\{ x/a^\I, y/a^\I, z/r^\I \}$).
%\[
%%q_1()=
%%\exists x,y,z.\;
%Edge(x, y)\wedge hasColour(x, z) \wedge hasColour(y, z).
%\]
On the other hand, given the observations above, that model would have to 
include element $p$ in the extension of $ACol$ at least once, effectively 
increasing the cardinality of $ACol$ to $3\times|V|+2$. Therefore, the 
evaluation of $q$ over that model would always give at least $3\times|V| + 2$ 
empty tuples. 
Clearly, the same holds for models that make use of further colours.
Therefore, $q^{\langle \T, \A_G \rangle}(\emptytpl) \ge 3\times|V|+2$.

``$\Leftarrow$'' 
Let $G$ be 3-colourable. It suffices to show that there exists a model $\I$ for 
which $q^\I(\emptytpl) = m$ with $m < 3\times|V| + 2$.
Since $G$ is 3-colourable, there is an assignment 
$\gamma: V \to \{r, g, b\}$ 
such that, for every 
$\pair{u, v} \in E$, $\gamma(u) \not= \gamma(v)$.
Consider an interpretation $\I_\gamma$ defined as follows:
\begin{align*}
\Delta^{\I_\gamma} & = \{ d_c \mid c \in V \cup \{ a, r, g, b \} \}, \\
c^{\I_\gamma} & = d_c \text{, for } c \in V \cup \{ a, r, g, b \},\\ 
Vertex^{\I_\gamma} & = \{ d_c \mid c \in V \cup \{ a \}\}, \\
Edge^{\I_\gamma}   & =
   \{\pair{d_u, d_v}, \pair{d_v, d_u} \mid \pair{u, v} \in E  \cup
   \{\pair{a, a}\}\},\\
hasColour^{\I_\gamma} & = 
\{ \pair{d_u, d_{\gamma(u)}} \mid u \in V \} \cup \{ \pair{d_a, d_r} \},\\
ACol^{\I_\gamma} & = \bagl d_r, d_g, d_b \bagr_{|V| + 1, |V|, |V|}.
\end{align*}
Interpretation $\I_\gamma$ is defined based on the contents of $V$, $E$, and 
the 3-colour assignment $\gamma$.
It is easy to verify that $\I_\gamma$ is a model of $\langle \T, \A_G \rangle$.
Next, we show that 
$q^{\I_\gamma}(\emptytpl) = 3\times|V| + 1$. 
First, we observe that subquery $q_1(x, y, z)$
%\[
%\exists x,y,z.\;Edge(x, y) \wedge hasColour(x, z) \wedge hasColour(y, z)
%\]
matches exactly once (i.e., under valuation $\{ x/d_a, y/d_a, z/d_r \}$). 
This holds because $\gamma$ is a proper 3-colouring of $G$ and, for every 
$\pair{u, v} \in E$, $\gamma(u) \not= \gamma(v)$.
Note also that there are three valuations for atom $ACol(w)$
contributing multiplicity $3\times|V| + 1$ in total. Consequently,
$q^{\I_\gamma}(\emptytpl) = 3\times|V| + 1$, as desired.

\end{proof}

\newtheorem{remark}{Remark}
\begin{remark}
\label{rem:una-conp-core}
When the UNA is dropped, we can modify the definition of ABox satisfaction and 
show that a similar reduction holds for establishing $\conp$-hardness of query 
answering for $\dlliteCbag$ ontologies. Under this new definition, a bag 
interpretation $\I$ \emph{satisfies} an ABox $\A$ if:
\begin{itemize}
\item[--] 
for each concept assertion $A(a)$ in $\A$, we have
$
\sum_{a_0 \in \Ind: a_0^\I = a^\I} \A(A(a_0)) \leq A^{\I}(a^\I)
$, and

\item[--]
for each role assertion $P(a,b)$ in $\A$, we have
$
\sum_{a_0, b_0 \in \Ind: a_0^\I = a^\I, b_0^\I = b^\I} \A(P(a_0,b_0)) 
\leq P^{\I}(a^\I, b^\I).
$
\end{itemize}
Observe that under the UNA, the definition of ABox satisfaction 
(Definition~\ref{def:bag-model}) is a special case of the above, hence, 
Theorem~\ref{thm:core-conph} is still valid under this new definition. 
We now discuss the modifications that are necessary for reducing 
non-3-colourability of undirected and connected graphs without self-loops to
query answering in $\dlliteCbag$ ontologies without making the UNA. 
For this, we need to make sure that the auxiliary vertex $a$ is not interpreted 
with the same element with any of the vertices of $G$ as well as that none of 
the colours $r, g, b$ are interpreted by the same element. To ensure this, we 
employ atomic concepts $V_a$, $V_G$, $Red$, $Blue$, and $Green$ which will hold 
the auxiliary vertex $a$, the vertices of $G$, and the three colours,
respectively. Then, we make sure that no interpretation mixes their role by 
introducing pairwise disjointness axioms:
${\sf Disj}(Red, Blue)$, ${\sf Disj}(Red, Green)$, ${\sf Disj}(Blue, Green)$, 
and ${\sf Disj}(V_a, V_G)$. Last, we modify $\A_G$ to have the additional 
assertions
$V_a(a)$, $Red(r)$, $Green(g)$, $Blue(b)$, and $V_G(u)$, for every vertex $u 
\in V$.
Following exactly the argumentation used in Theorem~\ref{thm:core-conph}, we 
can show that the above reduction works if the UNA is dropped.
\end{remark}

\bigskip

\newcommand{\nelem}[2]{[#1{:}\,#2]}
\newcommand{\enum}[1]{{#1}^\texttt{e}}
\newcommand{\enumcan}[1]{\cansymb^\texttt{e}(#1)}
\newcommand{\enumpcan}[2]{\cansymb_{#2}^\texttt{e}(#1)}

An \emph{enumerated bag} (\emph{e-bag}, for short) $\Theta$ over a set $M$ is a set of pairs $\nelem{c}{m}$ with $c \in M$ and $m \in \nati$, where $\nati$ is the set of positive integers, such that if $\nelem{c}{m} \in \Theta$ then $\nelem{c}{m-1} \in \Theta$ for all $m \in \nati$. There is a straightforward one-to-one correspondence between bags and e-bags, and we denote $\enum{\Omega}$ the enumerated version of a bag $\Omega$. This notion generalises to bag interpretations: the \emph{e-bag interpretation} $\enum{\I}$ corresponding to a bag interpretation $\I = \langle \Delta^{\I}, \cdot^{\I} \rangle$ is the pair $\langle \Delta^{\I}, \cdot^{\enum{\I}} \rangle$ such that $a^{\enum \I} = a^\I$ for each individual $a$ and $S^{\enum{\I}} = \enum{(S^\I)}$ for any $S \in \mathbf C \cup \mathbf R$. The interpretation function extends to inverse roles in the same way.

An \emph{enumerated homomorphism} (\emph{e-homomorphism}) from an e-bag interpretation $\enum{\I} = \langle \Delta^{\I}, \cdot^{\enum{\I}} \rangle$ to an e-bag interpretation $\enum{\J} = \langle \Delta^{\J}, \cdot^{\enum{\J}} \rangle$ is a family $(h, h_S, \ldots)$, $S \in \mathbf C \cup \mathbf R$, of functions
$$
\begin{array}{rll}
h: & \Delta^{\I} \to \Delta^{\J}, \\
h_S: & S^{\enum{\I}} \to S^{\enum{\J}}, & \text{for all } S \in \mathbf C \cup \mathbf R,
\end{array}
$$
such that 
\begin{itemize}
\item[--] $h(a^{\enum{\I}}) = a^{\enum{\J}}$ for each $a \in \Ind$,
\item[--] $h_A(\nelem{u}{m}) = \nelem{h(u)}{\ell}$ for all $A \in \Concepts$ and $\nelem{u}{m} \in A^{\enum{\I}}$, where $\ell\in \nati$ is some number (which can be different for different $A$ and $\nelem{u}{m}$),
\item[--] $h_P(\nelem{\pair{u, v}}{m}) = \nelem{\pair{h(u), h(v)}}{\ell}$ for 
all $P \in \Roles$ and $\nelem{\pair{u, v}}{m} \in P^{\enum{\I}}$, where 
$\ell\in 
\nati$ is some number.
\end{itemize}
To handle some cases uniformly, we sometimes write $h_{P^-}(\nelem{\pair{v, 
u}}{m})$ instead of $h_{P}(\nelem{\pair{u, v}}{m})$, for $P \in \mathbf R$.

Intuitively, an e-homomorphism is a usual homomorphism that additionally establishes correspondence for each enumerated tuple of elements in each relation in $\enum\I$.

An e-homomorphism $(h, h_S, \ldots)$ 
from $\enum{\I} = \langle \Delta^{\I}, \cdot^{\enum{\I}} \rangle$ to $\enum{\J} = \langle \Delta^{\J}, \cdot^{\enum{\J}} \rangle$
is \emph{predicate-injective on individuals} $\Ind$ if, for each $u$ such that there exists $a\in \Ind$ with $h(u) = a^{\enum{\J}}$,
\begin{itemize}
\item[--] $h_A(\nelem{u}{m}) \neq h_A(\nelem{u}{\ell})$ for all $A \in \Concepts$ and all $\nelem{u}{m}, \nelem{u}{\ell} \in A^{\enum{\I}}$ with $m \neq \ell$,
\item[--] $h_R(\nelem{\pair{u, v_1}}{m}) \neq h_R(\nelem{\pair{u, v_2}}{\ell})$ 
for 
all roles $R$ and all $\nelem{\pair{u, v_1}}{m}, \nelem{\pair{u, v_2}}{\ell} 
\in 
R^{\enum{\I}}$ with $v_1 \neq v_2$ or $m \neq \ell$.
%\item[--] $h^R(\nelem{(d_1, d)}{m}) \neq h^R(\nelem{(d_2, d)}{\ell})$ for all $R \in \Roles$ and all $\nelem{(d_1, d)}{m}, \nelem{(d_2, d)}{\ell} \in R^{\enum{\I}}$ with $d_1 \neq d_2$ or $m \neq \ell$.
\end{itemize}

%Essentially, an e-homomorphism behaves as a bag homomorphism on (the 
%interpretations of) individuals and as a usual homomorphism everywhere else.

\begin{lemma}
\label{lemma:e-homomorphism}
For any $\dlliteCbag$ ontology $\K$ and any bag model $\I$ of $\K$ there exists an e-homomorphism from $\enumcan{\K}$ to $\enum{\I}$ that is predicate-injective on $\Ind$.
\end{lemma}

\begin{proof}
Let $\can{\K} = \bigcup_{i \geq 0} \pcan{\K}{i}$ with $\pcan{\K}{i} = \langle \Delta^{\pcan{\K}{i}}, \cdot^{\pcan{\K}{i}} \rangle$.
We first define a witnessing predicate-injective e-homomorphism $(h, h_S, 
\ldots)$ for the elements in $\Delta^{\pcan{\K}{0}}$, that is, on the 
(interpretations of the) individuals, then extend it to elements introduced in 
${\pcan{\K}{1}}$, and finally recursively define it on all other elements. 

For the first step, consider an individual $a \in \Ind$ and the element $u = 
a^{\pcan{\K}{0}}$. We set $h(u) = a^\I$. Then, consider any atomic concept $A$ 
such that $A^{\pcan{\K}{0}}(u) = k$, $k \in \nati$, that is, such that 
$\nelem{u}{m} \in A^{\enumpcan{\K}{0}}$ for all $m \in \nati$ with $m \leq k$. 
By the definition of $\pcan{\K}{0}$, 
$\A(A(a)) = k$. Since $\I$ is a model of $\A$, we have that $A^\I(a^\I) \geq 
k$. In other words, $\nelem{h(u)}{m} \in A^{\enum{\I}}$ for all $m \leq k$, and 
we can set $h_A(\nelem{u}{m}) = \nelem{h(u)}{m}$ for all $m$.
Consider now individuals $a, b \in \Ind$ with corresponding elements $u = 
a^{\pcan{\K}{0}}$ and $v = b^{\pcan{\K}{0}}$ and an atomic role $P$ such that 
$P^{\can{\K}{0}}(u, v) = k$, $k \in \nati$, that is, such that $\nelem{\pair{u, 
v}}{m} \in P^{\enumpcan{\K}{0}}$ for all $m \in \nati$ with $m \leq k$.
By the definition of $\pcan{\K}{0}$, we have that $\A(P(a, b)) = k$. Since $\I$ 
is a model of $\A$, we have that $P^\I(a^\I, b^\I) \geq k$. In other words, 
$\nelem{\pair{h(u), h(v)}}{m} \in P^{\enum{\I}}$ for all $m \leq k$, and, 
similarly to the concept case, we set $h_P(\nelem{\pair{u, v}}{m}) = 
\nelem{\pair{h(u), h(v)}}{m}$ for all $m$.

For the second step, consider an individual $a \in \Ind$ with its 
interpretation $u = a^{\pcan{\K}{0}}$ and a role $P \in \Roles$ such that 
$\ccl{u}{\pcan{\K}{0}}{\T}(\exists P) = k$ for $k \in \nati$, but 
$(\exists P)^{\pcan{\K}{0}}(u) = l < k$ (the case where $P$ is not an atomic 
role is analogous). 
Then, 
$\delta = \ccl{u}{\pcan{\K}{0}}{\T}(\exists P) - (\exists P)^{\pcan{\K}{0}}(u) 
> 0$, hence $\Delta^{\pcan{\K}{1}} = \Delta^{\pcan{\K}{0}} \cup 
\{w^1_{u,P},\dots,w^\delta_{u,P}\}$ where $w^j_{u,P}$ are fresh anonymous 
elements.
Moreover, $P^{\pcan{\K}{1}}$ contains $P^{\pcan{\K}{0}}$ plus tuples 
$\pair{u, w^1_{u, P}}, \dots, \pair{u, w^\delta_{u, P}}$.
We next show that $h$ can be extended to all anonymous elements 
$w^j_{u,P}$ introduced at this step as a result of some $u$ and role $P$ with 
the above properties such that $(h, h_S,\dots)$, $S \in \Concepts \cup 
\Roles$, is predicate injective on $\Ind$.
Because 
$\ccl{u}{\pcan{\K}{0}}{\T}(\exists P) = k$ and 
$(\exists P)^{\pcan{\K}{0}}(u) = l < k$, there exists a sequence of 
concepts $C_0, \ldots, C_n$ with $C_n = \exists P$
such that 
$C_{i-1} \isa C_i \in \T$ for all $i \in [1, n]$ and 
$C_{0}^{\pcan{\K}{0}}(u) = k$. 
Since $(h, h_S,\dots)$ is predicate injective on $\Ind$ at the first step and 
$h(u) = a^\I$, we have 
$C_{0}^{\I}(a^\I) \ge k$. Because $\I$ is a model of $\K$, it satisfies all 
axioms in $\T$, hence, $C_{i}^{\I}(a^\I) \ge k$, and as a result 
$(\exists P)^{\I}(a^\I) \ge k$. In other words, 
$P^{\enum{\I}}$ contains at least $k$ pairs
$\nelem{\pair{a^\I, z_i}}{m_i}$, $i \in [1, k]$.
Observe that from the first step and
every pair 
$\nelem{\pair{u, v_1}}{m}, \nelem{\pair{u, v_2}}{m'} \in P^{\enumpcan{\K}{0}}$ 
with 
$v_1 \not= v_2$ or $m \not= m'$, we have
$h_P(\nelem{\pair{u, v_1}}{m}) \not= h_P(\nelem{\pair{u, v_2}}{m'})$.
Because $P^{\enumpcan{\K}{0}}$ contains $l$ such distinct tuples and 
$k = \delta + l$,
there are at least $\delta$ pairs
$\nelem{\pair{a^\I, r_{1}}}{n_1},\dots,\nelem{\pair{a^\I, 
r_{\delta}}}{n_\delta}$ in $P^{\enum{\I}}$ 
for which there is no 
$\nelem{\pair{u, v}}{m} \in P^{\enumpcan{\K}{0}}$ that maps to them under $h_P$.
Therefore, we can extend $h$ such that $h(w^j_{u,P}) = r_{j}$ and 
set $h_P$ so that 
$h_P(\nelem{\pair{u, w^j_{u,P}}}{1}) = \nelem{\pair{a^\I, r_{j}}}{n_j}$. 
Suppose now that there exists $w^j_{u,P}$ such that $h(w^j_{u,P}) = b^\I$ with 
$b \in \Ind$. Since $P^{\pcan{\K}{1}}(u, w^j_{u,P}) = 1$, we have
$\nelem{\pair{u, w^j_{u, P}}}{1} \in P^{\enumpcan{\K}{1}}$ and
$\nelem{\pair{w^j_{u, P}, u}}{1} \in (P^-)^{\enumpcan{\K}{1}}$, hence,
the requirement for $h_{P^-}$ w.r.t. $w^j_{u, P}$ is trivially satisfied.
Finally, consider an element $u = a^{\pcan{\K}{0}}$ such that
$\ccl{u}{\pcan{\K}{0}}{\T}(A) > A^{\pcan{\K}{0}}(u)$ with $A \in \Concepts$. 
In such a case, 
$A^{\pcan{\K}{1}}(u)$ is set to $\ccl{u}{\pcan{\K}{0}}{\T}(A)$.
Given the above discussion, it is trivial to verify that
$h_A$ satisfies the required condition on the pairs
$\nelem{u}{m} \in A^{\enumpcan{\K}{1}}$.
As a result of all the above, we have shown that $(h, h_S, \dots)$ is predicate 
injective on $\Ind$ at the second step as well.

Last, observe that for all $i > 1$, and for all $S \in \Concepts \cup \Roles$, 
extensions $S^{\pcan{\K}{i}}$ contain $S^{\pcan{\K}{i-1}}$ plus tuples $\tpl t$
mentioning only anonymous elements, for which we know by definition that 
$S^{\pcan{\K}{i}}(\tpl t) = 1$. Therefore, $h$ can be trivially extended to 
these anonymous elements so that $(h, h_S, \dots)$ is predicate injective on 
$\Ind$ at step $i$.
\end{proof}

Since a Boolean CQ $q$ can be seen as a bag of atoms, we can consider its corresponding \emph{Boolean enumerated CQ} (\emph{e-CQ}), which is the e-bag $\enum{q}$. We call the elements of $\enum{q}$ \emph{enumerated atoms} (\emph{e-atoms}).
For the following definition, it is convenient to partition a Boolean CQ $q$ to 
the subqueries $q_S$ each of which consists of all atoms in $q$ over 
atomic concept or role $S$ (with corresponding multiplicities)
and subquery $q_=$ consisting of all equalities in $q$. 
An \emph{enumerated valuation} (\emph{e-valuation}) of a Boolean e-CQ 
$\enum{q}$, for $q() = \exists \tpl y. \, \phi(\tpl y)$, over an e-bag 
interpretation $\enum{\I} = \langle \Delta^{\I}, \cdot^{\enum{\I}} \rangle$ is 
a family $(\nu, \nu_S, \ldots)$, $S \in \mathbf C \cup \mathbf R$, of functions
$$
\begin{array}{rll}
\nu: &  \tpl y \cup \Ind \to \Delta^{\I}, \\
\nu_S: & \enum{q}_S \to S^{\enum{\I}}, & \text{for all } S \in \mathbf C \cup \mathbf R,
\end{array}
$$
such that 
\begin{itemize}
\item[--] $\nu(a) = a^{\enum{\I}}$ for each $a \in \Ind$,
\item[--] $\nu(y) = \nu(t)$ for all equality e-atoms $\nelem{y = t}{m} \in 
\enum{q}_=$,
\item[--] $\nu_A(\nelem{A(t)}{m}) = \nelem{\nu(t)}{\ell}$ for all $A \in 
\Concepts$ and $\nelem{A(t)}{m} \in \enum{q}_A$, where $\ell\in \nati$ is some 
number, and
\item[--] $\nu_P(\nelem{P(t_1, t_2)}{m}) = \nelem{(\nu(t_1), \nu(t_2))}{\ell}$ 
for all $P \in \Roles$ and $\nelem{P(t_1, t_2)}{m} \in \enum{q}_P$, where 
$\ell\in \nati$ is some number.
\end{itemize}
Similarly to the case of e-homomorphisms, we sometimes write 
$\nu_{P^-}(\nelem{P^-(t_1, t_2)}{m})$ instead of $\nu_{P}(\nelem{P(t_1, 
t_2)}{m})$, for $P \in \mathbf R$.

Intuitively, a Boolean CQ can be seen as a bag interpretation with terms 
(variables and individuals) in the domain. Then, an e-valuation is just an 
e-bag homomorphism from the enumerated version of this special bag 
interpretation to a normal e-bag interpretation.
It is straightforward to check that the number of e-valuations of a Boolean 
e-CQ $\enum{q}$ over an e-bag interpretation $\enum{\I}$ is precisely the 
multiplicity $\cert{q}{\I}(\emptytpl)$ of the empty tuple in the evaluation of 
$q$ over $\I$.

The following lemma says that if two e-valuations over the bag canonical model 
coincide on all the (enumerated copies of the) atoms of a rooted CQ that 
involve terms evaluating to (the interpretations of) individuals, then they are 
the same e-valuation.

\begin{lemma}
\label{lemma:difference-locality}
Let $q$ be a rooted Boolean CQ and $\K$ be a $\dlliteCbag$ ontology. If two 
e-valuations $(\nu^1, \nu^1_{S'}, \ldots)$ and $(\nu^2, \nu^2_{S'}, \ldots)$ of 
$\enum{q}$ over $\enumcan{\K}$ are different, then there exist an individual $a 
\in \Ind$, e-atom $\nelem{S(\tpl t)}{m} \in \enum{q}$ and number $i \in \{1, 
2\}$ such that $\nu^i(a) \in \nu^i(\tpl t)$ and $\nu^1_S(\nelem{S(\tpl t)}{m}) 
\neq \nu^2_S(\nelem{S(\tpl t)}{m})$.
\end{lemma}

\begin{proof}
Let e-valuations $(\nu^1, \nu^1_{S'}, \ldots)$ and $(\nu^2, \nu^2_{S'}, 
\ldots)$ of $\enum{q}$ over $\enumcan{\K}$ be different, but, for the sake of 
contradiction, $\nu^1_S(\nelem{S(\tpl t)}{m}) = \nu^2_S(\nelem{S(\tpl t)}{m})$ 
for all $a \in \Ind$, $\nelem{S(\tpl t)}{m} \in \enum{q}$ and $i \in \{1, 2\}$ 
such that $\nu^i(a) \in \nu^i(\tpl t)$. Since the e-valuations are different, 
there exists
$\nelem{S(\tpl t)}{m} \in \enum{q}$ such that $\nu^1_S(\nelem{S(\tpl t)}{m}) \neq \nu^2_S(\nelem{S(\tpl t)}{m})$. Moreover, by assumption $\tpl t$ consists of only variables. Suppose that $S(\tpl t)$ is $P(x_1, x_2)$, where $P \in \Roles$ (we do it without loss of generality, because the case of $A(x)$ for $A \in \Concepts$ can be handled in the same way). 

Boolean CQ $q$ is rooted, so there exists a sequence 
$$
\nelem{R_1(t_0, t_1)}{m_1}, \nelem{R_2(t_1, t_2)}{m_2}, \ldots, \nelem{R_k(t_{k-1}, t_k)}{m_k}
$$
 of e-atoms such that $t_0 \in \Ind$, $\nelem{R_k(t_{k-1}, t_k)}{m_k}$ is either $\nelem{P(x_1, x_2)}{m}$ or $\nelem{P^-(x_2, x_1)}{m}$, and for each $j = 1, \ldots, k$ either $\nelem{R_j(t_{j-1}, t_j)}{m_j}$ is in $\enum{q}$, if $R_j$ is an atomic role, or $\nelem{P_j(t_{j}, t_{j-1})}{m_j}$ is in $\enum{q}$, if $R_j = P_j^-$. 
 
 We claim that 
 \begin{equation}
 \label{enum_anch_eq}
 \nu^1_{R_j}(\nelem{R_j(t_{j-1}, t_j)}{m_j}) = \nu^2_{R_j}(\nelem{R_j(t_{j-1}, t_j)}{m_j})
 \end{equation}
 for all $j = 1, \ldots, k$ (which, in particular, contradicts our assumption on $\nelem{P(x_1, x_2)}{m}$). To prove this claim, suppose for the sake of contradiction that it is not the case, and let $j \in \{1, \ldots, k\}$ be the smallest number such that \eqref{enum_anch_eq} does not hold.  
By assumption, we know that $\nu^i(t_{j-1}) \neq \nu^i(a)$ for both $i = 1, 2$ 
and any $a \in \Ind$ (therefore, $j \neq 1$, because $t_0 \in \Ind$). However, 
since $j$ is the smallest number, $\nu^1(t_{j-1}) = \nu^2(t_{j-1})$. So, the 
element $u = \nu^1(t_{j-1})$ in the bag canonical model $\can{\K} = \bigcup_{i 
\geq 0} \pcan{\K}{i}$ was introduced not in $\pcan{\K}{0}$, which implies, by 
construction, that $(\exists R_j)^{\can{\K}}(u) \leq 1$. In fact, since 
$(\nu^1, \nu^1_{S'}, \ldots)$ is an e-valuation, $(\exists R_j)^{\can{\K}}(u) = 
1$, that is, there exists just one $v \in \Delta^{\can{\K}}$ such that 
$R_j^{\can{\K}}(u, v) \geq 1$, and, moreover, $R_j^{\can{\K}}(u, v) = 1$. In 
other words, it holds that $\nelem{(u, v)}{1} \in R_j^{\enumcan{\K}}$, but 
$\nelem{(u, v)}{2} \notin R_j^{\enumcan{\K}}$.
Since $(\nu^1, \nu^1_{S'}, \ldots)$ and $(\nu^2, \nu^2_{S'}, \ldots)$ are 
e-valuations, $\nu^1_{R_j}$ and $\nu^2_{R_j}$ send $\nelem{R_j(t_{j-1}, 
t_j)}{m_j}$ to some enumerated pairs in $R_i^{\enumcan{\K}}$, which, by 
assumption, are different. However, we also know that $\nu^1(t_{j-1}) = 
\nu^2(t_{j-1})$, so the only possibility for both 
$\nu^1_{R_j}(\nelem{R_j(t_{j-1}, t_j)}{m_j})$ and 
$\nu^2_{R_j}(\nelem{R_j(t_{j-1}, t_j)}{m_j})$ is $\nelem{(u, v)}{1}$. 
Therefore, our assumption on existence of $j$ was wrong and 
\eqref{enum_anch_eq} indeed holds for all $j$. In particular, it holds for $j = 
k$, which contradicts the fact that  $\nu^1_P(\nelem{P(x_1, x_2)}{m}) \neq 
\nu^2_P(\nelem{P(x_1, x_2)}{m})$. Therefore, our assumption on $(\nu^1, 
\nu^1_{S'}, \ldots)$ and $(\nu^2, \nu^2_{S'}, \ldots)$ was wrong, and the lemma 
is proven.
\end{proof}

Having Lemmas~\ref{lemma:e-homomorphism} and~\ref{lemma:difference-locality} at 
hand, we are ready to prove that for $\dlliteCbag$ ontologies rooted queries 
can be evaluated over the bag canonical model.

\RootedUniversality*
\begin{proof}
First, note that it is enough to consider only Boolean rooted CQs, because the 
required property for a non-Boolean rooted CQ  $q(\tpl x)$ follows from the 
property for all Boolean CQs obtained from $q(\tpl x)$ by replacing variables 
$\tpl x$ by individuals from $\Ind$.

For a Boolean rooted CQ $q$ it is enough to show that for any $\dlliteCbag$ 
ontology $\K$, any bag model $\I$ of $\K$ and any e-valuation $(\nu, \nu_S, 
\ldots)$ of $q$ over $\enumcan \K$ there exists a unique e-valuation $(\nu', 
\nu'_S, \ldots)$ of $q$ over $\enum \I$. By Lemma~\ref{lemma:e-homomorphism} we 
know that
there exists an e-homomorphism $(h, h_S, \ldots)$ from $\enumcan{\K}$ to $\enum{\I}$ that is predicate-injective on $\Ind$.
Therefore, we can take the composition $(\nu, \nu_S, \ldots) \circ (h, h_S, 
\ldots) = (\nu \circ h, \nu_S \circ h_S, \ldots)$ as $(\nu', \nu'_S, \ldots)$; 
indeed, the result of this composition is an e-valuation of $q$ over $\enum \I$ 
and, by Lemma~\ref{lemma:difference-locality}, this result is
unique throughout e-valuations of $q$ over $\enumcan \K$.
\end{proof}

\QABoundedCan*
\begin{proof}
Let $q$ be the CQ $q(\tpl x) = \exists \tpl y.\, \phi(\tpl x, \tpl y)$.
First note that because CQs are safe and equalities between individuals are not
allowed, $\phi(\tpl x, \tpl y)$ contains at least one atom, thus, $n \ge 1$.
Observe that $\pcan{\K}{n}$ is a subinterpretation of $\can\K$, hence, from the 
monotonicity property of CQs, we have $q^{\pcan{\K}{n}} \subseteq q^{\can\K}$. 
To prove the inverse inclusion, we show that interpretations 
$\pcan{\K}{k}$ with $k>n$ do not 
contribute to the bag answers $q^{\can\K}$, and as a result, they can be 
disregarded.
In other words, we prove that for every tuple of individuals 
$\tpl a$ and every valuation 
$\lambda: \tpl x \cup \tpl y \cup \Ind \to \Delta^{\can\K}$
with
$\lambda(\tpl x) = \tpl a$ 
such that there exist a number $k > n$ 
and an atom $S_k({\tpl t}_k)$ in $\phi(\tpl x, \tpl y)$ with
$S_k^{\pcan{\K}{k}}(\lambda({\tpl t}_k)) >  
S_k^{\pcan{\K}{n}}(\lambda({\tpl t}_k))$, it holds that
$\prod_{S(\tpl t) \text{ in } \phi(\tpl x, \tpl y)} 
S^{\pcan{\K}{k}}(\lambda(\tpl t)) = 0$.
By definition of canonical models, for $k > n \ge 1$, 
interpretation $\pcan{\K}{k}$ differs from $\pcan{\K}{k-1}$ in that it contains
a number of tuples not present in $\pcan{\K}{k-1}$ having multiplicity 1 and 
mentioning only anonymous elements.
Hence, inequality 
$S_k^{\pcan{\K}{k}}(\lambda({\tpl t}_k)) >  
S_k^{\pcan{\K}{n}}(\lambda({\tpl t}_k))$
effectively means that we are considering only valuations that send an atom of 
$\phi(\tpl x, \tpl y)$ to a tuple of anonymous elements of $\can\K$ added after 
step $n$. 
%It is immediate from the 
%definition of the canonical model and the fact that $k \ge 2$ that such a 
%tuple, that is, $\lambda({\tpl t}_k)$, contains only anonymous elements. 
Suppose by contradiction that there are $\tpl a$ and $\lambda$ satisfying the 
above criteria but 
$\prod_{S(\tpl t) \text{ in } \phi(\tpl x, \tpl y)} 
S^{\pcan{\K}{k}}(\lambda(\tpl t)) \ge 1$. 
This means that $\lambda$ satisfies all equalities of $q$ and for every atom 
$S(\tpl t)$ of $q$, $S^{\pcan{\K}{k}}(\lambda(\tpl t)) \ge 1$.
Because $q$ is rooted, every connected component of the Gaifman graph of 
$q$ has a node, that is, an equivalence class, that mentions a free variable or 
an individual. 
Consider the component of $q$ that contains atom $S_k({\tpl t}_k)$ and the 
equivalence class $\eqc t$ of this component that contains a free variable or 
an individual.
Because CQs are safe by definition, this component
contains an atom $P({\tpl t}')$ mentioning a term in $\eqc t$. 
As a result, $\lambda({\tpl t}')$ contains at least one individual, which, 
given that $\lambda({\tpl t}_k)$ is a tuple of anonymous elements,
implies that $P({\tpl t}')$ and $S_k({\tpl t}_k)$ are different atoms.
By definition of canonical models, we know that $\pcan{\K}{1}$ is the 
subinterpretation of $\can\K$ containing tuples with at least one individual, 
hence we derive that 
$P^{\pcan{\K}{1}}(\lambda({\tpl t}')) \ge 1$. 
But then, since the image of $P({\tpl t}')$ under $\lambda$ falls into 
$\pcan{\K}{1}$ while the image of $S_k({\tpl t}_k)$ under $\lambda$ falls into 
$\pcan{\K}{k}$ but not into $\pcan{\K}{n}$ (which implies the same for all 
subinterpretations of $\pcan{\K}{n}$),  
and both atoms belong to the same connected component, it means that
$\phi(\tpl x, \tpl y)$ contains conjunction 
$\bigwedge_{j = 1}^k S_j({\tpl t}_j)$ 
%$\bigwedge_{i = 1}^l P_i({\tpl t}'_i) \wedge 
% \bigwedge_{j = 2}^k S_j({\tpl t}_j)$ 
such that
%\begin{itemize}
%\item[--]
\emph{(i)} atom $S_1({\tpl t}_1)$ is connected with $P({\tpl t}')$,
\emph{(ii)}
%\item[--]
%${\tpl t}'_i \cap {\tpl t}'_{i+1} \not= \emptyset$, for $1 \leq i < l$,
${\tpl t}_j \cap {\tpl t}_{j+1} \not= \emptyset$, for $1 \leq j < k$, and
%${\tpl t}'_l \cap {\tpl t}_{2} \not= \emptyset$;
\emph{(iii)}
%\item[--]
$S_j^{\pcan{\K}{j}}(\lambda({\tpl t}_j)) \ge 1$ and 
$S_j^{\pcan{\K}{j-1}}(\lambda({\tpl t}_j)) = 0$, for $j \in [1, k]$.
%\end{itemize}
In other words, the image of each one of the atoms under $\lambda$ 
falls respectively onto tuples created in 
$\pcan{\K}{1}, \pcan{\K}{2},\dots,\pcan{\K}{k}$. 
But then, this means that $q$ contains at least $k$ atoms, which is a 
contradiction given that $k > n$.
\end{proof}

\CQsRewriteUCQs*
\begin{proof}
First we prove the claim for max unions of CQs.
Consider the $\dlliteCbag$ TBox 
$\T = \{A \isa \exists R,\ \exists R^- \isa B \}$,
the rooted CQ
$q(x) = \exists y.\, R(x, y) \wedge B(y)$, and
the $\dlliteCbag$ ABox 
$\A = \bagl A(a), A(a), A(a), R(a, b), R(a, b), B(b), B(b), B(b) \bagr$
and let
$\K = \langle \T, \A \rangle$.
Then, $\can{\K}$ is such that
\[
\begin{array}{l @{~=~} l l @{~=~} l}
\Delta^{\can{\K}} & \Ind \cup \{w_{a, R}\}, & A^{\can{\K}} & \bagl a, a, 
a\bagr,\\
R^{\can{\K}} & \bagl (a, b), (a, b), (a, w_{a, R}) \bagr, & 
B^{\can{\K}} & \bagl b, b, b, w_{a, R} \bagr.
\end{array}
\]
Evaluating $q$ over $\can{\K}$, we get $q^{\can{\K}}(a) = 7$ for the 
individual $a$. 
Suppose now that there exists a rewriting of $q$ to a max union of CQs and let 
$q'(x) = q_1(x) \vee \cdots \vee q_n(x)$ be such a rewriting where 
$q_1,\dots,q_n$ are CQs. This means that $\bigcup_{i=1}^n q_i^{\can{\langle 
\emptyset, \A \rangle}} = q^{\can{\K}}$ or, alternatively, 
that there exists $i \in [1, n]$ with 
$q_i^{\can{\langle \emptyset, \A \rangle}} = q^{\can{\K}}$.
Observe that $\A$ contains three distinct assertions with multiplicities $3$, 
$2$, and $3$. Therefore, whenever there is a valuation for the terms of 
$q_i$ that maps an atom of $q_i$ to one of these assertions, the multiplicity 
is either $2$ or $3$. As a result and because $q_i$ is a CQ, any valuation of 
$q_i$ contributes to $q_i^{\can{\langle \emptyset, \A \rangle}}(a)$ a 
multiplicity that is a multiple of $2$ or $3$. 
Since $7$ is prime, there can be no valuation contributing a multiplicity of 
$7$. However, $7$ can be expressed as the sums $2 + 2 + 3$ or $2 \times 2 + 3$.
For the former sum, this means that there exist three distinct valuations 
contributing to $q_i^{\can{\langle \emptyset, \A \rangle}}(a)$ multiplicities 
$2$, $2$, and $3$, respectively, which is 
clearly impossible given the fact that to get $2$, query $q_i$ must be set 
equal to $\exists y.\, R(x, y)$, which excludes the possibility of getting a 
multiplicity of $3$.
For the latter sum, this means that there exist two distinct valuations 
contributing to $q_i^{\can{\langle \emptyset, \A \rangle}}(a)$ multiplicities 
$4$ and $3$, respectively, which is again impossible given the fact that to get 
$4$, query $q_i$ must be set equal to 
$\exists y.\, \exists z.\, R(x, y) \wedge R(x, z)$
(another possibility would have been to use the same variable for $y$ and $z$, 
but the argumentation stays the same), which excludes the possibility of 
getting a multiplicity of $3$.

We now prove the claim for arithmetic unions of CQs by building on the 
observations made in the proof above.
Consider the $\dlliteCbag$ TBox 
$\T' = \T \cup \{\ C \isa \exists P,\ \exists P^- \isa D \}$,
the rooted CQ
$q'(x, z) = \exists y.\, \exists u.\, 
R(x, y) \wedge B(y) \wedge P(z, u) \wedge D(u)$, and
the $\dlliteCbag$ ABox 
$\A' = \A \cup \bagl C(a), P(a, b), D(b) \bagr_{8,8,8}$
and let
$\K' = \langle \T', \A' \rangle$.
Then, $\can{\K'}$ is such that
\[
\begin{array}{l @{~=~} l l @{~=~} l}
\Delta^{\can{\K'}} & \Ind \cup \{w_{a, R}\}, & A^{\can{\K'}} & \bagl a, a, 
a\bagr,\\
R^{\can{\K'}} & \bagl (a, b), (a, b), (a, w_{a, R}) \bagr, & 
B^{\can{\K'}} & \bagl b, b, b, w_{a, R} \bagr, \\
P^{\can{\K'}} & \bagl (a, b) \bagr_8, & C^{\can{\K'}} & \bagl a \bagr_8,\\
D^{\can{\K'}} & \bagl b \bagr_8.
\end{array}
\]
First, observe that $\T'$ contains $\T$ plus a copy of the axioms in $\T$ with 
their predicates renamed. Hence, $\T'$ can be seen as having two disconnected 
parts.
Second, $q'$ has two rooted connected components the first of which, say 
$q_1(x)$, is query $q$ from the previous part of the proof, while the second, 
say $q_2(z)$, is an isomorphic query of $q$ with the predicates renamed 
according to the one-to-one mapping $f = \{(A, C), (R, P), (B, D) \}$.
Based on these observations, we draw the following conclusions:
\emph{(i)} the multiplicity of a tuple $(c_1, c_2)$ of individuals in 
$(q')^{\can{\K'}}$ is the result of multiplying numbers $q_1^{\can{\K'}}(c_1)$ 
and 
$q_2^{\can{\K'}}(c_2)$;
\emph{(ii)} a rewriting of $q'$ into an arithmetic union of CQs exists if and 
only if a rewriting for $q_1$ and $q_2$ exists;
\emph{(iii)} the rewritings of $q_1$ and $q_2$ should have the same number of 
CQs, which should be identical up to renaming of variables and predicates 
based on $f$.
Consider now the evaluation of $q'$ over $\can{\K'}$. This leads to a bag 
containing just tuple $(a, a)$ with multiplicity
$(q')^{\can{\K'}}((a, a)) = q_1^{\can{\K'}}(a) \times q_2^{\can{\K'}}(a) = 7 
\times 64$. 
Let also $q'_1(x)$ and $q'_2(z)$ be the rewritings for $q_1(x)$ and $q_2(z)$, 
respectively. Given the discussion in the first part of the proof, to get 
multiplicity $7$ for $q_1^{\can{\K'}}(a)$, we have only two ways: either as the 
sum of $2 + 2 + 3$ or as the sum $4 + 3$. 
For the former, $q'_1(x)$ should be equal to 
$\exists y.\, R(x, y) \veedot \exists y.\, R(x, y) \veedot A(x)$, while for the 
latter, $q'_1(x)$ should be equal to 
$\exists y_1.\, \exists y_2.\, R(x, y_1) \wedge R(x, y_2) \veedot A(x)$.
By construction, both of these queries when evaluated over 
$\can{\langle \emptyset, \A'\rangle}$ return the correct multiplicity for 
$q_1^{\can{\K'}}(a)$ and there are no other queries with this property. 
However, evaluating their identical versions up to renaming of variables and 
predicates based on $f$ over
$\can{\langle \emptyset, \A'\rangle}$, that is, queries 
$\exists u.\, P(z, u) \veedot \exists u.\, P(z, u) \veedot C(z)$
and
$\exists u_1.\, \exists u_2.\, P(z, u_1) \wedge P(z, u_2) \veedot C(z)$,
we get, respectively, multiplicity $24$ and $72$, both of which are different 
from 
$q_2^{\can{\K'}}(a) = 64$.
\end{proof}

%\AtomicRewriteUCQs*
%\begin{proof}
%The claim is a direct consequence of Theorem~\ref{thm:chase-univ}, 
%Lemma~\ref{lemma:rw-atomic-ucq}, and the following facts: 
%\emph{(i)} all queries of Lemma~\ref{lemma:rw-atomic-ucq} which are evaluated 
%over $\can\K$ comprise the class of non-Boolean rooted queries, which is a 
%subclass of the class of the queries of this theorem;
%\emph{(ii)} 
%all queries of Lemma~\ref{lemma:rw-atomic-ucq} which are evaluated over the 
%ABox belong to the class of UCQs, while
%their construction depends only on $q$ and the TBox;
%\emph{(iii)} all other queries of the class of queries of this theorem are the 
%Boolean queries containing at least one individual, which are in one-to-one 
%correspondence with the ones from~\emph{(ii)}, having also corresponding 
%UCQ rewritings.
%\end{proof}

\newcommand{\toAlg}[1]{E_{#1}}
\newcommand{\valref}[2]{\mathsf{ref}(#1,#2)}

Let $\cal S$ be a finite set of concept and role symbols.
We say that a bag interpretation 
$\langle \Delta^\I, \cdot^\I\rangle$ 
is finite relative to $\cal S$ if, 
for every $S \in \cal S$, bag $S^\I$ is finite.
Finite bag interpretations relative to a finite set $\cal S$ correspond to 
\emph{bag database instances} $I$ over \emph{bag schemas} $\cal S$ with 
\emph{domains} $\Delta^\I$ as these where defined in~\cite{GrumbachJCSS96}. 
Hence, in the following, given a finite set $\cal S$ and a bag interpretation  
$\I$ that is finite relative to $\cal S$, we denote by $I_\I$ the corresponding 
bag database instance.
Also, given a bag database instance $I_\I$ and a $\balgonee$ algebra expression 
$E$, we denote by $E(I_\I)$ the bag corresponding to the evaluation of $E$ over 
$I_\I$.
\begin{restatable}{proposition}{CalcToAlg} \label{prop:calc-to-alg}
  % Let $\I$ be an arbitrary bag interpretation. Then:
  % \begin{enumerate}
  % \item For each %domain-independent
  %   $\balgonee$-query $q$ there is a $\balgonee$ algebra expression $E$ such
  %   that $\ans q\I=\ans E\I$.
  % \item For each $\balgonee$ algebra expression $E$ there is a
  %   $\balgonee$-query $q$ such that $\ans q\I=\ans E\I$.
  % \end{enumerate}
  Let $\I$ by any bag interpretation that is finite relative to a finite set 
  $\cal S$.
  For each $\balgonee$-query $q$ there is a $\balgonee$ algebra expression $E$
  such that
  %for each finite bag interpretation $\I$ relative to $\cal S$, 
  $\ans q\I= E(I_\I)$.
\end{restatable}

\begin{proof}
  We refer to~\cite{GrumbachJCSS96} for the definition of the $\balgonee$
  operators. 
  %Given a $\balgonee$ algebra expression $E$ and a bag
  %interpretation $\I$, we write $\ans E\I$ for the evaluation of $E$ over $\I$,
  %defined in the intuitive way.
  % As shown in~\cite{GrumbachJCSS96}, the operators $\tau$ (tupling), $\beta$
  % (bagging), $\times$ (Cartesian product), $\uplus$ (additive union), $-$
  % (subtraction), $\sigma$ (selection), and $\pi$ (projection) constitute a
  % complete set of operators for $\balgonee$.
  %
  % For the first claim, for
  For each $\balgonee$-query $q$, we define a $\balgonee$ algebra expression
  $\toAlg q$ by induction on the structure of $q$ as follows, where, for each
  tuple of terms $\tpl t$ over $\Vars\cup\Ind$ and $t\in\tpl t\cup\Ind$,
  $\valref t{\tpl t}$ is defined as $t$ if $t\in\Ind$, and otherwise as the
  first position in $\tpl t$ containing $t$:
  \begin{itemize}
  % \item If $q(\tpl x)=(t=a)$ for $t\in\Vars\cup\Ind$, $\tpl x=\Vars\cap\set{t}$, and
  %   $a\in\Ind$, then
  %   \begin{itemize}
  %   \item $\toAlg q=\beta(\tau(a))$ if $t\in\Vars$,
  %   \item $\toAlg q=\beta(\tau())$ if $t\in\Ind$ is identical to $a$, and
  %   \item $\toAlg q=\beta(\tau())\setminus\beta(\tau())$ if $t\in\Ind$ is
  %     distinct from $a$.
  %   \end{itemize}
  \item If $q(\tpl x)=S(\tpl t)$ for $S\in\Concepts\cup\Roles$ and
    $\tpl t=\langle t_1,\dots,t_{|\tpl t|}\rangle$ a tuple over
    $\tpl x\cup\Ind$, then
    $$\toAlg q=\pi_{\valref{x_1}{\tpl t},\dots,\valref{x_{|\tpl x|}}{\tpl
        t}}(\sigma_{j\in[1,|\tpl t|]\setminus\set{\valref{x_1}{\tpl
          t},\dots,\valref{x_{|\tpl x|}}{\tpl t}}\colon j=\valref{t_j}{\tpl
        t}}(S)).$$
  % \item If $q(\tpl x)=q_0(\tpl x_0)\land(x'=x)$ for $x'\in\Vars$,
  %   $x\in\tpl x_0$, and $\tpl x=\tpl x_0\cup\set{x'}$, then
  %   \begin{itemize}
  %   \item
  %     $\toAlg q=\sigma_{\valref{t_1}{\tpl x_0}=\valref{t_2}{\tpl
  %         x_0}}(\toAlg{q_0})$ if $x'\in\tpl x_0=\tpl x$ (we assume w.l.o.g.\
  %     that the order of variables in $\tpl x$ and $\tpl x_0$ is the same), and
  %   \item $\toAlg q=\pi_{1,\dots,|\tpl x_0|,\valref x{\tpl x_0}}(\toAlg{q_0})$
  %     if $\tpl x=\tpl x_0 x'$ (we assume w.l.o.g.\ that
  %     $x'$ is added as the last variable to $\tpl x_0$).
  %   \end{itemize}
  \item If $q(\tpl x)=q_0(\tpl x_0)\land(x=t)$ for $x\in\tpl x_0$,
    $t\in\Vars\cup\Ind$, and $\tpl x=\tpl x_0\cup(\set{t}\setminus\Ind)$, then
    \begin{itemize}
    \item
      $\toAlg q=\sigma_{\valref{x}{\tpl x_0}=\valref{t}{\tpl
          x_0}}(\toAlg{q_0})$ if $t\in\tpl x_0\cup\Ind$ and $\tpl x=\tpl x_0$
      (we assume w.l.o.g.\ that the order of variables in $\tpl x$ and
      $\tpl x_0$ is the same), and
    \item $\toAlg q=\pi_{1,\dots,|\tpl x_0|,\valref x{\tpl x_0}}(\toAlg{q_0})$
      if $t\in\Vars\setminus\tpl x_0$ and $\tpl x=\tpl x_0 t$ (we assume
      w.l.o.g.\ that $t$ is added as the last variable to $\tpl x_0$).
    \end{itemize}
  \item If $q(\tpl x)=q_1(\tpl x_1)\land q_2(\tpl x_2)$, for
    $\tpl x=\tpl x_1\cup\tpl x_2$, then
    $$\toAlg q=\pi_{\valref{x_1}{\tpl x_1\tpl x_2},\dots,\valref{x_{|\tpl
          x|}}{\tpl x_1\tpl x_2}}(\sigma_{x\in\tpl x_1\cap\tpl x_2\colon\valref
      x{\tpl x_1}=|\tpl x_1|+\valref x{\tpl
        x_2}}(\toAlg{q_1}\times\toAlg{q_2})).$$
  \item If $q(\tpl x)=\exists\tpl y.q_0(\tpl x,\tpl y)$, then
    $\toAlg q=\pi_{1,\dots,|\tpl x|}(\toAlg{q_0})$ (we assume w.l.o.g.\ that in $q_0$
    variables in $\tpl y$ come after variables in $\tpl x$).
  \item If $q(\tpl x)=q_1(\tpl x)\lor q_2(\tpl x)$, then
    $\toAlg q=\toAlg{q_1}\cup\toAlg{q_2}$.
  \item If $q(\tpl x)=q_1(\tpl x)\veedot q_2(\tpl x)$, then
    $\toAlg q=\toAlg{q_1}\uplus\toAlg{q_2}$.
  \item If $q(\tpl x)=q_1(\tpl x)\setminus q_2(\tpl x)$, then
    $\toAlg q=\toAlg{q_1}-\toAlg{q_2}$.
  \end{itemize}
  It is straightforward to check that, for each $\balgonee$-query $q$ and each
  bag interpretation $\I$, we have $\ans q \I= \toAlg{q}(I_\I)$.
  % For the second claim, node that since $\balgonee$ may not contain nested
  % bags, the bagging operator $\beta$ can only be applied to a tuple over
  % $\Ind$, which can only be constructed by an application of $\tau$. Thus, for
  % each $\balgonee$ algebra expression $E$ and each $n\ge 0$ such that $E$
  % yields a bag of $n$-tuples over $\Ind$, we define a $\balgonee$-query
  % $\toCalc E(\tpl x)$ where $\tpl x=\langle x_1,\dots, x_n\rangle$ by induction
  % on the structure of $E$ as follows:
  % \begin{itemize}
  % \item If $E=\beta(\tau(a_1,\dots,a_n))$ for
  %   $\set{a_1,\dots,a_n}\subseteq\Ind$, then
  %   $\toCalc{E}(\tpl x)=\bigwedge_{i=1}^n(x_i=a_i)$.
  % \item If $E=E_1\times E_2$, where $E_1$ yields bags of arity $m\le n$ and
  %   $E_2$ bags of arity $n-m$, then
  %   $\toCalc{E}(\tpl x)=\toCalc{E_1}(\tpl x_1)\land\toCalc{E_2}(\tpl x_2)$, for
  %   $\tpl x_1=\langle x_1,\dots,x_m\rangle$ and
  %   $\tpl x_2=\langle x_{m+1},\dots,x_n\rangle$.
  % \item If $E=E_1\uplus E_2$, then
  %   $\toCalc{E}(\tpl x)=\toCalc{E_1}(\tpl x)\veedot\toCalc{E_2}(\tpl x)$.
  % \item If $E=E_1-E_2$, then
  %   $\toCalc{E}(\tpl x)=\toCalc{E_1}(\tpl x)\setminus\toCalc{E_2}(\tpl x)$.
  % \item If $E=\sigma_{\varphi=\varphi'}(E_1)$
  % \end{itemize}
\end{proof}

%There is a homomorphism from $\phi_{w'}$ to $C(<\T, {R(a, w)}>)$ such that 
%$\tpl t_{\tpl w'} = \{z \mid h(z) = a\}$ (i.e., all the entry points of the 
%component $\tpl w'$ are mapped to $a$ and only these).
%individual $a$ used in the ABox is an individual from $\tpl t_{\tpl w'}$ (if 
%there 
%exists one) otherwise it is an arbitrary one.

\BALGComplete*
\begin{proof}
  The complexity class $\balgonee$ is defined in~\cite{GrumbachJCSS96} by the
  problem of checking, given a bag database instance $I$, a tuple of individuals
  $\tpl b$, a number $n\ge 0$, and a fixed $\balgonee$ algebra expression $E$,
  whether the multiplicity of $\tpl b$ in bag $E(I)$ is exactly $n$.
  %
  % Note that each finite interpretation $\I$ (such as
  % $\can{\langle\emptyset,\A\rangle}$) can be seen as a database instance in the
  % sence of~\cite{GrumbachJCSS96}. 
  %
  % Let $E$ be the $\balgonee$ algebra expression corresponding to $q$ according
  % to Proposition~\ref{prop:calc-to-alg}.
  %
  We next reduce the problem of checking
  $\ans{q}{\can{\langle\emptyset,\A\rangle}}(\tpl a)\ge k$ to the above
  problem. Without loss of generality we assume that $k \in \nati \cup \{0\}$, since 
  inequality  
  $\ans{q}{\can{\langle\emptyset,\A\rangle}}(\tpl a)\ge k$ is always false
  whenever $k = \infty$. 
  % Note that every finite bag interpretation $\I$, such as
  % $\can{\langle\emptyset,\A\rangle}$, can be equivalently seen as a database
  % instance $I_\I$ in the sense of~\cite{GrumbachJCSS96}. Thus, for each % 
  % algebra expression $E$, we have $\ans E\I=E(I_\I)$. 
  For this consider the following definitions:
  \begin{itemize}
  \item[--] %$I$ as the database instance corresponding to 
    Let $\I$ be the extension of
    $\can{\langle\emptyset,\A\rangle}$ such that
\begin{align*}
\Delta^\I    & = \Delta^{\can{\langle\emptyset,\A\rangle}},\\
S^\I(\tpl t) & = 
\begin{cases}
k & \text{, if } \tpl t = \tpl a\\
0 & \text{, otherwise}
\end{cases}, \text{ for some } S\in\Concepts\cup\Roles \text{ not occurring in 
} \A,\\
        T^\I & = T^{\can{\langle\emptyset,\A\rangle}}, \text{ for each } 
        T\in\Concepts\cup\Roles \text{ different from } S.
\end{align*}
  Note that $\I$ is finite relative to the finite set $\cal S$ consisting of 
  the predicate symbols in $\A$ and symbol $S$, hence $I_\I$ is a bag database 
  instance with schema $\cal S$.
%    $\Delta^\I=\Delta^{\can{\langle\emptyset,\A\rangle}}$, for some 
%    $S\in\Concepts\cup\Roles$ not occurring in $\A$ and for every tuple $\tpl 
%    t$ of elements in $\Delta^\I$, $S^\I(\tpl t) = k$ if $\tpl t = \tpl a$ and 
%    $S^\I(\tpl t) = 0$ otherwise, 
%    $T^\I(\tpl c)=T^{\can{\langle\emptyset,\A\rangle}}(\tpl c)$ for each $\tpl 
%    c$ and each $T\in\Concepts\cup\Roles$ different from $S$;
  \item[--] Let $\tpl b=\tpl a$ and $n=0$.
  \item[--] Let $E$ be the algebra expression corresponding to query $S(\tpl x) 
  \setminus q(\tpl x)$ according to Proposition~\ref{prop:calc-to-alg}.
  \end{itemize}
  Then, $\ans{q}{\can{\langle\emptyset,\A\rangle}}(\tpl a)\ge k$ if and only if
  the multiplicity of $\tpl a$ in $E(I_\I)$ is 0.
  Indeed, suppose
  $\ans{q}{\can{\langle\emptyset,\A\rangle}}(\tpl a)< k$. Then
  $\ans{(S(\tpl x)\setminus q(\tpl x))}{\I}(\tpl a)>0$, thus by 
  Proposition~\ref{prop:calc-to-alg} the multiplicity of $\tpl a$ in $E(I_\I)$ 
  is greater than 0; the other direction is analogous.

  The above many-one reduction can be seen to be computable, for each $\A$,
  $\tpl a$, and $k$, by a Boolean circuit whose depth depends only on $q$. We
  conclude that the language
  $\set{\langle \A,\tpl
    a,k\rangle\mid\ans{q}{\can{\langle\emptyset,\A\rangle}}(\tpl a)\ge k}$ is
  contained in $\balgonee$ under $\logspace$-uniform $\aczero$ reductions, as
  required.
\end{proof}

In the following, for a CQ possibly with inequalities
$q(\tpl x) = \exists \tpl y.\, \phi(\tpl x, \tpl y)$
and a bag interpretation $\I$, 
we call a valuation 
$\lambda: \tpl x \cup \tpl y \cup \Ind \to \Delta^\I$ 
a \emph{homomorphism} from $q$ to $\I$ if 
for every atom $P(\tpl t)$ of $\phi(\tpl x, \tpl y)$, it holds
$P^\I(\lambda(\tpl t)) \ge 1$.
\begin{proposition}
\label{prop:anch-entry-points}
For any rooted CQ 
$q(\tpl x) = \exists \tpl y.\, \phi(\tpl x, \tpl y)$,
$\dlliteCbag$ ontology $\K$,
and equality-consistent subset $\tpl z \subseteq \tpl y$, 
we have
\begin{equation}
\label{eq:rw-proof-q-eq}
[q, \tpl z]^{\can\K} = 
%[q \wedge \bigwedge_{\text{ma-connected } \tpl z' \subseteq \tpl z} 
%\;
%\bigwedge_{t_1, t_2 \in \tpl t_{\tpl z'}, t_1 \not\in \Ind} (t_1 = t_2), \tpl 
%z]^{\can\K} =
[\exists \tpl y.\, 
\phi_{\tpl {\bar z}} \wedge 
\bigwedge_{\text{ma-connected } \tpl z' \subseteq \tpl z} \; 
\Big(\phi_{\tpl z'} \wedge
\bigwedge_{y \in \tpl t_{\tpl z'} \cap \Vars, t \in \tpl t_{\tpl z'}} (y = 
t)\Big),\
%t_1, t_2 \in \tpl t_{\tpl z'}, t_1 \not\in \Ind} (t_1 = t_2)\Big),\ 
\tpl z]^{\can\K}
\end{equation}
where
\emph{(i)} $\tpl {\bar z}$ are all the terms of $\phi(\tpl x, \tpl y)$ not 
appearing in $\tpl z$,  
\emph{(ii)} $\phi_{\tpl {\bar z}}$ is the subconjunction of 
$\phi(\tpl x, \tpl y)$ that consists of atoms and equalities mentioning only 
terms in $\tpl {\bar z}$,
\emph{(iii)} $\phi_{\tpl z'}$ is the subconjunction of 
$\phi(\tpl x, \tpl y)$ that consists of atoms and equalities mentioning a 
variable in $\tpl z'$, and
\emph{(iv)} $\tpl t_{\tpl z'}$ are all the terms of $\tpl {\bar z}$ appearing 
in $\phi_{\tpl z'}$.
\end{proposition}

\begin{proof}
First observe that by Definition~\ref{def:ma-connected} of ma-connected subsets 
of $\tpl z$ the following hold: 
\emph{(i)} if a variable $z$ belongs to a ma-connected subset $\tpl z'$ of 
$\tpl z$, then $\tpl z'$ contains all variables in $\eqc z$ plus all variables 
in the equivalence classes that are reachable from $\eqc z$ in the Gaifman 
graph of $q$ through nodes in $\eqc {\tpl z}$;
\emph{(ii)} any two ma-connected subsets of $\tpl z$ do not have any atom or 
equality in common.
Combining the above observations with the fact that $q$ is rooted, we derive 
that the query appearing on the right-hand side of 
equation~\eqref{eq:rw-proof-q-eq}, name it $q'$, contains 
exactly query $q$ plus a number of equalities between terms connecting 
$\phi_{\tpl {\bar z}}$ with the subconjunction corresponding to a ma-connected 
subset of $\tpl z$.
Let $\tpl a$ be a tuple of individuals and $\Lambda_{\tpl z}$ be the set of 
homomorphisms for $[q, \tpl z]^{\can\K}$ and $\tpl a$.
First, suppose that $\Lambda_{\tpl z} = \emptyset$, that is, 
$[q, \tpl z]^{\can\K}(\tpl a) = 0$. This means that there is no homomorphism 
$\lambda$ from $q$ to $\can\K$ with $\eval(\tpl x) = \tpl a$ that 
satisfies the equalities of $q$. But then, the same is true for $q'$, hence 
$[q', \tpl z]^{\can\K}(\tpl a) = 0$ as well.
Suppose now that there is a homomorphism $\lambda$ satisfying the equalities of 
$q$. This means that $\lambda$ satisfies the equalities of $q'$ as well except 
possibly the extra ones. 
We next prove that $\lambda$ satisfies these extra equalities for each 
non-empty ma-connected subset $\tpl z'$ of $\tpl z$ as well.
For this, consider such a subset $\tpl z'$.
Because $q$ is rooted, 
$\phi_{\tpl z'}$ should be connected with 
$\phi_{\tpl {\bar z}}$, hence, $\phi_{\tpl z'}$ contains atoms
$P_i(t_i, z_i)$ (resp., $P_i(z_i, t_i)$) such that $t_i \not\in \tpl z'$, 
$z_i \in \tpl z'$, and $i \in [1, n]$
(note that if $\phi_{\tpl {\bar z}}$ is empty the above still holds, because if 
we assume the opposite, then 
$\bigwedge_{\text{ma-connected } \tpl z' \subseteq \tpl z}\phi_{\tpl z'}$, 
as a rooted query, should contain a distinguished variable or an equality 
mentioning an individual, but since for every $\tpl z'$ we have 
$\tpl z' \subseteq \tpl z \subseteq \tpl y$ and $\tpl z$ is 
equality-consistent, we derive a contradiction in both cases). 
Suppose by contradiction that there is a pair $i, j$ 
such that
%$P_i \not= P_j$
%or 
$\eval(t_i) \not= \eval(t_j)$.
%In either case, 
Because $\eval(t_i), \eval(t_j) \in \Ind$, 
by definition of canonical models, we have
$\eval(z_i) = w_{\eval(t_i), P_i}$ and 
$\eval(z_j) = w_{\eval(t_j), P_j}$. 
But then, because $\tpl z'$ is ma-connected, $\phi_{\tpl z'}$ contains an atom 
that is sent by $\lambda$ to a tuple $(w_1, w_2)$ such that $w_1$ is either 
$w_{\eval(t_i), P_i}$ or an anonymous generated by $w_{\eval(t_i), P_i}$
and $w_2$ is either $w_{\eval(t_j), P_j}$ or an anonymous generated by $w_2$. 
Given that the anonymous elements of canonical models are characterised by the 
individual and the role that generated them and there can be no tuple having 
anonymous elements generated from different combination, the above situation is 
impossible.
Hence, not only we have $\eval(t_i) = \eval(t_j)$ for every $i, j \in [1, n]$, 
but also that $P_i = P_j$ holds. 
From the former it follows that the extra equalities in $q'$ are satisfied by 
all homomorphisms in $\Lambda_{\tpl z}$.
\end{proof}

\newcommand{\canaux}{{\can{\langle \T, \A' \rangle}}}
\newcommand{\qazp}{{q^a_{\tpl z'}}}

\begin{proposition}
\label{prop:unique-match}
Let $q(\tpl x) = \exists \tpl y.\, \phi(\tpl x, \tpl y)$
be a rooted CQ,
$\K = \langle \T, \A \rangle$ a $\dlliteCbag$ ontology, and
$\tpl z$ an equality-consistent subset of $\tpl y$.
For all non-empty, ma-connected, and realisable by $\T$ subsets $\tpl z'$ of 
$\tpl z$, we have
$\qazp^\canaux(\emptytpl) = 1$
where individual $a$ and ABox $\A'$ are picked according to 
Section~\ref{sec:q-rw}.
\end{proposition}

\begin{proof}
Consider a non-empty, ma-connected, and realisable by $\T$ subset $\tpl z'$ of 
$\tpl z$. 
Let us first inspect query $\qazp$. 
For $\tpl x' = \tpl t_{\tpl z'} \cap \Vars$, 
$a$ an individual in $\tpl t_{\tpl z'}$ if it exists or a fresh individual 
otherwise, we have
$
\qazp() =
       \exists \tpl x'.\, \exists \tpl z'. \;
       \phi_{\tpl z'} \land \; 
       \bigwedge_{t \in \tpl t_{\tpl z'}} (t = a) \; 
       \land \; 
       \bigwedge_{z \in \tpl z'} (z \neq a)
$
where, for $t \in \tpl t_{\tpl z'}$ and $z \in \tpl z'$, $\A'$ is the bag ABox 
having either only assertion $P(a, b)$ (with multiplicity 1), when 
$\alpha_{\tpl z'} = P(t, z)$, or only assertion $P(b, a)$, 
when $\alpha_{\tpl z'} = P(z, t)$.
Since $\tpl z'$ is realisable by $\T$,
by Definition~\ref{def:realisable-ma-connected} we have 
$(\qazp)^\canaux(\emptytpl) \ge 1$. Suppose that
$(\qazp)^\canaux(\emptytpl) > 1$. 
Observe that every concept/role extension under $\canaux$ 
is a set, hence, the multiplicity of a tuple in any such extension is $1$.
This means that all homomorphisms from $\qazp$ to $\canaux$ that satisfy the 
equalities and inequalities of $\qazp$ contribute multiplicity $1$ for 
$\emptytpl$ in bag $(\qazp)^\canaux$, hence, there exist at least two such 
homomorphisms.
Because $\tpl z'$ is ma-connected, the subgraph of the Gaifman graph of $q$ 
induced by the set of the equivalence classes of $\tpl z'$ is connected, and as 
a result, the Gaifman graph of $\phi_{\tpl z'}$ in $\qazp$ contains a single 
connected component. 
But then, because $\qazp$ contains equalities between all terms in $\tpl 
t_{\tpl z'}$ and individual $a$, all homomorphisms from $\qazp$ to $\canaux$ 
send all atoms containing a term in $\tpl t_{\tpl z'}$ to the assertion of 
$\A'$. Hence, $\phi_{\tpl z'}$ is essentially rooted.
Therefore, in order to have two homomorphisms from $\qazp$ to $\canaux$, it 
is necessary that $\qazp$ contains an atom $P(x, y)$ such that it can be sent 
to two different tuples, say $(u, v_1)$ and $(u, v_2)$ in $P^{\canaux}$.
However, observe that for every element $u$ of $\Delta^\canaux$ and for every 
role $P$, $\canaux$ may contain at most one $P$-successor for $u$, which 
implies that every atom $P(x, y)$ in $\qazp$ is sent to a tuple in $P^\canaux$ 
in a unique way.
Combining all of the above, we conclude that there is only one homomorphism from
$\qazp$ to $\canaux$ satisfying the equalities and inequalities of $\qazp$, 
which proves the claim.
\end{proof}

\BALGRewriteLemma*
\begin{proof}
Let $\Lambda_{\tpl z}$ be the set of valuations
$\lambda: \tpl x \cup \tpl y \cup \Ind \to \Delta^{\can\K}$
corresponding to Definition~\ref{def:eval-nulls} for $[q, \tpl z]^{\can\K}$.

\emph{1.}
Because variables $\tpl z$ are realisable by $\T$, by 
Definition~\ref{def:realisability} we have that $\tpl z$ are 
equality-consistent and all non-empty ma-connected subsets $\tpl z'$ of $\tpl 
z$ are realisable by $\T$. 
The former implies that there is no equality
$z = t$ with $z \in \tpl z$ and $t \not\in \tpl z$ in any $\phi_{\tpl z'}$. 
The latter implies that for any non-empty ma-connected subset $\tpl z'$ with 
$\tpl x' = \tpl t_{\tpl z'} \cap \Vars$, 
$a$ an individual in $\tpl t_{\tpl z'}$ if it exists or a fresh individual 
otherwise, and
query 
$
\qazp() =
       \exists \tpl x'.\, \exists \tpl z'. \;
       \phi_{\tpl z'} \land \; 
       \bigwedge_{t \in \tpl t_{\tpl z'}} (t = a) \; 
       \land \; 
       \bigwedge_{z \in \tpl z'} (z \neq a)
$,
we have 
$(\qazp)^{\canaux}(\emptytpl) \ge 1$
where, for $t_{\tpl z'} \in \tpl t_{\tpl z'}$ and $z_{\tpl z'} \in \tpl z'$, 
$\A'$ is the bag ABox 
having either only assertion $P(a, b)$ (with multiplicity 1), when 
$\alpha_{\tpl z'} = P(t_{\tpl z'}, z_{\tpl z'})$, or only assertion $P(b, a)$, 
when $\alpha_{\tpl z'} = P(z_{\tpl z'}, t_{\tpl z'})$
(note that by the proof of Proposition~\ref{prop:anch-entry-points}, atom 
$\alpha_{\tpl z'}$ always exists).
This means that there exists a homomorphism $\nu$ from $\phi_{\tpl z'}$ to 
$\canaux$ 
such that $\nu$ maps a term $t$ of $\qazp$ to $a$ if and only if $t \in \tpl 
t_{\tpl z'}$.
Consider now a valuation $\lambda \in \Lambda_{\tpl z}$, the subconjunction 
$\phi_{\tpl z'}$ of $q$, and atom $\alpha_{\tpl z'} = P(t_{\tpl z'}, z_{\tpl 
z'})$ (resp.,  
$\alpha_{\tpl z'} = P(z_{\tpl z'}, t_{\tpl z'})$) of $\phi_{\tpl z'}$.
By Definition~\ref{def:eval-nulls},
$\lambda$ maps term $t_{\tpl z'}$ to individuals and variable $z_{\tpl z'}$ to 
anonymous elements.
Therefore, whenever $\lambda$ is a homomorphism from $q$ to $\can\K$, it means 
that $P^{\can\K}$ contains a tuple 
$(a', w_{a', P})$ for some individual $a'$. But then, this means that 
$\can\K$ contains an isomorphic copy of $\canaux$ modulo the individuals $a$ 
and $b$ in $\A'$.
Hence, inequality 
$(q^{a'}_{\tpl z'})^{\can\K}(\emptytpl) \ge 1$
holds as well, which due to the above observations it implies that either $a'$ 
is the 
only individual contained in $\phi_{\tpl z'}$ and $a = a'$ or $\phi_{\tpl z'}$ 
does not mention any individual and the choice of $a$ above is irrelevant.
By Proposition~\ref{prop:unique-match}, we have 
$(\qazp)^{\canaux}(\emptytpl) = 1$, hence, 
$(q^{a'}_{\tpl z'})^{\can\K}(\emptytpl) = 1$ as well, which implies that the 
images of 
all atoms in $\phi_{\tpl z'}$ under $\lambda$ have multiplicity $1$.
Based on this fact, we derive the following equivalence for every ma-connected 
subset $\tpl z'$ of $\tpl z$.
\begin{equation}
\label{eq:rw-proof-1}
[\exists \tpl z'.\,\phi_{\tpl z'} \wedge 
\bigwedge_{y \in \tpl t_{\tpl z'} \cap \Vars, t \in \tpl t_{\tpl z'}} (y = t),\
\tpl z']^{\can\K}
=
[\exists z_{\tpl z'}.\,\alpha_{\tpl z'} \wedge 
\bigwedge_{y \in \tpl t_{\tpl z'} \cap \Vars, t \in \tpl t_{\tpl z'}} (y = t),\
z_{\tpl z'}]^{\can\K}
\end{equation}
We now inspect the form of $q_{\tpl z}$. 
Let 
$\tpl {\bar z}$ be all terms of $\phi(\tpl x, \tpl y)$ not appearing in $\tpl 
z$, 
$\tpl y_{\tpl {\bar z}} = \Vars \cap \tpl {\bar z}$,  
$\phi_{\tpl {\bar z}}$ the subconjunction of $\phi(\tpl x, \tpl y)$ consisting 
of atoms and equalities mentioning only terms in $\tpl {\bar z}$, 
and
$\tpl z_{\tpl z} = \bigcup_{\text{ma-connected } \tpl z' \subseteq \tpl 
z}\{z_{\tpl z'} \mid z_{\tpl z'} \text{ appears in } \alpha_{\tpl z'} \}$
(note that $\tpl z_{\tpl z} = \tpl z \cap \tpl y'$). Then, 
$q_{\tpl z}$ takes the following form:
\begin{equation}
\label{eq:rw-proof-q_z}
q_{\tpl z}(\tpl x) = \exists {\tpl y_{\tpl {\bar z}}}.\,
\exists \tpl z_{\tpl z}.\,
\phi_{\tpl {\bar z}} \wedge 
\bigwedge_{\text{ma-connected } \tpl z' \subseteq \tpl z} \Big(\alpha_{\tpl z'} 
\land 
\bigwedge_{y \in \tpl t_{\tpl z'} \cap \Vars, t \in \tpl t_{\tpl z'}} (y = 
t)\Big)
\end{equation}
Notice that, for all non-empty ma-connected subsets $\tpl z'$ of $\tpl z$, the 
query on the left-hand side of equation~\eqref{eq:rw-proof-1} corresponds to a 
conjunction w.r.t.\ $\tpl z'$ in the query at the right-hand side of 
equation~\eqref{eq:rw-proof-q-eq} of Proposition~\ref{prop:anch-entry-points}, 
while both map their common variables in $\tpl z$ to anonymous elements 
when evaluated over $\can\K$. 
By Definition~\ref{def:ma-connected} all ma-connected subsets of 
$\tpl z$ are pairwise disjoint, while their union makes up $\tpl z$.
Hence, considering equation~\eqref{eq:rw-proof-1} for all ma-connected subsets 
of $\tpl z$ and combining it with~\eqref{eq:rw-proof-q-eq} 
and~\eqref{eq:rw-proof-q_z}, 
we immediately derive
$[q, \tpl z]^{\can\K} = [q_{\tpl z}, \tpl z]^{\can\K}$.

\emph{2.}
Since $\tpl z$ is not realisable by $\T$, by Definition~\ref{def:realisability} 
we have that either 
$\phi(\tpl x, \tpl y)$ contains an equality $z = t$ with $z \in \tpl 
z$ and $t \not\in \tpl z$ 
or that
there is a non-empty ma-connected subset $\tpl z' \subseteq \tpl z$ 
which is not realisable by $\T$.
For the former, by Definition~\ref{def:eval-nulls}
all $\lambda \in \Lambda_{\tpl z}$ are such that
$\lambda(z) \not= \lambda(t)$, hence, equality $z = t$ is not satisfied by any 
of them, which results in 
$[q, \tpl z]^{\can\K} = \emptybag$.
For the latter, this means that, 
for $\tpl x' = \tpl t_{\tpl z'} \cap \Vars$, 
$a$ an individual in $\tpl t_{\tpl z'}$ if it exists or a fresh individual 
otherwise, and
query 
$
\qazp() =
       \exists \tpl x'.\, \exists \tpl z'. \;
       \phi_{\tpl z'} \land \; 
       \bigwedge_{t \in \tpl t_{\tpl z'}} (t = a) \; 
       \land \; 
       \bigwedge_{z \in \tpl z'} (z \neq a)
$,
we have 
$(\qazp)^{\canaux}(\emptytpl) = 0$
where, for $t \in \tpl t_{\tpl z'}$ and $z \in \tpl z'$, $\A'$ is the bag ABox 
having either only assertion $P(a, b)$ (with multiplicity 1), when 
$\alpha_{\tpl z'} = P(t, z)$, or only assertion $P(b, a)$, when $\alpha_{\tpl 
z'} = P(z, t)$.
We distinguish two cases: 
\emph{(i)}
either there is no homomorphism $\nu$ from $\phi_{\tpl z'}$ to $\canaux$ or 
\emph{(ii)}
there is one but it violates some equality or inequality of $\qazp$.
Assume there exists $\lambda \in \Lambda_{\tpl z}$ that is a 
homomorphism from $\phi_{\tpl z'}$ to $\can\K$ (otherwise we trivially have
$[q, \tpl z]^{\can\K} = \emptybag$).
From the existence of $\lambda$ and atom $\alpha_{\tpl z'}$ in 
$\phi_{\tpl z'}$, we have that $\can\K$ contains an isomorphic copy of 
$\canaux$, modulo individuals $a$ and $b$ (see also the proof of statement 1.). 
But then, considering case~\emph{(i)}, if there is no homomorphism $\nu$ from 
$\phi_{\tpl z'}$ to $\canaux$, this means there is no homomorphism $\lambda$ 
from $\phi_{\tpl z'}$ to $\can\K$ either, which is a contradiction (this is 
easily seen by the fact that there is a homomorphism from the image of 
$\phi_{\tpl z'}$ under $\lambda$ to $\canaux$ that preserves individuals and 
the fact that the composition of homomorphisms is another homomorphism).
Considering case~\emph{(ii)}, it means that
either there are two different individuals in the query connected with 
variables in $\tpl z'$ or that $\nu$ maps $z$ to $a$.
But then, the former means that $\phi_{\tpl z'}$ contains atoms $P_1(a, z_1)$ 
and $P_2(a', z_2)$ such that $a \not= a'$ with $P_1, P_2$ not necessarily 
distinct that contradicts Proposition~\ref{prop:anch-entry-points}, which 
requires all $t \in \tpl t_{\tpl z'}$ be mapped to the same individual by 
$\lambda \in \Lambda_{\tpl z}$.
On the other hand, if $\nu(z) = a$, then $\lambda(z) = a$ if $a \in \tpl 
t_{\tpl z'}$, otherwise $\lambda(z) = a'$ for some individual $a'$ in 
$\Delta^{\can\K}$. In either case, $\lambda(z) \in \Ind$, from which we conclude
that $\lambda \not\in \Lambda_{\tpl z}$, which is a contradiction.
Therefore, there is a subquery $\phi_{\tpl z'}$ in $q$ for which there is no 
$\lambda \in \Lambda_{\tpl z}$ that is a homomorphism from $\phi_{\tpl z'}$ to 
$\can\K$ and satisfies the equalities of $\phi_{\tpl z'}$, from which it 
follows that  
$[q, \tpl z]^{\can\K} = \emptybag$.
\end{proof}

For a $\dlliteCbag$ TBox $\T$, a concept $C$, a role $R$, and a term $t$, 
where $\zeta_{C_0}(t)$ is defined in Section~\ref{sec:q-rw}, we 
define the expressions $\eta_C(t)$ and $\theta_{\exists R}(t)$ as follows:
\begin{align}
\eta_C(t) & = 
\bigvee_{\T \models C_0 \isa C} \zeta_{C_0}(t) \label{eta} \\
\theta_{\exists R}(t) & = \big(\bigvee_{\T \models C_0 \isa \exists R} 
\zeta_{C_0}(t) \big) \,\backslash \, \zeta_{\exists R}(t) \label{theta}
\end{align}

\begin{lemma}
\label{lemma:rw-atomic-ucq}
For a $\dlliteCbag$ ontology $\K = \langle \T, \A \rangle$ the following hold:
\begin{enumerate}
\item \label{rw-concept}
For a query of the form
$q_C(x) = \zeta_C(x)$ and an individual $a$, we have 
$q_C^{\can\K}(a) = \eta_C^{\can{\langle \emptyset, \A\rangle}}(a)$.

\item \label{rw-role}
For a role $R \in \Roles$, queries of the form 
$q_R(x, y) = R(x, y)$ and
$q_{R^-}(x, y) = R(y, x)$, and a pair of individuals $\tpl a$, we have 
$q_R^{\can\K}(\tpl a) = q_R^{\can{\langle \emptyset, \A\rangle}}(\tpl a)$ and 
$q_{R^-}^{\can\K}(\tpl a) = q_{R^-}^{\can{\langle \emptyset, \A\rangle}}(\tpl 
a)$.

\item \label{rw-role-join}
For a role $R$, a query of the form $q(x) = R(x, x)$, and an individual $a$, 
we have 
$q^{\can\K}(a) = q^{\can{\langle \emptyset, \A\rangle}}(a)$.
\end{enumerate}
\end{lemma}

\begin{proof}
1. Recalling the definition of canonical models, it is immediate to verify that 
$C^{\can{\K}}(a) = C^{\pcan{\K}{1}} = \ccl{a}{\pcan{\K}{0}}{\T}(C)$ holds.
%By definition of $\pcan{\K}{1}$ and the concept closure of $a$ in 
%$\pcan{\K}{1}$, we have
%$\ccl{a}{\pcan{\K}{1}}{\T}(C) = \ccl{a}{\pcan{\K}{0}}{\T}(C)$.
Observe that the latter quantity is equal to 
$\bigcup_{\T \models C_0 \isa C}(C_0^{\pcan{\K}{0}}(a))$.
Last, notice that 
$C_0^{\pcan{\K}{0}}(a) = \A(C_0(a))$ holds
whenever $C_0 \in \Concepts$, and
$C_0^{\pcan{\K}{0}}(a) = \sum_{c \in \Ind} \A(R(a, c))$ holds
whenever $C_0 = \exists R$ with $R \in \Roles$ (the case is analogous when $C_0 
= \exists R^-$ with $R \in \Roles$). 
Recalling the semantics of CQs, this means that
$C_0^{\pcan{\K}{0}}(a) = \zeta_C^{\can{\langle \emptyset, \A\rangle}}(a)$ holds.
Combining this with
$C^{\can{\K}}(a) = \bigcup_{\T \models C_0 \isa C}(C_0^{\pcan{\K}{0}}(a))$ and 
the definition of equation~\eqref{eta}, we derive 
$C^{\can\K}(a) = \eta_C^{\can{\langle \emptyset, \A\rangle}}(a)$, which is 
equivalent to,
$q_C^{\can\K}(a) = \eta_C^{\can{\langle \emptyset, \A\rangle}}(a)$.

2. Since $\T$ does not contain role inclusion axioms, 
for every $i \ge 0$ in the definition of the canonical model for $\K$,
every role $P$, and every pair of individuals $\tpl a$, we have
$P^{\pcan{\K}{i}}(\tpl a) = P^{\pcan{\K}{0}}(\tpl a)$. 
Moreover, by definition of canonical models, we have
$\can\K = \bigcup_{i\ge0}\pcan{\K}{i}$, from which and the definition of
$\cup$ for bag interpretations, we get
$P^{\can{\K}} = \bigcup_{i\ge0} P^{\pcan{\K}{i}}$. 
Combining this with the fact that
$P^{\pcan{\K}{i}}(\tpl a) = P^{\pcan{\K}{0}}(\tpl a)$, we derive
$P^{\can{\K}}(\tpl a) = P^{\pcan{\K}{0}}(\tpl a)$.
Last, because $P^{\pcan{\K}{0}}(\tpl a) = \A(P(\tpl a))$, we get 
$q_P^{\can\K}(\tpl a) = q_P^{\can{\langle \emptyset, \A\rangle}}(\tpl a)$, 
which proves the claim.

3. The claim follows from the facts that 
$\T$ does not contain any role inclusion axioms 
and that canonical interpretations do not add to role extentions tuples 
with repeated elements.
\end{proof}

\begin{lemma}
\label{lemma:rw-abox}
For a rooted CQ 
$q(\tpl x) = \exists \tpl y.\, \phi(\tpl x, \tpl y)$, a $\dlliteCbag$ ontology 
$\K = \langle \T, \A \rangle$, and variables $\tpl z \subseteq \tpl y$ that are 
realisable by $\T$, we have
$[q_{\tpl z}, \tpl z]^{\can\K} = ({\bar q}_{\tpl z})^{\can{\langle \emptyset, 
\A\rangle}}$.
\end{lemma}
\begin{proof}
Denote by $\tpl {\bar z}$ all terms of $\phi(\tpl x, \tpl y)$ not appearing in 
$\tpl z$, by $\phi_{\tpl {\bar z}}$ the subconjunction of $\phi(\tpl x, \tpl 
y)$ that consists of all atoms and equalities mentioning only terms in $\tpl 
{\bar z}$, and set $\tpl y_{\tpl {\bar z}} = \tpl y \setminus \tpl z$. 
Then, recalling the definition of $\phi_{\tpl z'}$ in Section~\ref{sec:q-rw}, 
$q$ is written as
\[ 
q(\tpl x) = \exists \tpl y_{\tpl {\bar z}}.\, \exists \tpl z.\,
\phi_{\tpl {\bar z}} \wedge 
\bigwedge_{\text{ma-connected } \tpl z' \subseteq \tpl z} \phi_{\tpl z'}.
\]
For a ma-connected subset $\tpl z'$ of $\tpl z$ and for $P \in \Roles$, let
$R_{\tpl z'} = P$,   if $\alpha_{\tpl z'} = P(t, z)$, and
$R_{\tpl z'} = P^-$, if $\alpha_{\tpl z'} = P(z, t)$.
Last, denote term $t$ and variable $z$ 
appearing in $\alpha_{\tpl z'}$ by $t_{\tpl z'}$ and $z_{\tpl z'}$, 
respectively, and let 
$\tpl z_{\tpl z} = \bigcup_{\text{ma-connected } \tpl z' \subseteq \tpl 
z}\{z_{\tpl z'} \mid z_{\tpl z'} \text{ appears in } \alpha_{\tpl z'} \}$.
Based on this, query $q_{\tpl z}$ takes the form
\[
q_{\tpl z}(\tpl x) = \exists \tpl y_{\tpl {\bar z}}.\,
\exists \tpl z_{\tpl z}.\,
\phi_{\tpl {\bar z}} \wedge 
\bigwedge_{\text{ma-connected } \tpl z' \subseteq \tpl z} \Big(\alpha_{\tpl z'} 
\land 
\bigwedge_{y \in \tpl t_{\tpl z'} \cap \Vars, t \in \tpl t_{\tpl z'}} (y = 
t)\Big).
\]
Last, denote by 
${\sf apply}(\phi_{\tpl {\bar z}}, \eta)$ 
the formula obtained from 
$\phi_{\tpl {\bar z}}$ such that each occurrence of an atom $A(t)$ in 
$\phi_{\tpl {\bar z}}$, where $A \in \Concepts$ and $t \in \Ind \cup \Vars$,  
is replaced with $\eta_A(t)$ (defined in equation~\eqref{eta}).
Recalling equation~\eqref{theta} that defines $\theta_{\exists R}(t)$ for a 
role $R$ and a term $t$, query $\bar q_{\tpl z}$ takes the form
\[
\bar q_{\tpl z}(\tpl x) = \exists \tpl y_{\tpl {\bar z}}.\,
{\sf apply}(\phi_{\tpl {\bar z}}, \eta) \wedge 
\bigwedge_{\text{ma-connected } \tpl z' \subseteq \tpl z} 
\Big(\theta_{\exists R_{\tpl z'}}(t_{\tpl z'}) \land 
\bigwedge_{y \in \tpl t_{\tpl z'} \cap \Vars, t \in \tpl t_{\tpl z'}} (y = 
t)\Big).
\]
Consider now Definition~\ref{def:eval-nulls} and the set of valuations 
accounting for bag
$[q_{\tpl z}, \tpl z]^{\can{\K}}$.
All such valuations map a variable in $\phi_{\tpl {\bar z}}$ to an individual.
Because ${\sf apply}(\phi_{\tpl {\bar z}}, \eta)$ replaces unary atoms 
$A(t)$ in $\phi_{\tpl {\bar z}}$ with $\eta_A(t)$ leaving any binary atoms 
intact, by Lemma~\ref{lemma:rw-atomic-ucq} (Case~1 for unary atoms and Case~2), it follows that the evaluation of $\phi_{\tpl {\bar z}}$ over $\can\K$ 
coincides with the evaluation of ${\sf apply}(\phi_{\tpl {\bar z}}, \eta)$ over 
$\can{\langle \emptyset, \A\rangle}$. 
It remains to be shown that, for each ma-connected subset $\tpl z'$ of $\tpl 
z$, equality
$[\exists z_{\tpl z'}.\,\alpha_{\tpl z'}, z_{\tpl z'}]^{\can\K} = 
(\theta_{\exists R_{\tpl z'}}(t_{\tpl z'}))^{\can{\langle \emptyset, 
\A\rangle}}$ holds.
Then, the claim will follow because $q_{\tpl z}$ and $\bar q_{\tpl z}$ share 
the same equalities.
To see this last equivalance, observe that for an individual $a \in \Ind$ 
(or for $a = t_{\tpl z'}$ if $t_{\tpl z'} \in \Ind$),
the multiplicity for
$[\exists z_{\tpl z'}.\,\alpha_{\tpl z'}, z_{\tpl z'}]^{\can\K}(a)$
corresponds to the number of anonymous elements associated with $a$ in
extension $\exists R_{\tpl z'}$ under $\can\K$.
Recalling the definition of canonical models and the proof of 
Lemma~\ref{lemma:rw-atomic-ucq},
number 
$[\exists z_{\tpl z'}.\,\alpha_{\tpl z'}, z_{\tpl z'}]^{\can\K}(a)$ can be 
written successively as follows:
\begin{align*}
[\exists z_{\tpl z'}.\,\alpha_{\tpl z'}, z_{\tpl z'}]^{\can\K}(a) &=
\ccl{a}{\pcan{\K}{0}}{\T}(\exists R_{\tpl z'}) - 
(\exists R_{\tpl z'})^{\pcan{\K}{0}}(a)\\
& = \Big(\bigcup_{\T \models C_0 \isa \exists R_{\tpl z'}} 
C_0^{\pcan{\K}{0}}(a)\Big) -
[\exists z_{\tpl z'}.\,\alpha_{\tpl z'}, \emptyset]^{\pcan{\K}{0}}(a)\\
& = \big(\eta_{\exists R_{\tpl z'}}(t_{\tpl z'})\big)^{\can{\langle \emptyset, 
\A\rangle}}(a) -
\big(\zeta_{\exists R_{\tpl z'}}(t_{\tpl z'})\big)^{\can{\langle \emptyset, 
\A\rangle}}(a)\\
& = \big((\eta_{\exists R_{\tpl z'}}(t_{\tpl z'})) \setminus \zeta_{\exists 
R_{\tpl z'}}(t_{\tpl z'})\big)^{\can{\langle \emptyset, \A\rangle}}(a)\\
& = \big(\theta_{\exists R_{\tpl z'}}(t_{\tpl z'})\big)^{\can{\langle 
\emptyset, \A\rangle}}(a).
\end{align*}
%we have
%$[\exists z_{\tpl z'}.\,\alpha_{\tpl z'}, z_{\tpl z'}]^{\can\K}(a) = 
%\ccl{a}{\pcan{\K}{0}}{\T}(\exists R_{\tpl z'}) - 
%(\exists R_{\tpl z'})^{\pcan{\K}{0}}(a)$, which is equivalently written as
%$[\exists z_{\tpl z'}.\,\alpha_{\tpl z'}, z_{\tpl z'}]^{\can\K}(a) = 
%\big(\bigcup_{\T \models C_0 \isa \exists R_{\tpl z'}} 
%C_0^{\pcan{\K}{0}}(a)\big) -
%[\exists z_{\tpl z'}.\,\alpha_{\tpl z'}, \emptyset]^{\pcan{\K}{0}}(a)$
%and then as
%$[\exists z_{\tpl z'}.\,\alpha_{\tpl z'}, z_{\tpl z'}]^{\can\K}(a) = 
%(\eta_{\exists R_{\tpl z'}}(t_{\tpl z'}))^\A(a) -
%(\zeta_{\exists R_{\tpl z'}}(t_{\tpl z'}))^\A(a) = 
%\big((\eta_{\exists R_{\tpl z'}}(t_{\tpl z'})) \setminus \zeta_{\exists 
%R_{\tpl 
%z'}}(t_{\tpl z'})\big)^\A(a)=
%(\theta_{\exists R_{\tpl z'}}(t_{\tpl z'}))^\A(a)$.
\end{proof}

\BALGRewrite*
\begin{proof}
Let $q$ be of the form $q(\tpl x) = \exists \tpl y. \phi(\tpl x, \tpl y)$.
From~\eqref{partition_eq} and then by Lemmas~\ref{lemma:rw-realisable} 
and~\ref{lemma:rw-abox}, we have
\[
q^{\can\K} 
 = \biguplus_{{\tpl z} \subseteq \tpl y} [q, \tpl z]^{\can\K}
 = \biguplus_{\tpl z \text{ is realisable by } \T} [q_{\tpl z}, {\tpl 
 z}]^{\can\K}
 = \biguplus_{\tpl z \text{ is realisable by } \T} ({\bar q}_{\tpl 
 z})^{\can{\langle \emptyset, \A\rangle}}
 = \big(\bigveedot_{\tpl z \text{ is realisable by } \T} {\bar q}_{\tpl 
  z}\big)^{\can{\langle \emptyset, \A\rangle}}
 = {\bar q}^{\can{\langle \emptyset, \A\rangle}}.
\]
\end{proof}

\BALGRewriteCor*
\begin{proof}
The claim is an immediate consequence of Theorems~\ref{thm:chase-univ} 
and~\ref{thm:cq-rw}, 
%the semantics of CQs (Definition~\ref{def:CQ-eval}), 
and the fact that rewriting of a CQ $q$ depends only on $q$ and TBox.
\end{proof}

\newcommand{\guessi}{({\sf g1}){ }}
\newcommand{\guessv}{({\sf g2}){ }}
\newcommand{\verv}{({\sf v1}){ }}
\newcommand{\verm}{({\sf v2}){ }}
\QAComplexity*
\begin{proof}
For combined complexity, $\np$-hardness comes from reducing CQ query 
answering in $\dlliteC$ to CQ query answering for rooted queries in 
$\dlliteCbag$. It is known that given a CQ $q$ in $\dlliteC$, an ontology $\K = 
\langle \T, \A \rangle$, and a tuple of individuals $\tpl a$, the problem of 
deciding whether $\tpl a \in q^{\K}$ is $\np$-complete 
even when all variables in $q$ are free (i.e., $q$ is rooted). Then, 
$\np$-hardness follows from Proposition~\ref{thm:back-compat}. 

We now discuss membership in $\np$.
$\BagCert[\textup{rooted CQs},\dlliteCbag]$ decides for any given 
ontology $\K = \langle \T, \A\rangle$, rooted query 
$q(\tpl x) = \exists \tpl y.\, \phi(\tpl x, \tpl y)$, 
tuple of individuals $\tpl a$, and number $k \in \natinfz$, 
whether $q^\K(\tpl a) \geq k$. 
Without loss of generality in the following we assume that $k$ is a positive 
integer since for the cases of $k$ being $0$ or $\infty$, 
$\BagCert[\textup{rooted CQs},\dlliteCbag]$ is always true or false, 
respectively. 
By Theorem~\ref{thm:chase-univ} and the definition of universal models, we have 
$q^\K(\tpl a) = q^{\can\K}(\tpl a)$. Hence, in the following, we focus on the 
canonical model of $\K$.
By the semantics of bag query answering, we 
might have a number of valuations from the terms of $q$ to $\Delta^{\can\K}$ 
that contribute to the multiplicity of $q^{\can\K}(\tpl a)$.
In the worst case, each valuation may contribute $1$ to this multiplicity, 
which means that to verify whether
$q^{\can\K}(\tpl a) \geq k$ holds, we need at most $k$ different valuations 
$\lambda_1,\dots,\lambda_k$
that send the atoms of $q$ to tuples in $\can\K$,
bind $\tpl x$ to $\tpl a$, and 
satisfy the equalities of $q$.
By Theorem~\ref{thm:qa-bounded-can} the images of $q$ under any of these 
valuations fall into $\pcan{\K}{n}$ where $n$ is the number of atoms of $q$. 
Based on these observations, to prove membership in $\np$, we describe how we 
can obtain a tuple
$\langle \J, \lambda_1, \dots, \lambda_k \rangle$
%for 
%$q^{\can\K}(\tpl a) \geq k$,
where $\J$ is a subinterpretation of $\pcan{\K}{n}$
and each 
$\lambda_i: \tpl x \cup \tpl y \cup \Ind \to \Delta^\J$
satisfies $\lambda(\tpl x) = \tpl a$,
and verify that the multiplicity of $q^\J(\tpl a)$ with respect to 
$\lambda_1,\dots,\lambda_k$ is at least $k$. 
Then, we also prove that 
$\langle \J, \lambda_1, \dots, \lambda_k \rangle$ 
has size $N$ and 
that verification can be done in time $T$
such that both $N$ and $T$ are some polynomials with respect to the size of 
$\K$, $q$, and number $k$.
To obtain 
$\langle \J, \lambda_1, \dots, \lambda_k \rangle$ 
we guess
\guessi an interpretation 
$\J = \pcan{\K}{1} \cup \bigcup_{i=1}^k \J_i$ where each $\J_i$ is a  
subinterpretation of $\bigcup_{j=2}^n \pcan{\K}{j}$ and
\guessv
$k$ different valuations $\lambda_1,\dots,\lambda_k$ such that
$\lambda_i: \tpl x \cup \tpl y \cup \tpl I \to \Delta^\J$
and
$\lambda(\tpl x) = \tpl a$.
For the verification,
%To verify that 
%$\langle \J, \lambda_1, \dots, \lambda_k \rangle$
%is indeed a certificate for $q^\K(\tpl a) \geq k$, 
we
\verv 
check which of the valuations $\lambda_1, \dots, \lambda_k$ satisfy the 
equalities of $q$ letting $\Lambda_=$ be the corresponding subset and
\verm
compute quantity 
$m = \sum_{\lambda \in \Lambda_=} 
     \prod_{S(\tpl t) \text{ in } \phi(\tpl x, \tpl y)} 
     S^{\J}(\lambda(\tpl t))$ 
for checking whether $m \ge k$. 

We now elaborate on the guessing of $\J$. 
The guessed $\J$ is such that $\Delta^\J$ is a finite set comprising all 
individuals appearing in $\A$ and a number of anonymous elements of the form 
$w^j_{u,R}$ where $j$ is a positive number, $u$ an element of $\Delta^\J$, and 
$R$ a role in $\T$. 
Further $\J$ contains finite bag extensions for every concept or role $S$ 
appearing in $\T$ or $\A$.
The part of $\J$ corresponding to $\pcan{\K}{1}$ can be trivially computed from 
the assertions of $\A$ and the axioms of $\T$. To avoid an exponential 
computation, however, we guess the remaining interpretations $\J_1,\dots,\J_k$ 
using a non-deterministic algorithm having $n-1$ steps for each $\J_i$. 
Initially, $\J_i$ is set to $\pcan{\K}{1}$.
At each step, 
%either chooses to perform no operation or performs the following.
%for $A \in \Concepts$ 
%$C, D$ concepts,  
the algorithm picks
a tuple $\tpl t$ from an extension $S^{\J_i}$ with
$S \in \Concepts \cup \Roles$ and
a concept $D$ appearing in $\T$ 
%an axiom $C \isa D$ from $\T$ 
such that the following conditions are satisfied:
\begin{itemize}
\item[--]
if $S \in \Concepts$, then
$\T \models S \isa D$,
$\tpl t = w^j_{u,R}$, 
%$C = S$, 
and $D^{\J_i}(w^j_{u,R}) = 0$
where
$w^j_{u,R} \in \Delta^{\J_i}$;

\item[--]
if $S \in \Roles$, then
$\T \models \exists S^- \isa D$,
$\tpl t = (u, w^j_{u,R})$,
%$C = \exists P^-$, 
and
$D^{\J_i}(w^j_{u,R}) = 0$
where
$u,w^j_{u,R} \in \Delta^{\J_i}$
(resp., $\T \models \exists S \isa D$, $\tpl t = (w^j_{u,R^-}, u)$, 
$w^j_{u,R^-} \in \Delta^{\J_i}$).
\end{itemize}
Then, 
if $D \in \Concepts$, it sets $D^{\J_i}(w^j_{u,R}) = 1$; 
if $D = \exists P$ with $P\in\Roles$, it sets
$P^{\J_i}((w^j_{u,R}, w^1_{w^j_{u,R}, P})) = 1$
and adds
$w^1_{w^j_{u,R}, P}$ to $\Delta^{\J_i}$; and
if $D = \exists P^-$ with $P\in\Roles$, it sets
$P^{\J_i}((w^1_{w^j_{u,R}, P^-}, w^j_{u,R})) = 1$
and adds
$w^1_{w^j_{u,R}, P^-}$ to $\Delta^{\J_i}$.
%a tuple of the form $(u, w^j_{u,R})$ 
%(resp., $(w^j_{u,R^-}, u)$) 
%in extension $S^{\J_i}$ and an axiom 
%$\exists S^- \isa \exists P$ 
%(resp., $\exists S \isa \exists P$) 
%in $\T$ are picked for some $S\in\Roles$ such that 
%$P^{\J_i}((w^j_{u,R}, w^1_{w^j_{u,R}, P})) = 0$,
%if $P \in \Roles$ or
%$(P^-)^{\J_i}((w^1_{w^j_{u,R}, P}, w^j_{u,R})) = 0$,
%otherwise. 
%Then, $\J_i$ is modified so that
%$P^{\J_i}((w^j_{u,R}, w^1_{w^j_{u,R}, P})) = 1$ 
%if $P \in \Roles$, or
%$(P^-)^{\J_i}((w^1_{w^j_{u,R}, P}, w^j_{u,R})) = 1$, otherwise.
It can be readily verified that each $\J_i$ can be any subinterpretation of 
$\pcan{\K}{n}$ that always includes $\pcan{\K}{1}$ and potentially tuples 
${\tpl t}_1, \dots, {\tpl t}_l$ that would have been created respectively in 
$\pcan{\K}{1},\dots,\pcan{\K}{l}$ such that 
${\tpl t}_j \cap {\tpl t}_{j+1} \not= \emptyset$ 
with $1\leq j < l$ and $l \in [2, n]$.

%Given an element $d \in \Delta^\J$, we denote by $|d|$ its size and define it 
%recursively as 
%$1$ if $d \in \Ind$ and
%$|u|+3$ if $d = w^j_{u,R}$. 
%Given the fact that we need to guess a subinterpretation of $\pcan{\K}{n}$, 
%these anonymous elements are predetermined because their recursive definition 
%satisfies the following:
%\emph{(i)}   element $u$ must be ultimately an individual in $\A$,
%\emph{(ii)}  role $R$ must appear in $\T$, and
%\emph{(iii)} if $\delta$ is the maximum multiplicity of an assertion or the 
%result of some projection in $\A$, then $j \in [1, \delta]$.
%Therefore, the size of all of these elements is in the range $[1, 3n+1]$, that 
%is, elements of $\pcan{\K}{1}$ have either size $1$ or $1+3$, while elements 
%of 
%$\pcan{\K}{i}$ have size $3i+1$.

What remains to be shown is that the size, $N$, of tuple
$\langle \J, \lambda_1,\dots,\lambda_k \rangle$
and time, $T$, needed to verify that this tuple is a certificate 
for $q^{\can\K}(\tpl a) \ge k$ are polynomials 
in the size of $\K$, $q$, and number $k$.
Consider first size $N$.
It is easy to see that $\pcan{\K}{1}$ has a size that is 
polynomial in the size of $\A$ and $\T$. The remaining parts of 
$\J_1,\dots,\J_k$ are linear in the size of $q$ by construction. Therefore, 
$\J$ is of polynomial size with respect to $\K$, $q$, and number $k$.
As for the size of $\lambda_1,\dots,\lambda_k$, because $q$ has $n$ atoms, it 
contains at most $2n$ terms, hence
%hence, each term can be represented using $\log 2n$ bits.
%Then, 
each valuation can be represented by $2n$ pairs $(t, d)$ 
where $t$ is a term in $q$ and $d \in \Delta^\J$. 
Therefore, overall, $N$ is polynomial in the size of $\A$, $q$, and number $k$.
%Since every $d$ takes $\log 
%|??|$ bits, the size of all valuations is 
%$O(2nk(\log 2n + \log|??|))$. 
Consider now quantity $T$. %the verification of the certificate. 
Step \verv takes time $\Theta(n)$.
Step \verm takes polynomial time in the size of $q$, $\A$, 
and number $k$. To see this, first observe that retrieval of the multiplicities 
involved in a product for a specific valuation takes $O(n \times |\J|)$ time 
where $|\J|$ is the sum of the cardinalities of all extensions in $\J$. 
Each such number $l$ 
is determined by the maximum multiplicity in $\A$ and can be represented in 
binary using $\log l$ bits. 
Second, multiplication of $n$ such numbers can be done in 
polynomial time, while the result, $m$, can be represented using 
$n\log l$ bits. Since $|\J|$ is a polynomial determined by the input,
overall verification can be done in polynomial time.
This proves that 
$\BagCert[\textup{rooted CQs},\dlliteCbag] \in \np$. 
Since we have also showed that it is $\np$-hard, we conclude that
$\BagCert[\textup{rooted CQs},\dlliteCbag]$ is $\np$-complete.

For data complexity, it suffices to prove that for any fixed $\dlliteCbag$ TBox 
$\T$ and any fixed rooted CQ $q$, 
the problem of checking whether 
$q^{\langle \T, \A \rangle}(\tpl a) \ge k$ 
is in $\logspace$
for an input ABox $\A$, a tuple of individuals $\tpl a$, and a $k \in \natinfz$.
The claim follows from Theorem~\ref{thm:chase-univ} and 
Theorem~\ref{thm:cq-rw}, which in combination, allow us to decide whether 
$q^{\langle \T, \A \rangle}(\tpl a) \ge k$ by computing the rewriting $\bar q$ 
of $q$ in constant time and then deciding whether 
${\bar q}^{\can{\langle \emptyset, \A\rangle}}(\tpl a) \ge k$ holds. 
By Proposition~\ref{prop:balg-complete}, the latter problem is $\aczero$ 
reducible to $\balgonee$, which is known to be strictly included in 
$\logspace$, 
hence, the claim follows immediately.
Computation of $\bar q$ can be done in constant time because of the following:
\emph{(i)} the number of all possible subsets $\tpl z$ of  $\tpl y$ 
participating in the realisability check by $\T$ is constant;
\emph{(ii)} computing all ma-connected subsets of $\tpl z$ can be done in 
constant time;
\emph{(iii)} checking whether each ma-connected subset of $\tpl z$ is 
realisable by $\T$ can be done in constant time by employing 
Theorem~\ref{thm:qa-bounded-can} for bounding the depth of $\canaux$ to a 
constant number (in particular, to the number of atoms of $q$); and
\emph{(iv)} constructing the rewriting of the original query for each 
realisable ma-connected subset of $\tpl z$ can be done in constant time.
\end{proof}

\QARolescoNP*
\begin{proof}
We prove that there exists a $\dlliteRbag$ TBox $\T$ and a rooted CQ $q$ 
such that checking whether $q^{\langle \T, \A \rangle}(\emptytpl) \ge k$ 
for an 
input bag 
ABox $\A$ and $k \in \natinfz$ is $\conp$-hard. 
To prove this claim, we give a similar reduction to the proof of 
Theorem~\ref{thm:core-conph}.
We show that if $G = \langle V, E \rangle$ is an undirected and connected graph 
with no self-loops, then 
$G$ is not 3-colourable if and only if 
$q^{\langle \T, \A_G \rangle}(r) \ge 3\times|V| + 2$
%${\sf certBag}(q(w), \langle \T, \A_G \rangle, \nu_r, 3\times|V| + 2) = {\sf 
%true}$
where
%$\nu_r$ is valuation $\{w/r\}$,
$\T$ is the TBox
$
\{ 
             Vertex \isa \exists hasColour, 
          hasColour \isa Assign
\}
$,
$\A_G$ is an ABox constructed based on $G$, and $q(w)$ is the rooted query
\begin{align*}
\exists x.\,\exists y.\,\exists z.\,\exists k.\,\exists l.\,
Edge(x, y) \wedge hasColour(x, z) \wedge & hasColour(y, z) \wedge\\
& Assign(x, w)  \wedge Assign(y, w) \wedge Reachable(x, k) \wedge Assign(k, l).
\end{align*}

Let $\Ind \supseteq V \cup \{a, r, g, b\}$. ABox $\A_G$ is defined so that it 
contains the following assertions:
\begin{itemize}
\item[--]
$Vertex(u)$ for each $u \in V$,

\item[--] 
$Edge(u, v)$, $Edge(v, u)$ for each $\pair{u, v} \in E$,

\item[--] 
$Assign(u, r)$, $Assign(u, g)$, $Assign(u, b)$ for each $u \in V$,

\item[--] 
$Vertex(a)$, $Edge(a, a)$, $hasColour(a, r)$, $Assign(a, r)$ for an 
auxiliary vertex $a \notin V$,

\item[--]
$Reachable(a, a)$, $Reachable(a, u)$ and $Reachable(u, a)$ for every $u \in V$, 
$Reachable(u, v)$ and $Reachable(v, u)$ for every $u, v \in V$ with $u \not=v$.
\end{itemize}

Role $hasColor$ plays the role of a colour assignment to the vertices of $G$; 
this is also imposed by axiom $Vertex \isa \exists hasColour$.
Role $Assign$ provides a pre-defined list of colours for every vertex of $G$ 
that favours 3-colour assignments based on the colours $r$, $g$, and $b$. Any 
proper assignment of $G$ shall use at most $|V|$ times each one of the colours. 
However, if any assignment is not proper and exhausts the number of available 
colours (i.e., by assigning multiple colours to the same vertex) or uses an 
additional colour, these will have to be added to role $Assign$ due to the 
axiom 
$hasColour \isa Assign$. 
Role $Reachable$ plays the role of an accessibility relation of an individual 
from any other individual. This property is used for counting the total number 
of available colours among all vertices.

We next show that 
$G$ is not 3-colourable if and only if 
$q^{\langle \T, \A_G \rangle}(r) \ge 3\times|V| + 2$.
%${\sf certBag}(q(w), \langle \T, \A_G \rangle, \nu_r, 3\times|V| + 2) = {\sf 
%true}$.

``$\Rightarrow$''
Let $G$ be non-3-colourable. 
Consider a model $\I$ of $\langle \T, \A_G \rangle$ (which exists since 
$\langle \T, \A_G \rangle$ is satisfiable) such that, if 
$\gamma: V \to \{r, g, b\}$ 
is an assignment of colours to the vertices of $G$, then for $u \not= a$,
$hasColour^\I(\pair{u^\I, c^\I}) = 1$ if and only if $\gamma(u) = c$ with $c 
\in \{r, g, b\}$.
Since $G$ is not 3-colourable, then, for all assignments $\gamma$, there exists 
at least an edge $\pair{u, v} \in E$ with $\gamma(u) = \gamma(v) = c$. Without 
loss of generality assume that $c = r$. 
Consequently, for all models $\I$, $hasColour^\I$ contains tuples
$\pair{u^\I, c^\I}$ and $\pair{v^\I, c^\I}$, and hence, subquery
\[
%\exists x,y,z. 
q_1(x, y, z, w) = 
Edge(x, y) \wedge hasColour(x, z) \wedge 
hasColour(y, z) \wedge Assign(x, w) \wedge Assign(y, w)
\]
matches at least two times, each one contributing multiplicity equal to $1$; 
one match corresponds to valuation 
$\nu_1 = \{ x/u^\I, y/v^\I, z/c^\I, w/r^\I \}$
and one to 
$\nu_2 = \{ x/a^\I, y/a^\I, z/r^\I, w/r^\I \}$ 
(note that we are considering only valuations $\nu$ with $\nu(w) = r^\I$).
Extending the above query to $q(w)$, we observe that $\nu_1$ can be extended 
with variables $k$ and $l$ in $3\times|V| - 2$ ways. To see this, observe that 
every node in $V$ is related to $|V|$ other nodes in $Reachable^\I$ of 
which $|V| - 1$ are related to at least $3$ colours in $Assign^\I$ while the 
other one, namely $a^\I$, is related to at least $1$. 
Similarly, $\nu_2$ can be extended with variables $k$ and $l$ in $3\times|V| + 
1$ ways.
Therefore, $q$ has at least $6\times|V| - 1$ matches for every model $\I$ 
following a proper 3-colour assignment, and hence, $3\times|V| + 2$ is a 
certain multiplicity for $r$, as required. 
Clearly, the same statement holds for all of the 
models that add additional elements in $Vertex$, $Edge$, or assign multiple 
colours to some vertices exceeding the number of available colours. 
What is left to consider is those models that assign additional colours to 
vertices and not just one among $r$, $g$, and $b$. For such colour assignments, 
$G$ might turn out to be colourable. Suppose $G$ is 4-colourable (if it is not, 
then the above discussion carries over) and let $p \in \Ind$. Then, there 
exists a model that follows a 4-colour assignment 
$\gamma: V \to \{r, g, b, p\}$ such that 
$\gamma(u) \not= \gamma(v)$ for every 
$\pair{u, v} \in E$. 
Therefore, for that model we would get just one match for subquery $q_1(x, y, 
z, w)$
%\[
%%\exists x, y, z.
% Edge(x, y) \wedge hasColour(x, z) \wedge 
% hasColour(y, z) \wedge Assign(x, w) \wedge Assign(y, w)
%\]
corresponding to valuation
$\nu_2$. % = \{ x/a^\I, y/a^\I, z/r^\I, w/r^\I \}$.
On the other hand, given the observations above, that model would have 
associated at least one vertex to colour $p$ in the extension of 
$hasCol^\I$, and hence in $Assign^\I$, effectively 
increasing by one the number of colours to which that vertex is associated. 
Therefore, extending the above subquery to $q$, we observe that $\nu$ can be 
extended with variables $k$ and $l$ in at least $3\times|V| + 2$ ways.
Clearly, the same holds for models that make use of further colours.
Therefore, $q^{\langle \T, \A_G \rangle}(\emptytpl) \ge 3\times|V|+2$.

``$\Leftarrow$'' 
Let $G$ be 3-colourable. It suffices to show that 
there exists a model $\I$ for which $q^\I(r) = m$ with 
$m < 3\times|V| + 2$.
Since $G$ is 3-colourable, there is an assignment 
$\gamma: V \to \{r, g, b\}$ 
such that, for every $\pair{u, v} \in E$, $\gamma(u) \not= \gamma(v)$.
Consider an interpretation $\I_\gamma$ defined as follows:
\begin{align*}
\Delta^{\I_\gamma} & = \{ d_c \mid c \in V \cup \{a, r, g, b\}\},\\
c^{\I_\gamma}      & = d_c ,\text{ for } c \in V \cup \{a, r, g, b\},\\
Vertex^{\I_\gamma} & = \{d_u \mid u \in V \cup \{a\}\}, \\
Edge^{\I_\gamma}   & = \{\pair{d_u, d_v}, \pair{d_v, d_u}, \mid \pair{u, v} \in 
E \} \cup \{\pair{d_a, d_a}\},\\
hasColour^{\I_\gamma} & = \{\pair{d_u, d_c} \mid u \in V, c = \gamma(u) \} \cup 
\{\pair{d_a, d_r}\},\\
Assign^{\I_\gamma} & = \{\pair{d_u, d_r}, \pair{d_u, d_g}, \pair{d_u, d_b} \mid 
u \in V \} \cup \{\pair{d_a, d_r}\},\\
Reachable^{\I_\gamma} & = \{ \pair{d_u, d_v}, \pair{d_v, d_u} \mid u, v \in V 
\text{ and } u \not= v \}\ \cup\ 
\{ \pair{d_a, d_u}, \pair{d_u, d_a} \mid u \in V \} \cup 
\{\pair{d_a, d_a}\}.\\
\end{align*}
Interpretation $\I_\gamma$ is defined based on the contents of $V$, $E$, and 
the 3-colour assignment $\gamma$.
It is easy to verify that $\I_\gamma$ is a model of $\langle \T, \A_G \rangle$.
Next, we show that $q^{\I_\gamma}(r) = 3\times|V| + 1$.
First, we observe that subquery $q_1(x, y, z, w)$
%\[
%%\exists x, y, z.
%Edge(x, y) \wedge hasColour(x, z) \wedge
%hasColour(y, z) \wedge Assign(x, w) \wedge Assign(y, w)
%\]
matches exactly once, i.e., under valuation
$\nu = \{ x/d_a, y/d_a, z/d_r, w/d_r \}$. 
This holds because $\gamma$ is a proper 
3-colouring of $G$ and, for every $\pair{u, v} \in E$, 
$\gamma(u) \not= \gamma(v)$.
Note also that extending the above subquery to $q(w)$, valuation $\nu$ can be 
extended with variables $k$ and $l$ in $3\times|V| + 1$ ways. Consequently,
$q^{\I_\gamma}(r) = 3\times|V| + 1$.
\end{proof}

\begin{remark}
When the UNA is dropped, we can use a similar argumentation to the one given in 
Remark~\ref{rem:una-conp-core} to reduce the problem of non-3-colourability 
of undirected graphs to that of query answering over $\dlliteRbag$ ontologies. 
\end{remark}

\end{document}